\documentclass[onefignum,onetabnum,final]{siamart171218}

\usepackage{microtype}
\usepackage{graphicx}
\usepackage{caption}
\usepackage{subcaption}
\usepackage{booktabs} 

\usepackage{hyperref}

\usepackage{amsmath}
\usepackage{amssymb}
\usepackage{mathtools}
\usepackage{amssymb,amsmath}
\usepackage{boondox-cal}
\usepackage{comment}
\usepackage{algorithm}
\usepackage{algorithmic}
\usepackage{caption}
\usepackage[font=small,labelfont=bf]{caption}
\usepackage{subcaption}
\usepackage{xcolor}
\usepackage{graphicx}
\usepackage{dsfont}
\usepackage{listings}
\usepackage{soul}
\usepackage{epstopdf}

\newtheorem{assumption}[theorem]{Assumption}

\newcommand{\vertiii}[1]{{\left\vert\kern-0.25ex\left\vert\kern-0.25ex\left\vert #1 
    \right\vert\kern-0.25ex\right\vert\kern-0.25ex\right\vert}}

\newcommand{\EE}{\mathbb{E}}

\newcommand{\mcu}{\mathcal{u}}
\newcommand{\mcv}{\mathcal{v}}
\newcommand{\bmcu}{\bar{\mathcal{u}}}
\newcommand{\tmcu}{\tilde{\mathcal{u}}}

\newcommand{\Fs}{\mathcal{F}}

\newcommand{\Ls}{\mathcal{L}}

\newcommand{\Ns}{\mathcal{N}}

\newcommand{\Ds}{\mathcal{D}}

\newcommand{\Us}{\mathcal{U}}
\newcommand{\Os}{\mathcal{O}}

\newcommand{\RR}{\mathbb{R}}

\newcommand{\Sb}{\mathbb{S}}

\newcommand{\bA}{\bar{A}}
\newcommand{\bc}{\bar{c}}
\newcommand{\bx}{\bar{x}}
\newcommand{\bp}{\bar{p}}
\newcommand{\bu}{\bar{u}}
\newcommand{\bB}{\bar{B}}
\newcommand{\bF}{\bar{F}}
\newcommand{\bQ}{\bar{Q}}
\newcommand{\bR}{\bar{R}}
\newcommand{\bK}{\bar{K}}
\newcommand{\bP}{\bar{P}}
\newcommand{\bL}{\bar{L}}
\newcommand{\bV}{\bar{V}}
\newcommand{\bZ}{\bar{Z}}
\newcommand{\bomega}{\bar{\omega}}
\newcommand{\bPsi}{\bar{\Psi}}
\newcommand{\bbeta}{\bar{\beta}}
\newcommand{\bPhi}{\bar{\Phi}}
\newcommand{\rQ}{\mathrm{Q}}
\newcommand{\rbQ}{\bar{\mathrm{Q}}}

\newcommand{\uJ}{\ddot{J}}
\newcommand{\uK}{\ddot{K}}
\newcommand{\uA}{\ddot{A}}
\newcommand{\uN}{\ddot{N}}
\newcommand{\uP}{\ddot{P}}
\newcommand{\uDelta}{\ddot{\Delta}}
\newcommand{\uPhi}{\ddot{\Phi}}
\newcommand{\unabla}{\ddot{\nabla}}

\newcommand{\tA}{\tilde{A}}
\newcommand{\tx}{\tilde{x}}
\newcommand{\tu}{\tilde{u}}
\newcommand{\tB}{\tilde{B}}
\newcommand{\tE}{\tilde{E}}
\newcommand{\tF}{\tilde{F}}

\newcommand{\tP}{\tilde{P}}
\newcommand{\tK}{\tilde{K}}
\newcommand{\tL}{\tilde{L}}
\newcommand{\tN}{\tilde{N}}

\newcommand{\tJ}{\tilde{J}}

\newcommand{\tf}{\tilde{f}}
\newcommand{\tomega}{\tilde{\omega}}

\newcommand{\tnabla}{\tilde{\nabla}}

\newcommand{\bfK}{\mathbf{K}}

\newcommand{\cnabla}{\check{\nabla}}
\newcommand{\czeta}{\check{\zeta}}
\newcommand{\ddnabla}{\ddot{\nabla}}
\newcommand{\cJ}{\check{J}}

\newcommand{\hK}{\hat{K}}

\newcommand{\tr}{Tr}
\DeclareMathOperator*{\argmin}{argmin}

\DeclareMathOperator{\diag}{diag}
\newcommand{\Var}{Var}
\DeclareMathOperator{\proj}{Proj}

\definecolor{NavyBlue}{rgb}{0.0, 0.0, 0.5}
\definecolor{ceruleanblue}{rgb}{0.16, 0.32, 0.75}

\definecolor{green1}{rgb}{0.2,0.7,0.2}

\definecolor{CadetBlue}{rgb}{0.37, 0.62, 0.63}

\title{Independent RL for Cooperative-Competitive Agents: A Mean-Field Perspective}
\author{Muhammad Aneeq uz Zaman\thanks{Analog Devices Incorporated}
\and Alec Koppel\thanks{Artificial Intelligence Research, JP Morgan Chase \& Co}
\and Mathieu Laurière\thanks{School of Mathematics and Data Science, NYU Shanghai}
\and Tamer Ba\c{s}ar\thanks{Coordinated Science Lab, University of Illinois, Urbana-Champaign. Funding: Research of T. Ba\c{s}ar was supported in part by the Air Force Office of Scientific Research (AFOSR) Grant FA9550-24-1-0152}
} 

\begin{document}

\maketitle

\begin{abstract}
We address in this paper Reinforcement Learning (RL) among agents that are grouped into teams  such that there is cooperation within each team but general-sum (non-zero sum) competition across different teams. To develop an RL method that provably achieves a Nash equilibrium, we focus on a linear-quadratic structure. Moreover, to tackle the non-stationarity induced by multi-agent interactions in the finite population setting, we consider the case where the number of agents within each team is infinite, i.e., the mean-field setting. This results in a General-Sum LQ Mean-Field Type Game (GS-MFTG). We characterize the Nash equilibrium (NE) of the GS-MFTG, under a standard invertibility condition. This MFTG NE is then shown to be $\Os(1/M)$-NE for the finite population game where $M$ is a lower bound on the number of agents in each team. These structural results motivate an algorithm called Multi-player Receding-horizon Natural Policy Gradient (MRNPG), where each team minimizes its cumulative cost \emph{independently} in a receding-horizon manner. Despite the non-convexity of the problem, we establish that the resulting algorithm converges to a global NE through a novel problem decomposition into sub-problems using backward recursive discrete-time Hamilton-Jacobi-Isaacs (HJI) equations, in which \emph{independent natural policy gradient} is shown to exhibit linear convergence under time-independent diagonal dominance. Numerical studies included corroborate the theoretical results.
\end{abstract}
\section{Introduction}

%
Multi-agent reinforcement learning (MARL) has gained popularity in recent years for its ability to address sequential decision-making problems among agents \cite{zhang2021multi,li2021distributed}. While a substantial effort has gone into developing algorithms and performance guarantees when agents interact in a purely cooperative setting, relatively less effort has gone into settings where agents' objectives may be in opposition \cite{littman1994markov}, such as congestion \cite{toumi2020tractable}, financial markets \cite{lussange2021modelling}, and negotiations in markets \cite{krishna1998intelligent}. It is known that finding equilibrium policies, i.e., the Nash equilibrium (NE), for each agent in such a general-sum stochastic game is in general NP-hard \cite{jin2021v}. Furthermore, in many real-world scenarios, agents behave in groups, with cooperation inside the group and competition between groups. Therefore, in this work, we study mixed \emph{Cooperative-Competitive (CC)} team settings, and seek to understand conditions for which a NE is achievable.

To enable a tractable formulation, we make two structural specifications: (i) agents' dynamics are linear and their costs\footnote{Costs are negative payoffs/rewards.} are quadratic, i.e., the linear-quadratic (LQ) setting \cite{bacsar1998dynamic}; and (ii) the number of agents within a team approaches infinity such that it may be approximated by its mean-field (MF) limit \cite{huang2006large,lasry2006jeux}. This setting results in a General Sum LQ Mean-Field Type Game (GS-MFTG).  We provide more background on these two specifications.

\textbf{The LQ specification} is motivated by a recent study of RL methods in the LQ Regulator (LQR) setting, which has gained traction for its role as a benchmark problem in which one can establish rigorous performance guarantees \cite{fazel2018global,malik2019derivative}, as well as solve a variety of practical problems without the opacity of neural networks 
\cite{ivanov2012numerical}.  The LQ setting has several real-world applications as in finance (linear quadratic permanent income theory \cite{sargent2000recursive}, portfolio management \cite{cardaliaguet2018mean}) and engineering (Wireless Power Control \cite{huang2003individual}), etc. Aside from these direct use cases, LQ system theory has been essential in obtaining non-asymptotic sample bounds for RL algorithms like Policy Gradient \cite{fazel2018global} and Actor-Critic \cite{yang2019global}, hence paving the way for later works to obtain similar guarantees in more general settings \cite{agarwal2021theory,qu2021exploiting}. 
Our goal is to understand to what extent we can broaden the scope of the LQ setting to provide a discernible problem class in CC multi-agent settings.

\textbf{The mean-field approximation} is motivated by the fact that the complexity of equilibria in finite population CC settings grows with the size of the teams \cite{carmona2015mean,sanjari2022nash}, and thus the transient effect of competitive agents' policies on stochastic state transitions appears in any gradient estimate of the cost, which cannot be annihilated unless one holds other agents' policies fixed.

The aforementioned structures yield a GS-MFTG, which upon first glance may seem a pristine setting, but in actuality, even this simplified setting exhibits fundamental technical challenges. Similar to the RL for LQR setting, the objective is non-convex, which in principle should preclude finding a NE. This was already observed for the simpler zero-sum (purely competitive setting) in \cite{carmona2020policy}. Therefore, in this work, we pose the following question:

\begin{center}
\textit{Is it possible to construct a data driven 
method to achieve the Nash Equilibrium \\ in 
CC Games?}
\end{center}

We answer this question affirmatively and  our {\bf main  contributions} are as follows:
\begin{itemize}
\item We formalize the CC 
game in a finite population LQ framework and derive its mean-field approximation as a MFTG. This approximation introduces a bias that is $\Os(1/M)$ (Theorem \ref{thm:eps_Nash}), where $M$ is the minimum number of agents in any team. Inspired by adapted open-loop control analysis (Appendix \ref{sec:open_loop_analysis}), 
we decompose the GS-MFTG into two general-sum LQ games and establish existence, uniqueness and characterization of the Nash equilibrium (NE) 
of these games (Theorem \ref{thm:CLNE}), under standard invertibility conditions.
%
%

\item To learn the NE of the GS-MFTG we develop a \emph{Multi-player Receding-horizon Natural Policy Gradient (MRNPG)} algorithm, in which the players independently update their policies using natural policy gradients in a receding-horizon manner \cite{zhang2023revisiting}. The receding-horizon approach decomposes the harder problem of learning NE policies for all time-steps, into simpler sub-problems of learning NE policies for each time-step, in a retrograde manner. This approach is inspired by the Hamilton-Jacobi-Isaacs (HJI) equations.
%
\item We establish convergence of MRNPG in two steps. First, we establish that for each time-step the MRNPG algorithm converges to the NE policy at a linear rate (Theorem \ref{thm:inner_loop_conv}) under a time and system noise-independent \emph{diagonal dominance} condition\footnote{We further relax the diagonal dominance condition using a cost-augmentation technique in Appendix \ref{sec:cost_augment}.}
This new condition generalizes (while being much easier to verify) the System Noise condition of \cite{hambly2023policy} (Lemma \ref{lem:generality_of_DD}). Furthermore, we also obviate the need for covariance matrix estimation as in \cite{hambly2023policy}. 
Finally, when the policy gradient approximation error per time-step is $\Os(\epsilon)$, the resulting error in NE computation is shown to be $\Os(\epsilon)$ (Theorem \ref{thm:main_res}).
\item Finally we corroborate the convergence of MRNPG within the context of a numerical example, and provide a comparison with several benchmarks.
\end{itemize}
For ease of exposition, proofs of some of the auxiliary results are included in Supplementary Materials after References.

\begin{figure}[h!]
    \centering
    \includegraphics[width=0.6\textwidth]{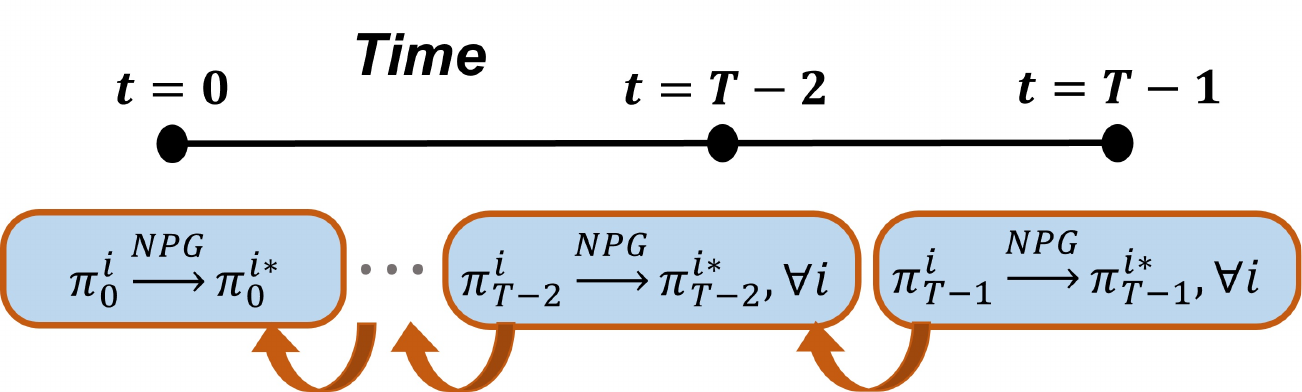}
    \caption{MRNPG Algorithm employs Natural Policy Gradient (NPG) for each agent at timestep $t$, starting from $t=T-1$ and moving in a receding horizon manner (backwards-in-time), to approximate the NE of the game.}
    \label{fig:MRNPG_flow}
\end{figure}
\textbf{MRNPG Illustration.} Figure \ref{fig:MRNPG_flow} shows the flow of the MRNPG algorithm. The algorithm utilizes Natural Policy Gradient (NPG) to converge to the NE $\pi^{i*}_t, \forall i$ first for $t=T-1$, then $t=T-2$ and continues in a receding horizon manner (backwards-in-time). MRNPG algorithm solves the HJI equations \cite{bacsar1998dynamic} in an approximate manner. According to the HJI equations the NE can be computed backwards in time using,
\begin{small}\begin{align*}
    \pi^{i*}_t = \argmin_{\pi} C^i_t(\pi, \pi^{-i*}_t | \pi^*_{[t+1,T-1]}), \hspace{0.2cm} t \in \{0,\ldots,T-1\},
\end{align*}\end{small}
\hspace{-0.28cm} for all $i$, where $C^i_t$ is the partial cost of agent $i$ and $\pi_{[t,t^\prime]}$ is the set of policies for all agents from time $t$ to $t^\prime$. MRNPG utilizes NPG to perform the above shown minimization (for each timestep $t$) in an independent data-driven stochastic manner to obtain $\tilde\pi^{i*}_t$ which is shown to be
\begin{small}\begin{align*}
    \tilde\pi^{i*}_t \approx \argmin_{\pi} C^i_t(\pi, \tilde\pi^{-i*}_t | \tilde\pi^*_{[t+1,T-1]}), \hspace{0.2cm} t \in \{0,\ldots,T-1\}
\end{align*}\end{small}
In the paper we show that due to the LQ framework and under the diagonal dominance condition, MRNPG converges linearly to the \emph{exact} NE of the game i.e. $\tilde\pi^{i*}_t \approx \pi^{i*}_t$  $\forall i,t$.


\textbf{Related Work} In the purely non-cooperative setting, considerable research activity has taken place, in Reinforcement Learning (RL) both in the finite population case \cite{zhang2021multi, hambly2023policy, mao2022improving} and in the mean-field limit (Mean-Field Games or MFGs for short) \cite{guo2019learning,subramanian2019reinforcement,elie2020convergence,zaman2021adversarial,angiuli2022unified,lauriere2022scalable,zaman2022reinforcement,zaman2023oracle}; see~\cite{lauriere2022learning} for a recent survey. Within the literature in competitive multi-agent  RL, two-player Zero-Sum games have proven to be particularly amenable to analysis with works such as \cite{zhang2019policy,bu2019global,carmona2021linear}. Conversely there have also been results in data driven techniques for solving equilibria where the utilities of the agents satisfy a potential function condition \cite{ding2022independent}. Apart from this work, RL for (General-Sum) CC agents remains predominantly an uncharted territory from a theoretical analytical standpoint besides a few empirical studies \cite{lowe2017multi}. On the other hand, there has been some recent works in RL for the purely cooperative setting (also called Mean-Field Control) \cite{subramanian2019reinforcement,carmona2019linear,gu2021mean,carmona2023model} or for the Zero-Sum game setting with finitely many \cite{zhang2020robust} or infinitely many agents \cite{carmona2021linear}. 

{\bf The work of \cite{hambly2023policy}} is particularly salient due to a lack of structure (Zero-Sum or Potential) in the utilities of the agents. It proves that natural policy gradient converges to the General-Sum Nash Equilibrium given knowledge of model parameters and a system noise inequality (Assumption 4 
\cite{hambly2023policy}). This inequality is hard to verify as the LHS increases with increasing system noise but the RHS may not decrease due to the dependence of cost on system noise. We generalize this convergence result to the data-driven (with unknown model parameters) CC setting under a more \emph{general, verifible}, \emph{time and system noise}-independent, diagonal dominance condition and obviate the need for covariance matrix estimation. This is made possible by employing the receding-horizon approach introduced for the LQR problem under perfect \cite{zhang2023revisiting} and imperfect \cite{zhang2023learning} information.



\noindent \textbf{Notation: } We use $[N] := \{1,\ldots,N\}$ for any $N \in \mathbb{N}$. We have $\lVert x \rVert_A$ representing $(x^\top A x)^{1/2}$  for a nonnegative-definite matrix $A$, and with $A=I$,  $\lVert \cdot \rVert := \lVert \cdot \rVert_2$. As in game theoretic notation, $a^{-i} := (a^j)_{j \neq i}$ represents the set of values $a$ for players other than $i$. Gaussian distribution with mean $\mu$ and covariance $\Sigma$ is denoted by $\Ns(\mu,\Sigma)$.

\section{Setup \& Equilibrium Characterization} \label{sec:form}
We consider a general-sum game among multiple teams, where agents within a team are cooperative, and distinct teams compete. Initially, we pose this problem in the finite-agent setting for the case where agents' dynamics are linear stochastic and costs (negative rewards) are quadratic, i.e., the LQ framework. Subsequently, to enable the characterization of Nash equilibria of the game, we consider the mean-field approximation within each team, i.e., the number of agents within each team tends to infinity, which alleviates the transient effect of other agents' decisions on the system dynamics \cite{carmona2015mean}. The result is a LQ mean-field type game (MFTG). In this section we delineate the Nash equilibria of the GS-MFTG and provide $\epsilon$-Nash guarantees for the finite agent Cooperative-Competitive game. 

{\bf Linear Quadratic Cooperative-Competitive Games.} We thus begin by considering the finite agent Cooperative-Competitive (CC) game problem. 
In a CC game, the agents are grouped into $N$ teams, with each team $i \in [N]$ having $M_i$ agents and agent $j$ in team $i$ having linear dynamics, in that it is driven by a linear function of 
the agent's state  $x^{i,j}_t$, the agent's action $u^{i,j,i}_t $, the average state $\bx^i_t = \sum_{j = 1}^{M_i} x^{i,j}_t/M_i$ of population $i$, and the average actions $\bu^{i,k}_t =\sum_{j = 1}^{M_i} u^{i,j,k}_t /M_i$ of population $i$. The control action $u^{i,j,i}_t$ refers to that of agent $j$ in team $i$, whereas $u^{i,j,k}_t$ for $k \in [N] \setminus i$ refers to the adversarial control input of player $k$ into the dynamics of agent $j$ of team $i$. In addition, the dynamics are affected by Gaussian noise with a team-specific covariance, $\omega^{i,j}_t \sim \Ns(0,\Sigma^i)$, which are independent for every $(i,j)$, as well as common noise $\omega^{0,i}_t \sim \Ns(0,\Sigma^0)$. Altogether, these lead to the linear dynamical system:
\begin{small}\begin{align}
    x^{i,j}_{t+1} =&  A^i_t x^{i,j}_t + \bA^i_t \bx^i_t \!+\! \sum_{k=1}^N \big(B^{i,k}_t u^{i,j,k}_t \!+ \!\bB^{i,k}_t\bu^{i,k}_t \big) +\omega^{0,i}_{t+1} + \omega^{i,j}_{t+1}, \label{eq:finite_agent_dyn}
\end{align}\end{small}
\hspace{-0.15cm}for $t \in \{0,\ldots,T-1\}$, where $A^i_t, \bA^i_t \in \mathbb{R}^{m \times m}$ and $B^{i,k}_t, \bB^{i,k}_t \in \mathbb{R}^{m \times p}$. 
For simplicity of analysis, we also assume that agents' initial states are null except for the exogenous input noise: $x^{i,j}_0 = \omega^{i,j}_0 + \omega^0_0$. All agents in team $i$ aim to optimize a single cost function  $J^i_M$ over a finite horizon, that depends on both the team's policy 
$\mcu^{i} \in \Us^i_M$ and those of other teams $\mcu^{-i} \in \Us^i_M$, where the set of policies $\Us^i_M$ is adapted to the state processes of all agents $(x^{i,j}_t)_{i \in [N], j \in [M_i], \forall t}$ in a causal manner. 
\begin{small}\begin{align} 
    & J^i_M(\mcu^{i}, \mcu^{-i}) = \frac{1}{M_i} \EE \hspace{-0.1cm} \sum_{j \in [M_i]} \sum_{t=0}^{T-1} \lVert x^{i,j}_t - \bx^i_t \rVert^2_{\rQ^i_t} \hspace{-0.05cm}+ \lVert \bx^i_t \rVert^2_{\rbQ^i_t} \hspace{-0.1cm}  \label{eq:finite_agent_utility} \\
    &  \hspace{2cm} +\!\sum_{k=1}^N\lVert u^{k,j,i}_t \!\!-\! \bu^{k,i}_t\rVert^2_{R^{k,i}_t} \!+\! \lVert \bu^{k,i}_t \rVert^2_{\bR^{k,i}_t}+\!\lVert x^{i,j}_T \!-\! \bx^i_T \rVert^2_{\rQ^i_T}\!\! + \!\lVert \bx^i_T \rVert^2_{\rbQ^i_T}  \nonumber 
\end{align}\end{small}
\hspace{-0.18cm} where $R^{ik}_t, \bR^{ik}_t \succ 0$ and  $\rQ^i_t, \rbQ^i_t \succeq 0$ are symmetric matrices of suitable dimensions. The subscript $M$ of $J$ refers to the fact that it is for the finite-agent setting. The cost contains a consensus term that penalizes deviation from the average and regulation of the average states and control actions. 

For each agent to minimize \eqref{eq:finite_agent_utility} with respect to policies $\mcu^i$, the appropriate solution concept is Nash Equilibrium, which we subsequently define. A set of policies $(\mcu^{i*})_{i \in [N]}$ 
is a Nash equilibrium (NE) if $J^i_M(\mcu^{i*}, \mcu^{-i*}) \leq J^i_M(\mcu^{i}, \mcu^{-i*})$ holds for all alternative policy selections $\mcu^i \in \Us^i_M$. 
%

{\bf Mean Field Approximation (GS-MFTG). } 
In general, equilibrium policies in finite-player dynamic games are functions of every player's state. This causes challenges in the computability and learnability of the NE with a large number of agents. 
Due to this difficulty, 
we shift focus to the mean-field setting where the number of agents in each team $M_i \rightarrow \infty$. The NE policies in the mean-field setting are shown to depend on the state of the generic agent and the average state (mean-field), thus resolving the scalability problem. This limiting game is termed the GS-MFTG.


The state dynamics in the MFTG can be formulated by concatenating the state dynamics of the $j^{th}$ agent (for any $j \in \mathbb{N}$) from each team and discarding the superscript $j$. Using \eqref{eq:finite_agent_dyn}, the dynamics of this joint state will be as follows
\begin{small}\begin{align} \label{eq:gen_agent_dyn}
	\hspace{-0.2cm} x_{t+1} = A_t x_t + \bA_t \bx_t + \sum_{i=1}^N \big(B^i_t u^i_t + \bB^i_t\bu^i_t \big) + \omega^0_{t+1} + \omega_{t+1} 
\end{align}\end{small}%
for $t = \{0,\ldots,T-1\}$, where $A_t = \diag((A^i_t)_{i \in [N]})$, $B^i_t = \diag ((B^{i,k}_t)_{k \in [N]})$ and $\bB^i_t = \diag((\bB^{i,k}_t)_{k \in [N]})$, 
$\bx_t = \EE[x_t \mid (\omega^0_s)_{s \leq t}]$ and $\bu^i_t = \EE[u^i_t \mid (\omega^0_s)_{s \leq t}]$ for $i \in [N]$.
The policies of the players belong to the feasible sets $\mcu^i \in \Us^i$, where $\Us^i$ is the set of all policies causally adapted to the state and mean-field process $\{ x_0, \bx_0, \ldots, x_t, \bx_t \}$.
The cost of player $i \in [N]$ is
\begin{small}\begin{align} 
\hspace{-0.2cm}	J^i(\mcu^{i}, \mcu^{-i}) = & \EE \Big[\sum_{t=0}^{T-1} \lVert x_t - \bx_t \rVert^2_{Q^i_t} + \lVert \bx_t \rVert^2_{\bQ^i_t} + \lVert u^i_t - \bu^i_t \rVert^2_{R^i_t} + \lVert \bu^i_t \rVert^2_{\bR^i_t} \label{eq:gen_agent_cost} \\
& \hspace{4cm} + \lVert x_T - \bx_T \rVert^2_{Q^i_T} + \lVert \bx_T \rVert^2_{\bQ^i_T} \Big] \nonumber
\end{align}\end{small}%
where $R^i_t = \diag\big( (R^{i,k}_t)_{k \in [N]} \big)$ and $\bR^i_t = \diag\big( (\bR^{i,k}_t)_{k \in [N]} \big) \succ 0$ are symmetric. The matrices $Q^i_t,\bQ^i_t \succeq 0$ are block matrices such that $(Q^i_t)_{ii} = \rQ^i_t$ and $(\bQ^i_t)_{ii} = \rbQ^i_t$ and $0$ otherwise. Each player $i \in [N]$ aims to minimize its cost function $J^i$ using its policy $\mcu^i \in \Us^i$
where now we formally introduce the solution concept of Nash equilibrium for the MFTG.
\begin{definition} \label{def:NE_GS_MFTG}
The set of policies 
$(\mcu^{i*})_{i \in [N]}$ are in Nash equilibrium (NE) for the MFTG if for each $i \in [N]$, 
	$J^i(\mcu^{i*}, \mcu^{-i*}) \leq J^i(\mcu^{i},\mcu^{-i*})$ 
for all $\mcu^i \in \Us^i$. 
\end{definition}

By the definition of the class of policies $\Us^i$ 
this NE is symmetric, i.e., cooperating agents will have symmetric policies.
%


{\bf Approximate Nash Equilibria. } 
The NE of the GS-MFTG is shown to be a $\Os(1/M)$-Nash equilibrium of the CC game \eqref{eq:finite_agent_dyn}-\eqref{eq:finite_agent_utility} where $M := \min_{i \in [N]} M_i$ is the minimum number of agents across all teams $i$. 
This result guarantees that the NE found using RL techniques (in Section \ref{sec:MRNPG}) will be arbitrarily close to the NE of the finite population CC game given that $M$ is large enough. 
\begin{theorem} \label{thm:eps_Nash}
    The NE of the MFTG is $\epsilon$-Nash for the finite agent CC game \eqref{eq:finite_agent_dyn}-\eqref{eq:finite_agent_utility} where $\epsilon = \Os(1/\min_{i \in [N]}M_i)$, i.e.
    \begin{small}\begin{align*}
        J^i_M(\mcu^{i*},\mcu^{-i*}) - \inf_{\mcu^i \in \Us^i_M} J^i_M(\mcu^i,\mcu^{-i*}) = \Os \bigg( \frac{T \sigma}{\min_i M_i} \bigg)
    \end{align*}\end{small}
\end{theorem}
The guarantee is obtained by analyzing the difference between finite and infinite population cost functions and by bounding $\epsilon$ with a function of this difference. The complete proof is provided in Appendix \ref{sec:eps_Nash}. 
Although a similar (albeit slower $\Os(1/\sqrt{N}$) bound has been shown in the case of purely competitive MFGs \cite{zaman2022reinforcement}, to our knowledge, this work is the first to quantify the equilibrium gap between finite population CC games and GS-MFTGs. 

{\bf Characterization of Nash Equilibria. }
Next, we study a decomposition of the GS-MFTG into two sub-problems, the first one pertaining to the mean-field and the second one to the deviation from the mean-field. This decomposition is inspired by the analysis of the open-loop Nash equilibrium of the MFTG, but since it is not central to this exposition it is deferred to Appendix \ref{sec:open_loop_analysis}. 
Then we use the discrete-time Hamilton-Jacobi-Isaacs (HJI) equations to characterize the NE and present certain \emph{invertibility} conditions to guarantee its existence and uniqueness. These conditions are a mainstay of scenarios with finitely many competing players \cite{bacsar1998dynamic}; it might also be noted that the MFG framework  with infinitely many competing players typically requires a different set of conditions \cite{huang2006large,lasry2006jeux}. 

To further elaborate on this, let us define $y_t = x_t - \bx_t$ and $v^i_t = u^i_t - \bu^i_t$. The dynamics of $y_t$ and $\bx_t$ can be written in a decoupled manner using \eqref{eq:gen_agent_dyn},
\begin{small}\begin{align} 
	y_{t+1} = A_ty_t + \sum_{i=1}^N B^i_t v^i_t  + \omega_{t+1},
	\bx_{t+1} = \tA_t \bx_t + \sum_{i=1}^N \tB^i_t \bu^i_t + \omega^0_{t+1}, \label{eq:dyn_decomp}
\end{align}\end{small}%
where $\tA^i_t = A^i_t + \bA^i_t$ and $\tB^i_t = B^i_t + \bB^i_t$. Since the dynamics of processes $y_t$ and $\bx_t$ are decoupled, we can decompose the cost of $i^{th}$ player, $J^i$ \eqref{eq:gen_agent_cost}, into two decoupled parts as well:
\begin{small}\begin{align}
	J^i(\mcu^{i}, \mcu^{-i}) = & J^i_y(\mcv^i, \mcv^{-i}) + J^i_{\tx}(\bmcu^i,\bmcu^{-i}), \nonumber\\
    J^i_y(\mcv^i, \mcv^{-i}) = & \EE \Big[ \sum_{t=0}^{T-1} \big[ \lVert y_t \rVert^2 _{Q^i_t} + \lVert v^i_t \rVert^2_{R^i_t} \big] + \lVert y_T \rVert^2 _{Q^i_T} \Big], \nonumber\\
    J^i_{\tx}(\bmcu^i,\bmcu^{-i}) = & \EE \Big[ \sum_{t=0}^{T-1} \big[ \lVert \bx_t \rVert^2 _{\bQ^i_t} + \lVert \bu^i_t \rVert^2_{\bR^i_t} \big] + \lVert \bx_T \rVert^2_{\bQ^i_T} \Big]. \label{eq:cost_decomp}
\end{align}\end{small}
This results in two decoupled $N$-player LQ game problems, and hence we use the discrete-time HJI equations to characterize the NE of the GS-MFTG in the following theorem. 
Before stating the theorem we introduce the coupled Riccati equations for $N$-player LQ games \cite{bacsar1998dynamic,hambly2023policy}. Consider control matrices
\begin{small}\begin{align*}
	K^{i*}_t = (R^i_t + (B^i_t)^\top P^{i*}_{t+1} B^i_t)^{-1} (B^i_t)^\top P^{i*}_{t+1} L^i_t, \hspace{0.1cm} \forall i \in [N] 
\end{align*}\end{small}%
where $P^{i*}_t$ are determined using Coupled Riccati equations,
\begin{small}\begin{align} \label{eq:Riccati_NE}
	P^{i*}_t & = L^\top_t P^{i*}_{t+1} L_t + (K^{i*}_t)^\top R^i_t K^{i*}_t + Q^i_{t}, \hspace{0.2cm} P^{i*}_{T} = Q^i_T, 
\end{align}\end{small}%
s.t. $L_t = A_t - \sum_{i=1}^N B^i_t K^{i*}_t$ and $L^i_t = A_t - \sum_{j \neq i} B^j_t K^{j*}_t$. The expressions for $\bK^{i*}_t$ and $\bP^{i*}_t$ can be obtained by replacing $A_t, B^i_t, Q^i_t$ and $R^i_t$ matrices by $\tA_t, \tB^i_t, \bQ^i_t$ and $\bR^i_t$.
\begin{theorem} \label{thm:CLNE}
The set of policies $(\mcu^{i*})_{i \in [N]}$ constitutes a NE if, and only if,
\begin{align} \label{eq:CLNE_control}
	u^{i*}_t(x_t) = -K^{i*}_t (x_t - \bx_t) - \bK^{i*}_t \bx_t,
\end{align}
$i \in [N], t \in \{0,\ldots.T-1\}$ where the control parameters $K^{i*}_t$ are guaranteed to exist and be unique if the matrices $\Phi_t$ and $\bPhi_t$ are invertible. 
\end{theorem}
The proof can be found in Appendix \ref{sec:CLNE}. The matrices $\Phi_t$ are block diagonal with the $i$th diagonal entry $R^i_t + (B^i_t)^\top P^{i*}_{t+1} B^i_t$ and $ij$th non-diagonal entries $(B^i_t)^\top P^{i*}_{t+1} B^j_t$. The emergence of the invertibility condition is a consequence of the HJI equations and it naturally arises in similar games with a finite number of competing entities, e.g. LQ Games \cite{bacsar1998dynamic}. 

\section{Multi-player Receding-horizon NPG (MRNPG)} \label{sec:MRNPG}
In this section, we discuss the challenges that arise when solving the NE through a data-driven approach, having established its linear form as shown in \eqref{eq:CLNE_control}. In our CC setting, finding the NE is elusive as the cost function even for a single agent LQ  control problem is non-convex \cite{fazel2018global,hu2023toward}. Additionally, vanilla policy gradient is known to diverge, in purely competitive $N$-player LQ games \cite{mazumdar2019policy}. \cite{hambly2023policy} has proven the convergence of natural policy gradient for \emph{purely competitive agents} albeit under complete knowledge of the model parameters and a system noise condition. In Section \ref{sec:convergence}, show the linear rate of convergence of the natural policy gradient to the CC NE even when model parameters are unknown  under a diagonal dominance condition which is shown to generalize the system noise condition (Assumption 4 in \cite{hambly2023policy}) in Lemma \ref{lem:generality_of_DD}. We also eliminate a potential source of error in the analysis by obviating the need for covariance matrix estimation.

We now provide some details of our algorithmic construction to achieve this result. The key idea is the use of receding-horizon approach (inspired by the discrete-time Hamilton-Jacobi-Isaacs (HJI) equations) whereby finding policies for all the agents at a fixed time $t$ (and moving backwards-in-time) reveals a quadratic cost structure, allowing linear convergence guarantee. 
The net result is that teams comprised of finitely many agents ($M \geq 2$) following the MRNPG algorithm approach the NE (Definition \ref{def:NE_GS_MFTG}) of the GS-MFTG \eqref{eq:gen_agent_dyn}-\eqref{eq:gen_agent_cost}.

{\bf Disentangling the Mean Field and deviation from Mean-Field.}
Let us define a joint state $x^j_t$ for $j \in [M]$ by concatenating the states of $j^{th}$ agents in all teams at time $t$ such that $x^j_t = [(x^{1,j}_t)^\top, \ldots, (x^{N,j}_t)^\top]^\top$. 
Due to the linear form of NE \eqref{eq:CLNE_control} we restrict the controllers to be linear in state and empirical mean-field, $x^j_t$ and $\tx_t$, respectively (without loss of generality Theorem \ref{thm:CLNE}), such that $u^{i,j}_t = -K^i_t (x^{j}_t - \tx_t) - \bK^i_t \tx_t$, which results in $\tu^i_t = -\bK^i_t \tx_t$, where the average state $\tx_t = \frac{1}{M} \sum_{j = 1}^M x^j_t$ and average control $\tu^i_t = \frac{1}{M} \sum_{j = 1}^M u^{i,j}_t$. Under these control laws, the dynamics of agent $j \in [M]$
and the dynamics of the empirical mean-field $\tx_t$ are linear. Specifically the dynamics of the average state $\tx_t$ is
\begin{align}
    \tx_{t+1} = \bL_t \tx_t + \tomega^0_{t+1}, \label{eq:tx_t}
\end{align}
where $L_t = A_t - \sum_{i=1}^N B^i_t K^i_t$, $\bL_t = \tA_t - \sum_{i=1}^N \tB^i_t \bK^i_t$ and $\tomega^0_{t+1} =  \omega^0_{t+1} + \sum_{j=1}^M \omega^j_{t+1}/M$. Now as in the Section \ref{sec:form}, we introduce the deviation $y^j_t = x^j_t - \tx_t$ with dynamics
\begin{align} \label{eq:y_t}
	y^j_{t+1} & = L_t y^j_t + \tomega^j_{t+1}, 
\end{align}
where $\tomega^j_{t+1} = (M-1) \omega^j_{t+1}/M - \sum_{k \neq j} \omega^k_{t+1}/M$. The cost of any agent $j \in [M]$ and player/team $i \in [N]$ under control laws $(\bfK^i,\bfK^{-i})$ where $\bfK^i = (K^i_t,\bK^i_t)_{t \in [T]}$ for any $i \in [N]$ 
can be decomposed in a similar manner;
\begin{align*} 
	& \tJ^{i,j} (\bfK^i,\bfK^{-i})= \tJ^{i,j}_y(K^i,K^{-i}) + \tJ^{i}_{\tx}(\bK^i, \bK^{-i}) \\ 
    & \tJ^{i,j}_y(K^i,K^{-i})  = \EE \Big[\sum_{t=0}^{T-1} \lVert y^j_t \rVert^2_{Q^i_t + (K^i_t)^\top R^i_t K^i_t} + \lVert y^j_T \rVert^2_{Q^i_T} \Big] \\
    & \tJ^{i}_{\tx}(\bK^i, \bK^{-i}) = \EE \Big[\sum_{t=0}^{T-1} \lVert \tx_t \rVert^2_{\bQ^i_t + (\bK^i_t)^\top \bR^i_t \bK^i_t}  + \lVert \tx_T \rVert^2_{\bQ^i_T} \Big],
\end{align*}
where $K^i = (K^i_t)_{t \in [T]}$ and $\bK^i = (\bK^i_t)_{t \in [T]}$. 
This problem re-parameterization allows us to decouple the cost function into terms of consensus error $y$ and mean field $\tx$.

{\bf Receding Horizon Mechanism.} We solve each problem using the receding-horizon approach. This approach is inspired by the HJI equations \cite{bacsar1998dynamic} which obtain the NE by solving for the NE policy at time $T-1$ and then moving in retrograde-time. The receding-horizon approach is a data-driven version of the discrete-time HJI equations. 
Next, we provide details of the receding-horizon approach for the process $y$. At each time-step $t$ 
we solve for the set of controllers at time $t$, $(K^i_t)_{i \in [N]}$, which minimize the cost $\tJ^{i,1}_{y,t}(K^i,K^{-i})$, while keeping the controllers $(K_s)_{t < s < T}$ fixed 
\begin{align} \label{eq:min_tJ_iyt}
\min_{K^i_t} \tJ^{i}_{y,t}(K^i,K^{-i}) = \EE \Big[ \lVert y^1_t \rVert^2_{Q^i_t + (K^i_t)^\top R^i_t K^i_t} + \sum_{s=t+1}^T \lVert y^1_s \rVert^2_{Q^i_s + (K^i_s)^\top R^i_s K^i_s} \Big],
\end{align}
where $y_t \sim \Ns(0,\Sigma_y),  \Sigma_y \succ 0$ and $K^i_T = 0$ for each player $i \in [N]$. Notice that the choice of agent $1$ is arbitrary. The minimization problem at time $t$, \eqref{eq:min_tJ_iyt} (due to forward-in-time controllers $(K_s)_{t < s < T}$ being fixed) is quadratic in the control parameter $K^i_t$ which allows it to satisfy a specific PL condition (Lemma \ref{lem:PL}) which ensures the linear rate of convergence of natural policy gradient to the NE. Similarly for the process $\tx$ at each time-step 
$t \in \{T,\ldots,0\}$ we fix the controllers forward-in-time $(\bK_s)_{t < s < T}$ and solve for the set of controllers $(\bK^i_t)_{i \in [N]}$ which minimize the cost 
\hspace{-0.8cm}
\begin{small}
\begin{align} \label{eq:min_tJ_ixt}
\min_{\bK^i_t} \tJ^{i}_{\tx,t}(\bK^i,\bK^{-i}) :=  \EE \Big[ \lVert \tx_t \rVert^2_{\bQ^i_t + (\bK^i_t)^\top \bR^i_t \bK^i_t} + \sum_{s=t+1}^T \lVert \tx_s \rVert^2_{\bQ^i_s + (\bK^i_s)^\top \bR^i_s \bK^i_s} \Big]
\end{align}
\end{small}%
where $\tx_t \sim \Ns(0,\Sigma_{\tx}), \Sigma_{\tx} \succ 0$ and $\bK^i_T = 0$, $\forall i \in [N]$. For given time $t$, the set of controllers which satisfy equations \eqref{eq:min_tJ_iyt} and \eqref{eq:min_tJ_ixt} for a fixed set of forward-in-time controller $(K_s)_{t < s < T}$ are called \emph{local-Nash controllers} and defined as
\begin{align} \label{eq:local_Nash_cont}
    \tK^{i*}_t = \argmin_{K^i_t} \tJ^{i}_{y,t}(K^i,K^{-i}),
    \tilde{\bK}^{i*}_t = \argmin_{\bK^i_t} \tJ^{i}_{\tx,t}(\bK^i,\bK^{-i}). \nonumber
\end{align}


{\bf MRNPG Algorithm Construction.}
Before introducing the MRNPG algorithm, we first characterize what constitutes a valid search direction in policy space i.e. policy gradient in the receding-horizon setting. 
\begin{lemma}\label{lemma:receding_horizon_gradient}
    In a receding-horizon setting for a fixed $t \in \{T-1,\ldots,0\}$ the policy gradient of cost $\tJ^i_{y,t}(K^i,K^{-i})$ with respect to $K^i_t$ is $\nabla^i_{y,t} (K^i,K^{-i}) := \delta \tJ^i_{y,t}(K^i,K^{-i})/\delta K^i_t$ 
    \begin{small}\begin{align} \label{eq:pol_grad}
        \nabla^i_{y,t} (K^i,K^{-i})= & 2 \big((R^i_t +  (B^i_t)^\top P^i_{y,t+1} B^i_t)K^i_t - (B^i_t)^\top P^i_{y,t+1} \big(A_t-\sum_{j \neq i}B^j_t K^j_t \big)  \big)\Sigma_y,
    \end{align}\end{small}%
    with $P^i_{t+1}$ defined in terms of controllers $(K^i_s)_{i \in [N], t < s < T}$
    \begin{small}\begin{align} \label{eq:tP_iyt}
    	\hspace{-0.27cm}P^i_{y,t} = Q^i_t + (K^i_t)^T R^i_t K^i_t + L_t^\top P^i_{y,t+1} L_t, P^i_{y,T} = Q^i_T.
    \end{align}\end{small}
\end{lemma}
Proof can be found in Appendix \ref{sec:receding_horizon_gradient}. The policy gradient $\nabla^i_{\tx,t} (K^i,K^{-i}) := \delta \tJ^i_{\tx,t}(\bK^i,\bK^{-i})/\delta \bK^i_t$ has a similar expression. Comparing \eqref{eq:pol_grad} with the policy gradient in \cite{hambly2023policy} we note that due to the receding-horizon approach the matrix $\tP^i_{y,t}$ \eqref{eq:tP_iyt} is fixed and the policy gradient $\nabla^i_{y,t}$ is a function of $\Sigma_y$ which can explicitly be chosen in the receding-horizon approach obviating the need for estimating $\Sigma_y$. 
We use mini-batched zero-order techniques as in \cite{fazel2018global,malik2019derivative} to approximate the policy gradients with $\tnabla^i_{y,t} (K^i,K^{-i})$  and $\tnabla^i_{\tx,t} (\bK^i,\bK^{-i})$, which are the stochastic gradients of the costs $\tJ^i_{y,t}$ and $\tJ^i_{\tx,t}$ 
\begin{small}\begin{align} \label{eq:policy_grad_1}
	\tnabla^i_{y,t} (K^i,K^{-i}) & = \frac{m}{N_br^2} \sum_{j=1}^{N_b} \tJ^i_{y,t}(\hK^i(e_j,t),K^{-i}) e_j 
\end{align}\end{small}%
respectively, where $e_j \sim \mathbb{S}^{pN \times mN}(r)$ is the perturbation and $\hK^i(e,t):= (K^i_t+e,\ldots,K^i_{T-1})$ 
is the  perturbed controller set at time-step $t$. $N_b$ denotes the mini-batch size and $r$ the smoothing radius of the stochastic gradient. $\tnabla^i_{\tx,t} (\bK^i,\bK^{-i})$ is computed in a similar manner. In Appendix \ref{subsec:Anal_stoc_cost} 
we generalize to the sample-path cost oracle from the expected cost oracle of \eqref{eq:min_tJ_ixt}. 
%

Now we state the MRNPG algorithm. The algorithm is quite simple; starting at time $t = T-1$, each team/player $i \in [N]$ updates its control parameters $(K^i_t, \bK^i_t)$ using natural policy gradient and then moves one step backwards-in-time.
\begin{algorithm}[h!]
	\caption{MRNPG for GS-MFTG}
	\begin{algorithmic}[1] \label{alg:RL_GS_MFTG}
		\STATE {Initialize $K^i_t =0, \bK^i_t = 0$ for all $i \in [N], t \in \{0,\ldots,T-1\}$}
		\FOR {$t = T-1,\ldots,1,0,$}
		\FOR {$k = 1,\ldots,K$}
		\STATE {\bf Natural Policy Gradient} for $i \in [N]$
  \begin{small}\begin{align} \label{eq:GDU}
			\hspace{-0.5cm}\begin{pmatrix} K^i_t \\ \bK^i_t	\end{pmatrix} & \leftarrow 
            \begin{pmatrix} K^i_t \\ \bK^i_t	\end{pmatrix} - \eta^i_k  \begin{pmatrix} \tnabla^i_{y,t} (K^i,K^{-i}) \Sigma^{-1}_y \\ \tnabla^i_{\tx,t} (\bK^i,\bK^{-i}) \Sigma^{-1}_{\tx}\end{pmatrix}  
		\end{align}\end{small}
        \hspace{-0.7cm}
		\ENDFOR
		\ENDFOR
	\end{algorithmic}
\end{algorithm}
Notice that (in Algorithm \ref{alg:RL_GS_MFTG}) to compute the natural policy gradients, only the perturbed costs are required, as shown in \eqref{eq:policy_grad_1}, 
and estimating the covariance matrix is not required (as in \cite{hambly2023policy}). This is an independent learning algorithm as all teams independently compute their natural policy gradients and the learning rates $\eta^i_k$ are also independent.

\section{Achieving Nash Equilibrium} \label{sec:convergence}
In this section we analyze the MRNPG algorithm 
and show linear rate of convergence to the NE. To establish our main result, several key steps are needed. Some are standard, such as unbiasedness and smoothness properties of gradient estimators (Lemma \ref{lemma:smoothness_bias}). More unique to this work is the establishment of a Polyak-\L ojasiewicz (PL, also known as gradient dominance) inequality that relates the difference of costs with the update direction as Lemma \ref{lem:PL}. While such a result is expected if an objective function is strongly convex, in a non-convex setting it generally does not hold. That it does in the single-agent LQR setting is the central contribution of \cite{fazel2018global}. Here, we generalize it to the LQ game theoretic setting. In particular, we establish a PL condition of cost $\tJ^i_{y,t}$ which is much simpler than the prior work \cite{hambly2023policy}. We proceed then with the following technical lemma.
\begin{lemma}\label{lem:PL}(Polyak-\L ojasiewicz inequality) The cost function $\tJ^i_{y,t}(\tK^i,K^{-i*})$ satisfies the following growth condition with respect to gradient $\lVert \nabla^i_{y,t}(\tK^i,K^{-i*}) \rVert^2_F$
\begin{small}\begin{align*} 
    &\tJ^i_{y,t}(\tK^i,K^{-i*}) - \tJ^i_{y,t}(K^{i*},K^{-i*})  \leq \frac{\lVert \Sigma_{K^*} \rVert}{{\sigma}_R \sigma^2_y} \lVert \nabla^i_{y,t}(\tK^i,K^{-i*}) \rVert^2_F
\end{align*}\end{small}%
where $\tK^i = (K^i_t,K^{i*}_{t+1},\ldots,K^{i*}_{T-1})$ and $\sigma_y$ is the minimum eigenvalue of $\Sigma_y$.
\end{lemma}
The proof can be found in Appendix \ref{sec:PL}. The cost $\tJ^{i}_{\tx,t}$ also satisfies a similar PL condition. Next we analyze the MRNPG algorithm and show the linear rate of convergence to the NE. Let us first introduce the smoothed gradients \cite{fazel2018global,malik2019derivative,hu2023toward} of costs $\tJ^i_{y,t}$ and $\tJ^i_{\tx,t}$, respectively as, $\ddnabla^i_{y,t} (K^i,K^{-i})$ and $\ddnabla^i_{\tx,t} (\bK^i,\bK^{-i})$ which are the expectations of the stochastic gradients \eqref{eq:policy_grad_1}. The smoothed gradient $\ddnabla^i_{y,t} (K^i,K^{-i})$ is given by
\begin{small}\begin{align*}
	\ddnabla^i_{y,t} (K^i,K^{-i}) & = \frac{m}{r} \EE_e \big[\tJ^i_{y,t}(\hK^i(e,t),K^{-i}_t) e \big],
\end{align*}\end{small}%
where $e \sim \Sb^{m-1}(r)$. $\ddnabla^i_{\tx,t} (\bK^i,\bK^{-i})$ has a similar expression. Now we state some results from the literature which shows that the stochastic gradient is an unbiased estimator of the smoothed gradient and quantifies the bias between the smoothed, stochastic and policy gradients. 
\begin{lemma}[\cite{malik2019derivative}]\label{lemma:smoothness_bias}
Consider the smoothed gradient $\ddnabla^i_{y,t} (K^i,K^{-i})$ in \eqref{eq:policy_grad_1} for the per-team consensus error $y_t^j$ for team $j$ at time $t$, as well as the stochastic gradient $\tnabla^i_{y,t} (K^i,K^{-i})$ of the receding horizon cost in Lemma \ref{lemma:receding_horizon_gradient}, then 
%
	\begin{small}\begin{align*}
		\EE[\tnabla^i_{y,t} (K^i,K^{-i})] & = \ddnabla^i_{y,t} (K^i,K^{-i}), \\
        \lVert \ddnabla^i_{y,t} (K^i,K^{-i}) - \nabla^i_{y,t} (K^i,K^{-i}) \rVert_2 & = \Os(r), \\
        \lVert \tnabla^i_{y,t} (K^i,K^{-i}) - \ddnabla^i_{y,t} (K^i,K^{-i}) \rVert_2 & = \Os \bigg( \frac{ \sqrt{\log \delta^{-1}}}{N_br}  \bigg)
	\end{align*}\end{small}%
where the last equality follows with probability $1-\delta$,  $N_b$ is a mini-batch size, and $r$ is the smoothing parameter. 
\end{lemma}
Similar bounds can also be obtained for the gradients pertaining to process $\tx$, $\nabla^i_{\tx,t} (K^i,K^{-i})$, $\tnabla^i_{\tx,t} (\bK^i,\bK^{-i})$ and $\ddnabla^i_{\tx,t} (\bK^i,\bK^{-i})$. These bounds show that the approximation error in the stochastic gradient $\delta^i_t := \max (\lVert \tnabla^i_{y,t} (K^i,K^{-i}) - \nabla^i_{y,t} (K^i,K^{-i}) \rVert ,$ $ \lVert \tnabla^i_{\tx,t} (\bK^i,\bK^{-i}) - \nabla^i_{\tx,t} (K^i,K^{-i}) \rVert) = \Os(\epsilon)$ with probability $1-\delta$ if the smoothing radius, $r = \Os(\epsilon)$ and mini-batch size, $N_b = \Theta(\log(1/\delta)/\epsilon^{2})$. Now we introduce the diagonal dominance condition. 
\begin{assumption}[Diagonal Dominance] \label{asm:diag_dom}
    The diagonal dominance condition entails that 
    \begin{small}\begin{align*}
        \sigma(R^i_t) \geq \sqrt{2m(N-1)} \gamma^2_{B,t} \gamma^i_{P,t+1}
    \end{align*}\end{small}%
where $\gamma_{B,t} := \max_{i \in [N]} \lVert B^i_t \rVert$ and $\gamma^i_{P,t} := \lVert P^{i*}_{t} \rVert$ where matrices $P^{i*}_{t}$ are defined in \eqref{eq:Riccati_NE}. 
\end{assumption}
This condition is similar to conditions in the static games literature \cite{frihauf2011nash} and ensures that the matrices $\Phi_t$ and $\bPhi_t$ (Theorem \ref{thm:CLNE}) are diagonally dominant and hence invertible. In Section \ref{sec:cost_augment} we propose a cost augmentation mechanism which ensures the diagonal dominance condition. The System Noise (SN) condition (Assumption 4
) in \cite{hambly2023policy} is hard to verify as the LHS increases with increased system noise but the RHS may not decrease due to the dependence of the cost on system noise. In contrast, Assumption \ref{asm:diag_dom} is dependent on model parameters and can be verified by direct computation. Moreover Assumption \ref{asm:diag_dom} is independent of the time-horizon $T$ whereas the SN condition gets harder to satisfy with increasing $T$. In the following lemma we establish that for any $T \in \mathbb{N}$, Assumption \ref{asm:diag_dom} generalizes the SN assumption.
\begin{lemma} \label{lem:generality_of_DD}
    For any $T \in \mathbb{N}$, Assumption \ref{asm:diag_dom} generalizes the System Noise (SN) assumption in (Assumption 4 in
 \cite{hambly2023policy}).
\end{lemma}

Now under Assumption \ref{asm:diag_dom} we show the linear rate of convergence of receding-horizon update \eqref{eq:GDU} to the local Nash controllers \eqref{eq:local_Nash_cont}.
\begin{theorem}\label{thm:inner_loop_conv}
    Under Assumption \ref{asm:diag_dom} for each $t \in \{T-1,\ldots,0\}$, if conditions in Theorem \ref{thm:CLNE} are satisfied, $k = \Theta\big(\log\big(\frac{1}{\epsilon}\big)\big)$, $\eta^i_k$ is upper bounded by model parameters, the approximation error in the stochastic gradient is $\delta_t = \Os(\epsilon)$, then the optimality gap $\lVert K^{i,k}_t - \tK^{i*}_t \rVert = \Os(\epsilon)$ and $\lVert \bK^{i,k}_t - \tilde\bK^{i*}_t \rVert = \Os(\epsilon)$ for all $i \in [N]$. 
\end{theorem}

%
Closed-form expressions of the bounds can be found in the proof (Section \ref{sec:inner_loop_conv}). These results are more general than those in \cite{li2021distributed,fazel2018global,hu2023toward,malik2019derivative} due to the presence of competing players learning in an independent fashion. Due to the receding-horizon approach, the cost functions \eqref{eq:min_tJ_iyt}-\eqref{eq:min_tJ_ixt} to be minimized have a quadratic structure, which satisfies a PL condition (Lemma \ref{lem:PL}) and allows for the linear rate of convergence. Moreover, the receding-horizon approach bounds the difference between the \emph{target} controller (one which solves \eqref{eq:min_tJ_ixt}) and the NE $K^{i*}_t$, by controlling the error in the $\tP^i_{y,t}$ matrix \eqref{eq:tP_iyt}. Finally the receding-horizon approach obviates the system noise condition in \cite{hambly2023policy} by explicitly designing the covariance matrix $\Sigma_y$ to be positive definite. All these factors combine to ensure that if $K = \Os(\log(1/\epsilon))$ the error in control parameters are $\Os(\epsilon)$. Using this result now we state the main result of the paper, presenting the finite sample convergence bounds of Algorithm \ref{alg:RL_GS_MFTG}.
\begin{theorem} \label{thm:main_res}
    If all conditions in Theorem \ref{thm:inner_loop_conv} are satisfied, then $e^K_t := \max_{j \in [N]} \lVert K^j_t - K^{j*}_t \rVert = \Os(\epsilon)$ for sufficiently small $\epsilon > 0$ and $t \in \{T-1,\ldots,0\}$. 
\end{theorem}
The error $e^K_t$ is due to a combination of the approximation error in the stochastic gradient $\delta_t$ at time $t$ and the accumulated error in the forward-in-time controllers $e^K_{t+1}, \ldots, e^K_{T-1}$. We first characterize these two quantities and then show that if the approximation error $\delta_t = \Os(\epsilon)$ and the number of inner-loop iterations $K = \Os(\log(1/\epsilon))$, then the accumulated error at any time $t \in \{0,\ldots, T-1\}$ never exceeds $\epsilon$ scaled by a constant multiplier. One important point to note is that $\epsilon=\Os(1/N)$ to keep the approximation error $\delta_t$ small, which avoids instability in the algorithm.

\section{Numerical Analysis} \label{sec:numer_anly}
\begin{figure}[t!]
    \centering
    \includegraphics[width=\linewidth]{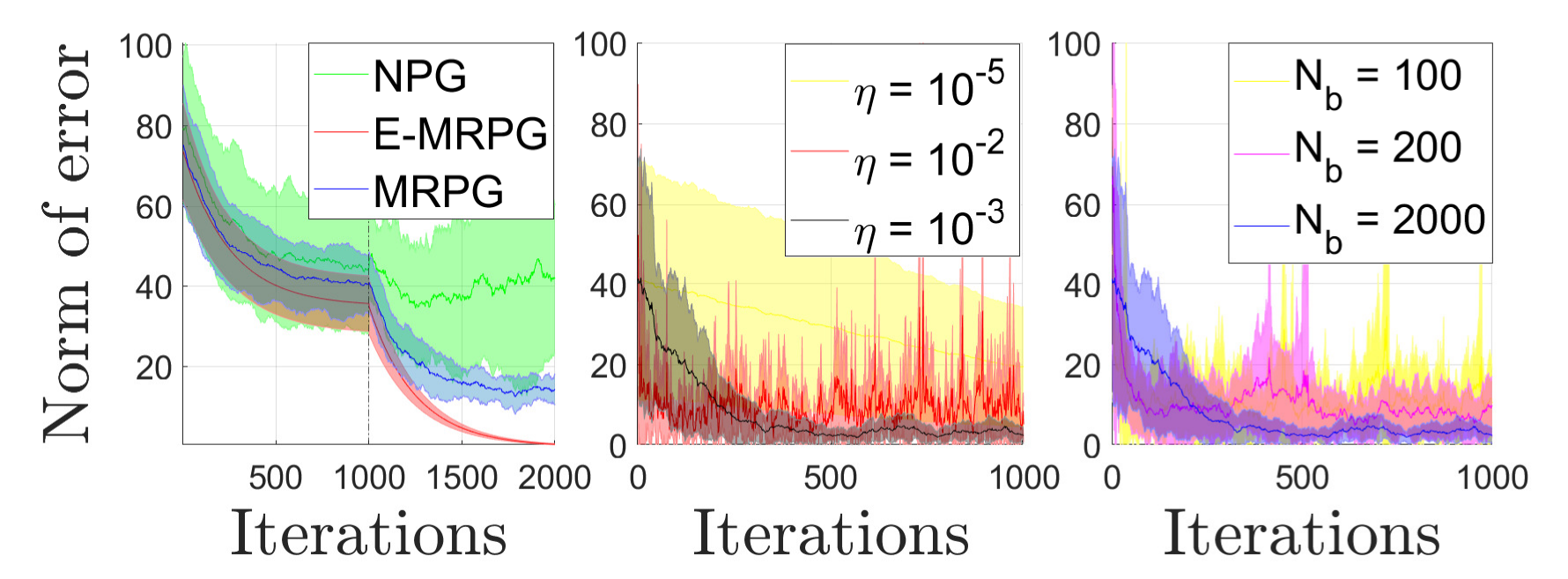}
    \caption{Numerical Analysis of MRNPG algorithm. \textbf{(left)} comparison with \emph{Vanilla} Natural Policy Gradient (NPG) and Exact-MRNPG, \textbf{(center)} performance \emph{with respect to} different values of learning rate $\eta^i_k$, and \textbf{(right)} mini-batch size $N_b$.}
    \label{fig:numer_anly}
\end{figure}
We first simulate the MRNPG algorithm for time horizon $T=2$, number of teams $N=2$, number of agents per team $M=1000$ and  agents having scalar dynamics. We note that in contrast to Zero-Sum Game studies in the literature \cite{jin2021v}, we consider a General-Sum setting where $N=2$, which essentially carries the same kind of difficulty as $N>2$ albeit with a different sample complexity. For each time-step the number of inner-loop iterations $K=1000$, the mini-batch size $N_b=5000$ and learning rate $\eta^i_k = 0.001$. In Figure \ref{fig:numer_anly}(left) we first compare the error convergence in the MRNPG algorithm with \emph{Vanilla} Natural Policy Gradient (NPG) which does not utilize the receding-horizon approach and the Exact-MRNPG which uses exact natural policy gradients for update. As expected, the Exact-MRNPG converges very well for each time-step and MRNPG also converges albeit with variance due to noise in the policy gradients. The NPG algorithm, on the other hand, is seen to be slow at convergence with high variance. The superior convergence is due to the fact that MRNPG decomposes the problem and solves each problem in retrograde-time. In Figure \ref{fig:numer_anly}~(center) we evaluate MRNPG performance for $N=1, T=1$ and different values of learning rate $\eta^i_k$. $\eta^i_l = 0.001$ is shown to provide the best convergence properties with fast decrease in error (unlike slow convergence with $\eta^i_k = 0.001$) and reliable performance (unlike the high variance with $\eta^i_k = 0.01$). Figure \ref{fig:numer_anly}~(right) shows that the variance of the natural policy gradient update decreases with increasing mini-batch size $N_b$ but at the cost of higher sample complexity.
\begin{figure}[t!]
    \centering
    \includegraphics[width=\linewidth]{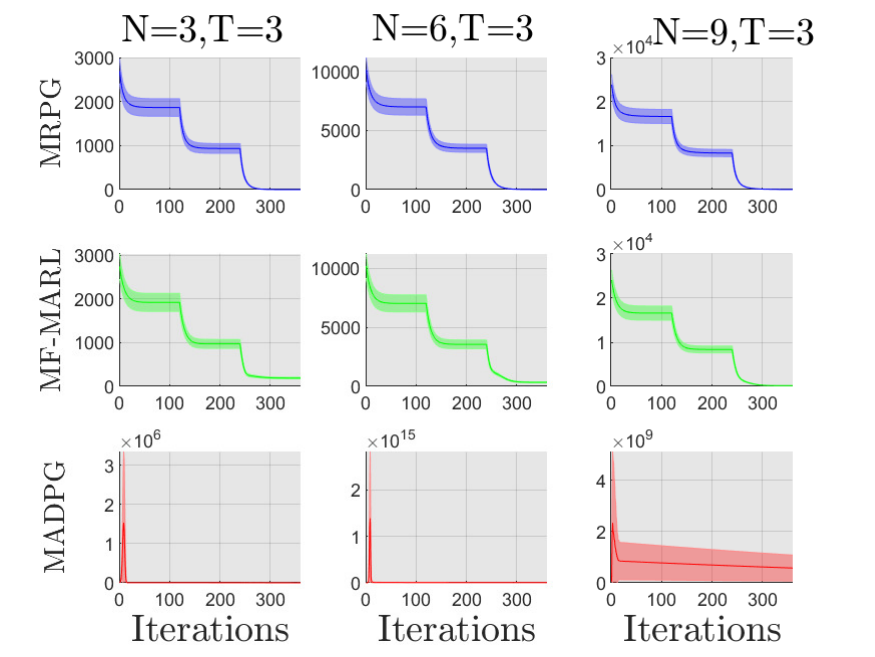}
    \caption{Error convergence for exact versions of MRNPG, MF-MARL and MADPG for $N=\{3,6,9\}$, $T=3$ and $m=p=2$.}
    \label{fig:comp_MRNPG_MADPG}
\end{figure}

{In Figure \ref{fig:comp_MRNPG_MADPG} we  provide a comparison of MRNPG with MADPG \cite{lowe2017multi} and MF-MARL \cite{yang2018mean} in CC setting. Notice that since CC setting is quite novel, direct comparison is only possible with a limited number of works. We use exact versions of all algorithms allowing faster convergence (compare the number of iterations in Figures \ref{fig:numer_anly} and \ref{fig:comp_MRNPG_MADPG}), and extension to data-driven stochastic versions appear to be straightforward. The MF-MARL algorithm is also a receding-horizon algorithm (albeit for finite state and action spaces) but assume a large number of competing players (notice that MRNPG can deal with any number of competing players). The MADPG algorithm \cite{lowe2017multi} is a variant of the actor-critic method where each agent has an actor-critic and the critic is learned in a centralized manner. Figure \ref{fig:comp_MRNPG_MADPG} compares the algorithms for increasing number of players/teams $N = \{3,6,9\}, T=3, m=p=2$. For smaller $N$ ($N=3,6$) MF-MARL shows an offset in error convergence compared to MRNPG. This is due to the fact that MF-MARL computes a mean-field equilibrium which will be $\epsilon$-Nash with $\epsilon \rightarrow 0$ as $N$ gets large. On the other hand, increasing $N$ is shown to cause an overshoot in error convergence of MADPG compared to the steady decrease in error for MRNPG. This is due to the fact that the error in earlier time-steps is significantly affected by error in the later time-steps (HJI). Due to the receding-horizon nature of MRNPG, it learns backwards-in-time, thus allowing MRNPG to control the error in later time-steps, and consequently avoiding the overshoot displayed by MADPG. Hence MRNPG shows good performance for a wide range of $N$.
} 

\section{Conclusion}
This paper has addressed the problem of achieving a Nash equilibrium in a General-Sum Mean-Field Type Game (GS-MFTG) within a Cooperative-Competitive (CC) multi-agent setting. The paper has developed the Multi-player Receding-horizon Natural Policy Gradient (MRNPG) algorithm. We have then shown linear convergence to the Nash equilibrium (NE) of the MFTG, relaxing the need of system noise conditions and covariance matrix estimation, with a diagonal dominance condition. The theoretical results have been corroborated through numerical analysis  and a comparison with benchmark algorithms (MADPG and MF-MARL), showing good convergence of MRNPG for a large range of $N$. 
The main limitation of the present work is that in order to have a full analysis, we worked within a linear-quadratic structure. In future work, we plan to study more complex CC settings, based on the multi-player receding-horizon algorithm we have developed here.

\appendix
\begin{center}
    \Large \bf Appendix
\end{center}

\renewcommand{\thesubsection}{\Alph{subsection}}

\section{Proof of Lemma \ref{lem:generality_of_DD}}\textbf{[Generality of the Diagonal Dominance Condition]}
Let us start by introducing some of the variables used in the SN condition \cite{hambly2023policy}. 
\begin{align*}
    \bar{\rho} & := \rho^* + N \gamma_{B,0} \sqrt{\frac{T \psi}{\underline{\sigma}_X \sigma_R}} + \frac{1}{20 T^2} \\
    \rho^* & := \max \bigg\{ \max_{0 \leq t \leq T-1} \bigg\lVert A_t + \sum_{t=1}^N B^i_t  K^{i*}_t \bigg\rVert, 1 + \delta \bigg\}, \delta > 0 \\
    \psi & := \max_{i \in [N]} \big\{ \tJ^i_{y,0}(K^{i,(0)},K^{-i*}) - \tJ^i_{y,0}(K^{i*},K^{-i*}) \big\} \\
    \underline{\sigma}_X & := \min \Big\{ \sigma_{\min} \big(\EE[x_0 x^\top_0]\big), \sigma_{\min} \big(\EE[\omega_t \omega^\top_t ]\big)\Big\} 
\end{align*}
Furthermore $\Sigma_{K^*} := \sum_{t=0}^{T-1} \Sigma^{K^*}_t, \hspace{0.2cm} \Sigma^{K^*}_t = \EE \big[x^{K^*}_t \big(x^{K^*}_t \big)^\top \big]$. The SN condition can be written down as,
\begin{align*}
    \frac{\big(\underline{\sigma}_X \big)^5}{ \big\lVert \Sigma_{K^*} \big\rVert} > 20 m (N-1)^2 T^2\frac{\gamma^4_{B,0} (\max_i (\tJ^i_{y,0}(K^{i*},K^{-i*}) + \psi)^4}{\sigma_Q^2 \sigma_R^2} \bigg( \frac{\bar{\rho}^{2T} - 1}{\bar{\rho}^{2} - 1} \bigg)^2
\end{align*}
The first thing to notice is that the term in the SN condition $\big(\frac{\bar{\rho}^{2T} - 1}{\bar{\rho}^{2} - 1}\big)^2 > 1$ because $\bar{\rho} > 1$ (by definition) hence can be discarded. Furthermore $\psi > 0$ and using Courant-Fischer theorem we can deduce $\lVert  \Sigma_{K^*} \rVert \geq \underline{\sigma}_X$, hence the SN condition can be simplified
\begin{align*}
    \sigma^2_y \geq \big(\underline{\sigma}_X \big)^2 > \sqrt{20 m} (N-1) \frac{\gamma^2_{B,0} (\max_i (\tJ^i_{y,0}(K^{i*},K^{-i*}))^2}{\sigma_Q \sigma_R} 
\end{align*}
This can equivalently be written as
\begin{align*}
    \sigma_R & > \sqrt{20 m} (N-1)\frac{\gamma^2_{B,0} (\max_i (\tJ^i_{y,0}(K^{i*},K^{-i*}))^2}{\sigma_Q \sigma^2_y}
    \geq \sqrt{20 m} (N-1)\frac{\gamma^2_{B,0} \sigma^2_y (\gamma^i_{P,1})^2}{\sigma_Q \sigma^2_y} \\
    & \geq \sqrt{20 m} (N-1)\gamma^2_{B,0} \gamma^i_{P,1}
\end{align*}
where the second inequality is obtained using Lemma 13 
in \cite{hambly2023policy} and the last one is obtained using the fact that $P^i_T = Q^i_T$ for any $T$. 
By observation $\sqrt{20 m} (N-1)\gamma^2_{B,0} \gamma^i_{P,1} > \sqrt{2m(N-1)} \gamma^2_{B,0} \gamma^i_{P,1}$ hence if the SN condition (Assumption 4 
in \cite{hambly2023policy}) is true then Assumption \ref{asm:diag_dom} is also true and hence is more general than the SN condition.

\section{Proof of Theorem \ref{thm:inner_loop_conv}} \label{sec:inner_loop_conv}\textbf{[Linear convergence of NPG]}
    We prove for a given $t \in \{0,\ldots,T-1\}$ the linear rate of convergence of natural policy gradient \eqref{eq:GDU} for the control policies $(K^i_t)_{i \in [N]}$ with the set of future controllers $(K^i_s)_{s > t, i \in [N]}$ fixed. The techniques for showing the same result for $(\bK^i_t)_{i \in [N]}$ is quite similar and hence is omitted. The linear convergence result follows due to the PL condition (Lemma \ref{lem:PL}) for a given $t \in\{0,\ldots,T-1\}$. We first introduce the set of \emph{local}-NE policies $(\tK^{i*}_t)_{i \in [N]}$ which satisfy the equation \eqref{eq:min_tJ_iyt} for all $i \in [N]$. These are essentially the targets that we want $(K^i_t)_{i \in [N]}$ to achieve.
    
    First we define the natural policy gradient for agent $i$ if all agents are following policies $((K^1_s)_{t \leq s \leq T-1},\dots,(K^N_s)_{t \leq s \leq T-1})$
    \begin{align*}
        E^i_t := \nabla^i_{y,t} (K^i,K^{-i}) \Sigma^{-1}_y = R^i_t K^i_t - (B^i_t)^\top P^i_{t+1} (A_t - \sum_{j=1}^N B^j_t K^j_t) .
    \end{align*}
    The policy update \eqref{eq:GDU} in Algorithm \ref{alg:RL_GS_MFTG} uses $\tnabla^i_{y,t} (K^i,K^{-i}) \Sigma^{-1}_y$ which is a stochastic approximation of $E^i_t$. Now we define another set of controls for a given $i \in [N]$, $(K^i,\tK^{-i*})$ where the player $i \in [N]$ has control policy $(K^i_s)_{t \leq s \leq T-1}$ but the other agents $j \neq i$ follow the local-NE only at time $t$, $(\tK^{j*}_t,K^j_{t+1},\ldots, K^j_{T-1})$. This set of policies $(K^i,\tK^{-i*})$ is if all players $j \neq i$ somehow achieved their respective local-NE policies but player $i$ is following policy $K^i_t$. This set of policies will be useful in the analysis of the algorithm. The natural policy gradient under this set of policies is given as
    \begin{align*}
        E^{i*}_t = \nabla^i_{y,t} (K^i,\tK^{-i*}) \Sigma^{-1}_y =  R^i_t K^i_t - (B^i_t)^\top P^i_{t+1} (A_t - B^i_t K^i_t - \sum_{j \neq i} B^j_t \tK^{j*}_t) .
    \end{align*}
    Notice that the policy update \eqref{eq:GDU} in Algorithm \ref{alg:RL_GS_MFTG} is \emph{not} an estimate of $E^{i*}_t$. Now we introduce some properties of the cost $\tJ^i_{y,t}$ for the set of controllers $(K^i,\tK^{-i*})$ including the PL condition
    \begin{lemma}\label{lem:PL_2} 
    Given that $K^{i\prime} = (K^{i \prime}_t,K^i_{t+1},\ldots)$, the cost function $\tJ^i_{y,t}$ satisfies the following smoothness property:
    \begin{align}
        & \tJ^i_{y,t}(K^{i \prime},\tK^{-i*}) - \tJ^i_{y,t}(K^{i},\tK^{-i*}) = \tr[((K^{i \prime}_t - K^i_t)^\top (R^i_t + (B^i_t)^\top P^i_{t+1} B^i_t) \nonumber \\
        & \hspace{5cm} (K^{i \prime}_t - K^i_t) + 2(K^{i \prime}_t - K^i_t)^\top E^{i*}_t) \Sigma_y].  \label{eq:almost_smooth}
    \end{align}
    Furthermore, the cost function $\tJ^i_{y,t}(\tK^i,K^{-i*})$ also satisfies the PL growth condition with respect to policy gradient $ \nabla^i_{y,t}(\tK^i,K^{-i*}) $ and natural policy gradient $E^{i*}_t$:
    \begin{align*} 
        \tJ^i_{y,t}(K^i,\tK^{-i*}) - \tJ^i_{y,t}(\tK^{i*},\tK^{-i*}) 
        & \leq \frac{\sigma_y}{\sigma_R} \lVert E^{i*}_t  \rVert^2_F \\
        & \leq \frac{1}{{\sigma}_R \sigma_y} \lVert \nabla^i_{y,t}(K^i,\tK^{-i*}) \rVert^2_F
    \end{align*}
    where the value function matrix $P^i_{y,t}$ is defined recursively by
    \begin{align} \label{eq:P_i_t2}
        P^i_{y,t} = Q^i_t + (K^i_t)^T R^i_t K^i_t + (A_t - \sum_{j=1}^N B^j_t K^j_t)^\top P^i_{y,t+1} (A_t - \sum_{j=1}^N B^j_t K^j_t), P^i_{y,T} = Q^i_T.
    \end{align}
    and $\sigma_y$ is the minimum eigenvalue of $\Sigma_y$.
    \end{lemma}
    The proof of this Lemma is similar to that of Lemma \ref{lem:PL} and thus omitted. Before starting the proof of Theorem \ref{thm:inner_loop_conv}, we state the bound on learning rates $\eta^i$ and smallest singular value $\sigma_y$ of matrix $\Sigma_y$
    \begin{align} 
        &\eta^i \leq \min \bigg (\frac{1}{4 \big\lVert R^i_t + (B^i_t)^\top P^i_{t+1} B^i_t \big\rVert^2 + 3 \big\lVert R^i_t + (B^i_t)^\top P^i_{t+1} B^i_t \big\rVert} \label{eq:suff_cond_conv}\\
        & \hspace{4cm} , \frac{1}{2 \big\lVert R^i_t + (B^i_t)^\top P^i_{t+1} B^i_t \big\rVert^2 + 1} \bigg) \nonumber 
    \end{align}
    where 
    $\sigma_R$ is the LSV of matrices $(R^i_t)_{t \in \{0,\ldots, T-1\}}$. For simplicity of analysis we will set $\Sigma_y = \sigma_y I$. The bound on learning rate $\eta^i$ is standard in Policy Gradient literature \cite{fazel2018global,malik2019derivative}. 
    The update \eqref{eq:GDU} in Algorithm \ref{alg:RL_GS_MFTG} employs stochastic natural policy gradient $\tE^i_t := \tnabla^i_{y,t} (K^i,K^{-i}) \Sigma^{-1}_y$ where $\tnabla^i_{y,t}(K^i,K^{-i})$ is the stochastic policy gradient \eqref{eq:policy_grad_1}. Using Lemma \ref{lemma:smoothness_bias}, if the smoothing radius $r = \Os(\epsilon)$ and mini-batch size $N_b = \Theta(\log(1/\delta)/\epsilon^{2})$, we can obtain the approximation error in stochastic gradient $\delta^i_t := \lVert \tnabla^i_{y,t} (K^i,K^{-i}) - \nabla^i_{y,t} (K^i,K^{-i}) \rVert = \Os(\epsilon)$ with probability $1-\delta$. 
    
    We denote the updated controller after one step of stochastic natural policy gradient as $K^{i\prime}$, which is defined as follows:
    \begin{align}
        K^{i \prime}_t & = K^i_t - \eta^i \tE^i_t, \nonumber \\
        & = K^i_t - \eta^i (E^i_t + \Delta E^i_t), \text{ where } \Delta E^i_t = \tE^i_t - E^i_t, \nonumber \\
        & = K^i_t - \eta^i (E^{i*}_t + \Delta E^i_t + \Delta E^{i*}_t), \text{ where } \Delta E^{i*}_t = E^i_t - E^{i*}_t \label{eq:Kprime-K}
    \end{align}
    Now we characterize the difference in the costs produced by the update of the controller for agent $i$ from $K^i_t$ to $K^{i \prime}_t$ (given that the other players follow the NE controllers $(K^{j*}_t)_{j \neq i}$) at the given time $t$. The future controllers $(K^i_s)_{s > t, i \in [N]}$ are assumed to be fixed and the set of value matrices $(P^i_{t+1})_{i \in [N]}$ is a function of the set of future controllers \eqref{eq:P_i_t2}. To conserve space in the following analysis we use the notation $\vertiii{A}^2 := A^\top A$ and $\vertiii{A}^2_B := A^\top B A$ for matrices $A$ and $B$ of appropriate dimensions. Using Lemma \ref{lem:PL_2} we get
    \begin{align}
        & \tJ^i_{y,t}(K^{i \prime},\tK^{-i*}) - \tJ^i_{y,t}(K^{i},\tK^{-i*})  \nonumber \\
        & = \sigma_y \tr[(K^{i \prime}_t - K^i_t)^\top (R^i_t + (B^i_t)^\top P^i_{t+1} B^i_t)(K^{i \prime}_t - K^i_t) + 2(K^{i \prime}_t - K^i_t)^\top E^{i*}_t) ], \nonumber \\
        & = \sigma_y \tr[(\eta^i)^2(E^{i*}_t + \Delta E^i_t + \Delta E^{i*}_t)^\top (R^i_t + (B^i_t)^\top P^i_{t+1} B^i_t)(E^{i*}_t + \Delta E^i_t + \Delta E^{i*}_t)] \nonumber \\
        & \hspace{1.5cm} - 2 \eta^i \tr [(E^{i*}_t + \Delta E^i_t + \Delta E^{i*}_t)^\top E^{i*}_t \Sigma_y ] \nonumber 
        \end{align}
        Simplifying the expression 
        \begin{align}
        & \tJ^i_{y,t}(K^{i \prime},\tK^{-i*}) - \tJ^i_{y,t}(K^{i},\tK^{-i*})  \nonumber \\
        & = \sigma_y\tr \big[(\eta^i)^2 \big(\vertiii{E^{i*}_t}^2_{(R^i_t + (B^i_t)^\top P^i_{t+1} B^i_t)} + 2 (\Delta E^{i*}_t + \Delta E^i_t)^\top (R^i_t + (B^i_t)^\top P^i_{t+1} B^i_t) E^{i*}_t  \nonumber \\
        & \hspace{0.25cm} + \vertiii{\Delta E^{i*}_t}^2_{(R^i_t + (B^i_t)^\top P^i_{t+1} B^i_t)} + 2 (\Delta E^i_t )^\top (R^i_t + (B^i_t)^\top P^i_{t+1} B^i_t) \Delta E^{i*}_t \nonumber \\
        & \hspace{0.25cm} + \vertiii{\Delta E^{i}_t}^2_{(R^i_t + (B^i_t)^\top P^i_{t+1} B^i_t)}\big) \big] - 2 \eta^i \sigma_y \tr \big[\vertiii{E^{i*}_t}^2 \big] - 2 \eta^i \tr \big[ \Sigma_y (\Delta E^i_t + \Delta E^{i*}_t)^\top E^{i*}_t  \big] \nonumber \\
        & \leq \sigma_y \big[ (\eta^i)^2 \tr \big( \vertiii{E^{i*}_t}^2_{(R^i_t + (B^i_t)^\top P^i_{t+1} B^i_t)} \big) + \frac{(\eta^i)^2}{2}\tr\big(\vertiii{E^{i*}_t}^2 \big)  \nonumber \\
        & \hspace{0.5cm} + 2 (\eta^i)^2 \tr\big(\vertiii{(\Delta E^{i*}_t + \Delta E^i_t)^\top (R^i_t + (B^i_t)^\top P^i_{t+1} B^i_t)}^2 \big) \nonumber \\
        & \hspace{0.5cm}  + \frac{3}{2} (\eta^i)^2 \tr\big( \vertiii{\Delta E^{i*}_t}^2_{(R^i_t + (B^i_t)^\top P^i_{t+1} B^i_t)} \big) + 3 (\eta^i)^2 \tr \big( \vertiii{\Delta E^{i}_t }^2_{(R^i_t + (B^i_t)^\top P^i_{t+1} B^i_t)} \big) \big]  \nonumber \\
        & \hspace{0.5cm} - 2 \eta^i \sigma_y \tr \big(\vertiii{E^{i*}_t}^2 \big) + \eta^i \sigma_y \tr \big( \vertiii{E^{i*}_t}^2 \big) + \frac{\eta^i}{\sigma_y} \tr \big(\vertiii{\Sigma_y(\Delta E^i_t + \Delta E^{i*}_t) }^2 \big) \nonumber 
        \end{align}
        where we have used the fact that $2\tr(A^\top B) \leq \tr(A^\top A) + \tr(B^\top B)$ to obtain the inequality. We further have the following inequalities (bounds):
        \begin{align}
        & \tJ^i_{y,t}(K^{i \prime},\tK^{-i*}) - \tJ^i_{y,t}(K^{i},\tK^{-i*})  \label{eq:cost_diff_stoc} \\
        & \leq \sigma_y \big[ (\eta^i)^2 \tr(\vertiii{E^{i*}_t}^2_{(R^i_t + (B^i_t)^\top P^i_{t+1} B^i_t)} ) \nonumber \\
        & \hspace{0.5cm} + \frac{(\eta^i)^2}{2}\tr(\vertiii{ E^{i*}_t}^2 ) + 2 (\eta^i)^2 \tr(\vertiii{(\Delta E^{i*}_t + \Delta E^i_t)^\top (R^i_t + (B^i_t)^\top P^i_{t+1} B^i_t)}^2) \nonumber \\
        & \hspace{0.5cm} + \frac{3}{2} (\eta^i)^2 \tr(\vertiii{\Delta E^{i*}_t}^2_{(R^i_t + (B^i_t)^\top P^i_{t+1} B^i_t)}) + 3 (\eta^i)^2 \tr(\vertiii{\Delta E^{i}_t }^2_{(R^i_t + (B^i_t)^\top P^i_{t+1} B^i_t)}) \big] \nonumber \\
        & \hspace{0.5cm}  -  \eta^i \sigma_y \tr( \vertiii{E^{i*}_t}^2) + \eta^i \sigma_y \tr(\vertiii{\Delta E^i_t}^2 + \vertiii{\Delta E^{i*}_t}^2) \nonumber \\
        & \leq \sigma_y \bigg((\eta^i)^2 \big\lVert R^i_t + (B^i_t)^\top P^i_{t+1} B^i_t \big\rVert^2 + \frac{(\eta^i)^2}{2} - \eta^i \bigg) \tr\big((E^{i*}_t)^\top E^{i*}_t \big) \nonumber \\
        & \hspace{0.5cm} + \sigma_y \Big(4(\eta^i)^2 \big\lVert R^i_t + (B^i_t)^\top P^i_{t+1} B^i_t \big\rVert^2 + 3 (\eta^i)^2 \big\lVert R^i_t + (B^i_t)^\top P^i_{t+1} B^i_t \big\rVert + \eta^i \Big) \nonumber \\
        & \hspace{7cm} \big(\lVert \Delta E^i_t \rVert^2_F + \lVert \Delta E^{i*}_t \rVert^2_F \big)   \nonumber 
    \end{align}
    By definition we know that $\eta^i \leq$
    \begin{align*}
        \min \bigg (\frac{1}{4 \big\lVert R^i_t + (B^i_t)^\top P^i_{t+1} B^i_t \big\rVert^2 + 3 \big\lVert R^i_t + (B^i_t)^\top P^i_{t+1} B^i_t \big\rVert}, \frac{1}{2 \big\lVert R^i_t + (B^i_t)^\top P^i_{t+1} B^i_t \big\rVert^2 + 1} \bigg).
    \end{align*}
    Thus using this bound \eqref{eq:cost_diff_stoc} can be written as
    \begin{align}
         & \tJ^i_{y,t}(K^{i \prime},\tK^{-i*}) - \tJ^i_{y,t}(K^{i},\tK^{-i*}) \nonumber \\
         & \hspace{2cm} \leq - \eta^i \sigma_y \tr \big( (E^{i*}_t)^\top E^{i*}_t \big) + 2 \eta^i  \sigma_y \big(\lVert \Delta E^i_t \rVert^2_F + \lVert \Delta E^{i*}_t \rVert^2_F \big)\label{eq:cost_diff_inter}
    \end{align}
    Using the definition of $\Delta E^{i*}_t$ we can deduce that
    \begin{align*}
        \lVert \Delta E^{i*}_t \rVert_F & = \bigg\lVert (B^i_t)^\top P^i_{t+1} \sum_{j \neq i} B^j_t (K^j_t - K^{j*}_t) \bigg\rVert_F  \\
        & \leq \lVert B^i_t \rVert_F \lVert P^i_{t+1} \rVert_F \sum_{j \neq i} \lVert B^j_t \rVert_F \lVert K^j_t - K^{j*}_t \rVert_F\leq \gamma^2_B \gamma_{P,t+1} \sum_{j \neq i} \lVert K^j_t - K^{j*}_t \rVert_F
    \end{align*}
    where $\gamma_B := \max_{i \in [N], t > 0} \lVert B^i_t \rVert_F$ and $\gamma_{P,t} := \max_{i \in [N]} \lVert P^i_t \rVert_F$. Using these bound we can re-write \eqref{eq:cost_diff_inter} as 
    \begin{align*}
         & \tJ^i_{y,t}(K^{i \prime},\tK^{-i*}) - \tJ^i_{y,t}(K^{i},\tK^{-i*}) \\
         & \leq - \eta^i \sigma_y \tr \big( (E^{i*}_t)^\top E^{i*}_t \big) +2 \eta^i \sigma_y \lVert \Delta E^i_t \rVert^2_F + 2\eta^i \gamma^4_B \gamma^2_{P,t+1} \sigma_y \bigg(\sum_{j \neq i} \lVert K^j_t - K^{j*}_t \rVert^2_F \bigg)^2 \\
         & \leq - \eta^i \sigma_R (\tJ^i_{y,t}(K^{i},\tK^{-i*}) - \tJ^i_{y,t}(\tK^{i*},\tK^{-i*})) + 2 \eta^i \sigma_y \lVert \Delta E^i_t \rVert^2_F \\
         & \hspace{7cm}+ 2\eta^i \gamma^4_B \gamma^2_{P,t+1} \sigma_y \bigg( \sum_{j \neq i} \lVert K^j_t - K^{j*}_t \rVert_F \bigg)^2 
    \end{align*}
    where the second inequality is due to the gradient domination condition in Lemma \ref{lem:PL_2}. 
    Next we characterize the cost difference $\tJ^i_{y,t}(K^{i \prime},\tK^{-i*}) - \tJ^i_{y,t}(\tK^{i*},\tK^{-i*})$:
    \begin{align*}
        &\tJ^i_{y,t}(K^{i \prime},\tK^{-i*}) - \tJ^i_{y,t}(\tK^{i*},\tK^{-i*}) \\
        &= \tJ^i_{y,t}(K^{i \prime},\tK^{-i*}) - \tJ^i_{y,t}(K^{i},\tK^{-i*}) + \tJ^i_{y,t}(K^{i},\tK^{-i*}) -  \tJ^i_{y,t}(\tK^{i*},\tK^{-i*})\\
        & \leq \bigg( 1 - \eta^i \sigma_R \bigg) \big(\tJ^i_{y,t}(K^{i},\tK^{-i*}) -  \tJ^i_{y,t}(\tK^{i*},\tK^{-i*}) \big) + 2 \eta^i \sigma_y \lVert \Delta E^i_t \rVert^2_F \\
        & \hspace{6cm} + 2\eta^i \gamma^4_B \gamma^2_{P,t+1} \sigma_y \bigg( \sum_{j \neq i} \lVert K^j_t - \tK^{j*}_t \rVert_F \bigg)^2\\
        & \leq \bigg( 1 - \eta^i \sigma_R \bigg) (\tJ^i_{y,t}(K^{i},\tK^{-i*}) -  \tJ^i_{y,t}(\tK^{i*},\tK^{-i*})) + 2 \eta^i \sigma_y \lVert \Delta E^i_t \rVert^2_F \\
        & \hspace{2cm} + 2\eta^i  \frac{\gamma^4_B \gamma^2_{P,t+1}(N-1)}{ \sigma_R} \sum_{j \neq i}  \tJ^j_{y,t}(K^{j},\tK^{-j*}) -  \tJ^j_{y,t}(\tK^{j*},\tK^{-j*})
    \end{align*}
    where the last inequality follows by utilizing the smoothness condition \eqref{eq:almost_smooth} in Lemma \ref{lem:PL_2} as shown below:
    \begin{align}
        \label{eq:reverse_PL} &\sum_{j \neq i} \tJ^j_{y,t}(K^{j},\tK^{-j*}) - \tJ^j_{y,t}(\tK^{j*},\tK^{-j*}) \\
        & \hspace{0.7cm} = \sum_{j \neq i} \tr[((K^{j}_t - \tK^{j*}_t)^\top (R^j_t + (B^j_t)^\top P^j_{t+1} B^j_t)(K^{j}_t - \tK^{j*}_t)) \Sigma_y] \nonumber \\
        & \hspace{0.7cm} \geq \sigma_y \sigma_R  \sum_{j \neq i}  \lVert K^j_t - \tK^{j*}_t\rVert^2_F \geq \frac{\sigma_y \sigma_R}{N-1} \bigg( \sum_{j \neq i}  \lVert K^j_t - \tK^{j*}_t\rVert_F \bigg)^2 \nonumber 
    \end{align}
    Assume that $\lVert \Delta E^i_t \rVert^2_F \leq \delta_t$ and define $\bar{\eta} := \max_{i \in [N]} \eta^i$. 
    Now if we characterize the sum of difference, of costs $\sum_{i=1}^N \big( \tJ^i_{y,t}(K^{i \prime},\tK^{-i*}) - \tJ^i_{y,t}(\tK^{i*},\tK^{-i*}) \big)$ we get
    \begin{align}
        & \sum_{i=1}^N \big( \tJ^i_{y,t}(K^{i \prime},\tK^{-i*}) - \tJ^i_{y,t}(\tK^{i*},\tK^{-i*}) \big) \nonumber \\
        & \leq \sum_{i=1}^N  \bigg( 1 - \eta^i \bigg( \sigma_R -  2\frac{m(N-1)}{ \sigma_R}\gamma^4_B \gamma^2_{P,t+1} \bigg)\bigg) \cdot \nonumber \\
        & \hspace{2cm} \big( \tJ^i_{y,t}(K^{i },\tK^{-i*}) - \tJ^i_{y,t}(\tK^{i*},\tK^{-i*}) \big) + 2 \eta^i \sigma_y \lVert \Delta E^i_t \rVert^2_F \nonumber \\
        & \leq \bigg( 1 - \underline{\eta}  \frac{\sigma_R}{2} \bigg) \sum_{i=1}^N (\tJ^i_{y,t}(K^i,\tK^{-i*}) - \tJ^i_{y,t}(\tK^{i*},\tK^{-i*})) + 2 \sum_{i=1}^N \eta^i \sigma_y \lVert \Delta E^i_t \rVert^2_F \nonumber \\
        & \leq \bigg( 1 - \underline{\eta}   \frac{\sigma_R}{2} \bigg) \sum_{i=1}^N (\tJ^i_{y,t}(K^i,\tK^{-i*}) - \tJ^i_{y,t}(\tK^{i*},\tK^{-i*})) + 2 N \bar{\eta}^2 \sigma_y \delta_t \label{eq:cost_update}
    \end{align}
    where the second inequality is obtained using 
    the diagonal dominance condition $\sigma^2_R \geq 2m(N-1) \gamma^4_B \gamma^2_{P,t+1} $ and $\underline{\eta} := \min_{i \in [N]} \eta^i$. Now let us define a sequence of controllers $(K^{i,k}_t)_{i \in [N]}$ for $k = \{0,1,2,\ldots\}$ such that $K^{i,k+1}_t = K^{i,k}_t - \eta^i \tE^{i,k}_t$ where $\tE^{i,k}_t$ is the stochastic policy gradient at iteration $k$. Using \eqref{eq:cost_update} we can write
    \begin{align}
        & \sum_{i=1}^N \tJ^i_{y,t}(K^{i,k},\tK^{-i*}) - \tJ^i_{y,t}(\tK^{i*},\tK^{-i*}) \label{eq:nash_gap_cost} \\
        & \hspace{0.7cm} \leq ( 1 - \underline{\eta}   \sigma_R/2 )^k \sum_{i=1}^N (\tJ^i_{y,t}(K^{i,0},\tK^{-i*}) - \tJ^i_{y,t}(\tK^{i*},\tK^{-i*})) \nonumber \\
        & \hspace{4cm} + 2 N \bar{\eta}^2 \sigma_y \delta_t \sum_{j=0}^k ( 1 - \underline{\eta} \sigma_R/(2  \sigma_y) )^j \nonumber \\
        & \hspace{0.7cm} \leq ( 1 - \underline{\eta} \sigma_R/2 )^k \sum_{i=1}^N (\tJ^i_{y,t}(K^{i,0},\tK^{-i*}) - \tJ^i_{y,t}(\tK^{i*},\tK^{-i*})) + \frac{4 N \bar{\eta}^2 \sigma^2_y }{\underline{\eta} \sigma_R} \delta_t \nonumber
    \end{align}
    We can use \eqref{eq:forward_PL} and 
    \eqref{eq:reverse_PL} to write down the following inequalities:
    \begin{align}
        & \tJ^i_{y,t}(K^{i},\tK^{-i*}) - \tJ^i_{y,t}(\tK^{i*},\tK^{-i*}) \geq \sigma_y \sigma_R \lVert K^i_t - \tK^{i*}_t\rVert^2_F, \nonumber \\
        & \tJ^i_{y,t}(K^{i},\tK^{-i*}) - \tJ^i_{y,t}(\tK^{i*},\tK^{-i*}) \nonumber \\
        & \hspace{3cm} \leq \sigma_y \big( (\gamma_R + \gamma^2_B \gamma_{P,t+1}) \lVert K^i_t - K^{i*}_t \rVert^2_F + \bar{e} \lVert K^i_t - K^{i*}_t \rVert_F \big) \label{eq:reverse_forward_PL}
    \end{align}
    These inequalities can be utilized to upper bound the difference between output of the algorithm and the Nash equilibrium controllers. Using \eqref{eq:reverse_forward_PL} and \eqref{eq:nash_gap_cost} we get
    \begin{align*}
        & \sum_{i=1}^N \lVert K^{i,k}_t - \tK^{i*}_t\rVert^2_F \leq \\
        & \hspace{0.5cm} \frac{4 N \bar{\eta}^2 \sigma_y }{\underline{\eta} \sigma^2_R} \delta_t + \bigg( 1 - \underline{\eta}   \frac{\sigma_R}{2} \bigg)^k \sum_{i=1}^N \frac{ (\gamma_R + \gamma^2_B \gamma_{P,t+1} )\lVert K^{i,0}_t - K^{i*}_t \rVert^2_F + \bar{e} \lVert K^{i,0}_t - K^{i*}_t \rVert_F}{\sigma_R}  \nonumber 
    \end{align*}
    Hence if $k = \Theta\big(\log\big(\frac{1}{\epsilon}\big)\big)$ and $\delta_t = \Os(\epsilon)$, then $\sum_{i=1}^N \tJ^i_{y,t}(K^{i,k},\tK^{-i*}) - \tJ^i_{y,t}(\tK^{i*},\tK^{-i*}) = \Os(\epsilon)$ and using \eqref{eq:almost_smooth} we can deduce that $\lVert K^{i,k}_t - \tK^{i*}_t \rVert^2 = \Os(\epsilon)$ for $i \in [N]$. This linear rate of convergence can also be proved for the cost $\tJ^i_{\bx,t}$ using similar techniques, and hence we do not include it here.

\section{Proof of Theorem \ref{thm:main_res}} \textbf{[Linear convergence of Alg. \ref{alg:RL_GS_MFTG}]}
\newline
    \textbf{Definitions:} Throughout the proof we refer to the output of the inner loop of Algorithm \ref{alg:RL_GS_MFTG} as the set of \emph{output controllers} $(K^i_t)_{i \in [N],t \in \{0,\ldots,T-1\}}$. In the proof we use two other sets of controllers as well. The first set $(K^{i*}_t)_{i \in [N],t \in \{0,\ldots,T-1\}}$ is the NE as characterized in Theorem \ref{thm:CLNE}. The second set is called the \emph{local}-NE (as in proof of Theorem \ref{thm:inner_loop_conv}) and is denoted by $(\tK^{i*}_t)_{i \in [N],t \in \{0,\ldots,T-1\}}$. The proof quantifies the error between the output controllers $(K^i_t)_{i \in [N]}$ and the corresponding NE controllers $(K^{i*}_t)_{i \in [N]}$ by utilizing the intermediate local-NE controllers $(\tK^{i*}_t)_{i \in [N]}$ for each time $t \in \{T-1,\ldots,0\}$. For each $t$ the error is shown to depend on error in future controllers $(K^i_s)_{s \geq t, i \in [N]}$ and the approximation error $\Delta_t$ introduced by the NPG update. If $\Delta_t = \Os(\epsilon)$, then the error between the output and NE controllers is shown to be $\Os(\epsilon)$.
    
    Let us start by denoting the NE value function matrices for agent $i \in [N]$ at time $t \in \{0,1,\ldots,T\}$, under the NE control matrices $(K^{i*}_s)_{i \in [N], s \in \{t+1,\ldots,T-1\}}$ by $P^{i*}_t$. The NE control matrices can be characterized as:
    \begin{align}
        K^{i*}_t & := \argmin_{K^i_t} \tJ^{i}_{y,t}((K^i_t,K^{i*}_{t+1},\ldots,K^{i*}_{T-1}),K^{-i*}) \nonumber \\
        & = -(R^i_t + (B^i_t)^\top P^{i*}_{t+1} B^i_t)^{-1} (B^i_t)^\top P^{i*}_{t+1} A^{i*}_t \label{eq:control_NE}
    \end{align}
    where $A^{i*}_t := A_t + \sum_{j \neq i} B^j_t K^{j*}_t$. 
    The NE value function matrices are defined recursively using the NE controllers as 
    \begin{align}
        P^{i*}_t & = Q^i_t + (A^{i*}_t)^\top P^{i*}_{t+1} A^{i*}_t \label{eq:Pi*t} \\
        & \hspace{2cm} - (A^{i*}_t)^\top P^{i*}_{t+1} B^i_t (R^i_t + (B^i_t)^\top P^{i*}_{t+1} B^i_t)^{-1} (B^i_t)^\top P^{i*}_{t+1} A^{i*}_t \nonumber \\
        & = Q^i_t + (A^{i*}_t)^\top P^{i*}_{t+1} ( A^{i*}_t + B^i_t K^{i*}_t), P^{i*}_T = Q^i_T \nonumber
    \end{align}
    The sufficient condition for existence and uniqueness of the set of matrices $K^{i*}_t$ and $P^{i*}_t$ is shown in Theorem \ref{thm:CLNE}. Let us now define the perturbed value matrices $P^i_t$ (resulting from the control matrices $(K^i_s)_{i \in [N], s \in \{t+1,\ldots,T-1\}}$). At time $t$, let us assume the existence and uniqueness of the \emph{target} control matrices $(\tK^{i*}_t)_{i \in [N]}$ such that
    \begin{align}
        \tK^{i*}_t & := \argmin_{K^i_t} \tJ^{i}_{y,t} (K^i,\tK^{-i*})
        = -(R^i_t + (B^i_t)^\top P^i_{t+1} B^i_t)^{-1} (B^i_t)^\top P^i_{t+1} \tA^i_t \label{eq:control_approx_NE}
    \end{align}
    where $\tK^{j*} = (\tK^{j*}_t, K^j_{t+1},\ldots,K^j_{T-1})$ and $\tA^i_t := A_t + \sum_{j \neq i} B^j_t \tK^{j*}_t$. We will show the existence and uniqueness of the target control matrices given that the inner-loop convergence in Algorithm \ref{alg:RL_GS_MFTG} (Theorem \ref{thm:inner_loop_conv}) is close enough. Using these matrices we define the value function matrices $(\tP^{i}_t)_{i \in [N]}$ as follows:
    \begin{align}
        \tP^i_t & =  Q^i_t + (\tK^{i*}_t)^\top R^i_t \tK^{i*}_t + \bigg(A_t + \sum_{j=1}^N B^j_t \tK^{i*}_t \bigg)^\top P^i_{t+1} \bigg(A_t + \sum_{j=1}^N B^{i}_t \tK^{j*}_t \bigg) \nonumber \\
        & = Q^i_t + (\tA^i_t)^\top P^i_{t+1} \tA^i_t - (\tA^i_t)^\top P^i_{t+1} B^i_t (R^i_t + (B^i_t)^\top P^i_{t+1} B^i_t)^{-1} (B^i_t)^\top P^i_{t+1} \tA^i_t \nonumber \\
        & = Q^i_t + (\tA^i_t)^\top P^i_{t+1} ( \tA^i_t + B^i_t \tK^{i*}_t), \tP^i_T = Q^i_T \label{eq:tPit}
    \end{align}
    Finally we define the perturbed value function matrices $P^i_t$ which result from the perturbed matrices $(K^i_s)_{i \in [N], s \in \{t+1,\ldots,T-1\}}$ obtained using the NPG \eqref{eq:GDU} in Algorithm \ref{alg:RL_GS_MFTG}:
    \begin{align}
        P^i_t = Q^i_t + (K^i_t)^\top R^i_t K^i_t + \bigg(A_t + \sum_{j=1}^N B^j_t K^j_t \bigg)^\top P^i_{t+1} \bigg(A_t + \sum_{j=1}^N B^j_t K^j_t \bigg) \label{eq:Pit}
    \end{align}
    Throughout this proof we assume that the output of the inner loop in Algorithm \ref{alg:RL_GS_MFTG}, also called the \emph{output matrices} $K^i_t$,
    are $\Delta_t$ away from the target matrices $\tK^{i*}_t$, such that $\lVert K^i_t - \tK^{i*}_t \rVert \leq \Delta_t$. We know that
    \begin{align*}
        \lVert K^i_t - K^{i*}_t \rVert \leq \lVert K^i_t - \tK^{i*}_t \rVert + \lVert \tK^{i*}_t - K^{i*}_t \rVert \leq \Delta_t + \lVert \tK^{i*}_t - K^{i*}_t \rVert.
    \end{align*}
    
{
\noindent\textbf{Motivation for proof:} Figure \ref{fig:MRNPG_flow_2} shows the flow of the MRNPG algorithm. The algorithm utilizes NPG to converge to the NE $\pi^{i*}_t, \forall i$, for $t=T-1$, then $t=T-2$ and continues in a receding horizon manner (backwards-in-time). Theorem \ref{thm:inner_loop_conv} proves that the NPG converges to the local NE under the diagonal dominance condition \ref{asm:diag_dom}. Figure \ref{fig:agent_i_costs} shows two distinct cost functions of a given agent $i \in [N]$ in a game with time horizon $T=2$. Let the cost $\tilde{C}^i(\cdot)$ be the cost under, MRNPG algorithm at timestep $t=0$, \emph{after} the NE has been approximated and fixed at timestep $t=1$ i.e. $\tK^{i*}_1 \approx K^{i*}_1, \forall i$. Hence this cost $\tilde{C}^i(K^i_0,K^{-i}_0) = \tJ^{i}_{y,0} \big((K^i_0,\tK^{i*}_1),(K^{-i}_0,\tK^{-i*}_1) \big) \approx \tJ^{i}_{y,0} \big((K^i_0,K^{i*}_1),(K^{-i}_0,K^{-i*}_1) \big) =: C^{i*}(K^i_0,K^{-i}_0)$. On the other hand, let cost $C^i(\cdot)$ be the cost of player $i$ under Vanilla PG method for timestep $t=0$, where the NE has \emph{not} been approximated for timestep $t=1$, hence $C^i(K^i_0,K^{-i}_0) = \tJ^{i}_{y,0} \big((K^i_0,K^{i}_1),(K^{-i}_0,K^{-i}_1) \big) \neq C^{i*}(K^i_0,K^{-i}_0)$ since $K^i_1 \neq K^{i*}_1, \forall i$. For each cost $\tilde{C}^i$ or $C^i$ NPG (\eqref{eq:GDU} Algorithm \ref{alg:RL_GS_MFTG}) converges to the target/minimum of the cost (Theorem \ref{thm:inner_loop_conv}). In Figure \ref{fig:agent_i_costs} we denote the minimum of $C^i$ by $K^{i\prime}_0$ and the minimum of $\tilde{C}^i$ by $\tilde{K}^{i*}_0$. The Hamilton-Jacobi-Isaacs equations \cite{bacsar1998dynamic} in the context of this two-shot LQ game states that 
$
    K^{i*}_0 = \argmin_K C^{i*}(K,K^{-i*}_0), \forall i
$
which coupled with $\tilde{C}^i \approx C^{i*} \implies \tK^i_0 \approx K^{i*}_0$. Hence MRNPG leads to a closer approximation of the NE than Vanilla PG algorithms.
\begin{figure}[h!]
     \hfill
     \begin{subfigure}[b]{0.42\textwidth}
         \centering
         \includegraphics[width=0.8\textwidth]{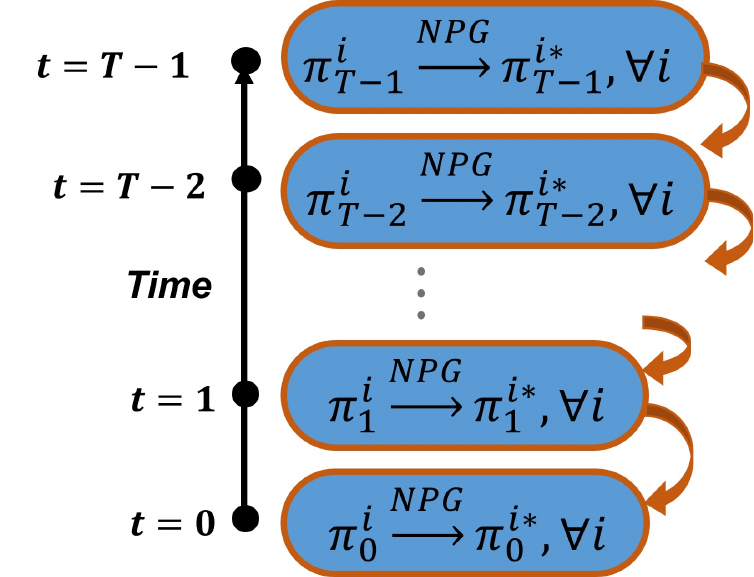}
         \caption{}
          \label{fig:MRNPG_flow_2}
     \end{subfigure}
     \hfill
     \begin{subfigure}[b]{0.4\textwidth}
         \centering
         \includegraphics[width=\textwidth]{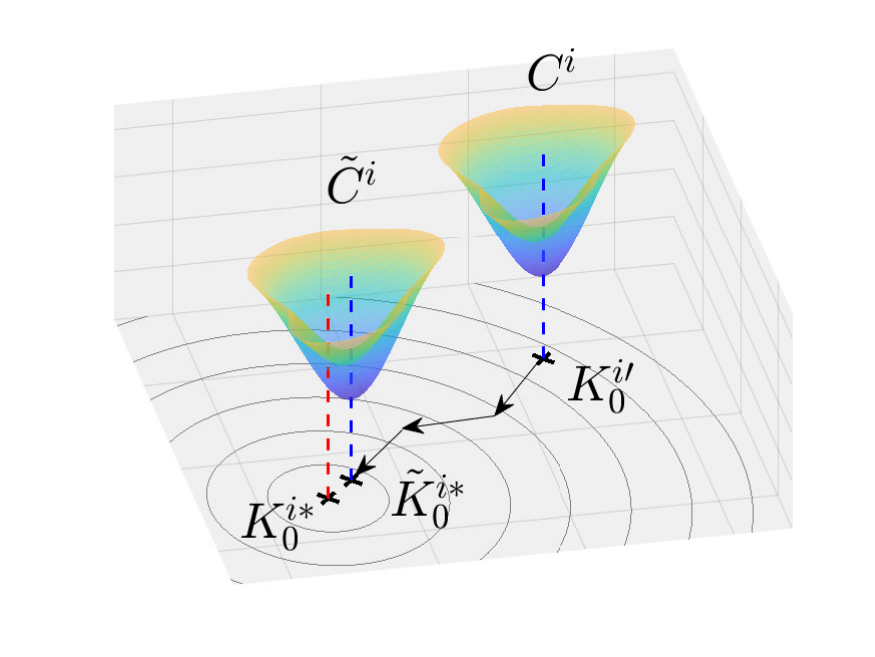}
         \caption{}
         \label{fig:agent_i_costs}
     \end{subfigure}
     \hfill
        \caption{(a) MRNPG Algorithm converges to the NE using Natural Policy Gradient (NPG) starting from $t=T-1$ and moving in a receding horizon manner (backwards-in-time). (b) Cost functions ($C^i$ and $\tilde{C}^i$) of the $i$th agent in a two-shot game ($T=2$). The arrows denote how NPG in the previous timestep $t=1$ converges to the \emph{true} NE.}
        \label{fig:MRNPG}
\end{figure}
}
    
    \noindent \textbf{Now we obtain an upper bound on the term $\lVert \tK^{i*}_t - K^{i*}_t \rVert$} using \eqref{eq:control_NE} and \eqref{eq:control_approx_NE}:
    \begin{align*}
        & \tK^{i*}_t - K^{i*}_t \nonumber \\
        & = -(R^i_t + (B^i_t)^\top P^i_{t+1} B^i_t)^{-1} (B^i_t)^\top P^i_{t+1} \tA^i_t + (R^i_t + (B^i_t)^\top P^{i*}_{t+1} B^i_t)^{-1} (B^i_t)^\top P^{i*}_{t+1} A^{i*}_t \nonumber \\
        & = -(R^i_t + (B^i_t)^\top P^i_{t+1} B^i_t)^{-1} (B^i_t)^\top P^i_{t+1} (\tA^i_t - A^{i*}_t) \nonumber \\
        & \hspace{2cm} - (R^i_t + (B^i_t)^\top P^i_{t+1} B^i_t)^{-1} (B^i_t)^\top (P^i_{t+1} - P^{i*}_{t+1}) A^{i*}_t \nonumber  \\
        & \hspace{2cm} - \big((R^i_t + (B^i_t)^\top P^i_{t+1} B^i_t)^{-1} - (R^i_t + (B^i_t)^\top P^{i*}_{t+1} B^i_t)^{-1} \big) (B^i_t)^\top P^{i*}_{t+1}  A^{i*}_t \nonumber \\
        & = -(R^i_t + (B^i_t)^\top P^i_{t+1} B^i_t)^{-1} (B^i_t)^\top P^i_{t+1} \sum_{j \neq i} B^j_t (\tK^{j*}_t - K^{j*}_t) \nonumber \\
        & \hspace{1cm} - (R^i_t + (B^i_t)^\top P^i_{t+1} B^i_t)^{-1} (B^i_t)^\top (P^i_{t+1} - P^{i*}_{t+1}) A^{i*}_t \nonumber \\
        & \hspace{1cm} - \big((R^i_t + (B^i_t)^\top P^i_{t+1} B^i_t)^{-1} - (R^i_t + (B^i_t)^\top P^{i*}_{t+1} B^i_t)^{-1} \big) (B^i_t)^\top P^{i*}_{t+1}  A^{i*}_t \nonumber \\
        & = -(R^i_t + (B^i_t)^\top P^i_{t+1} B^i_t)^{-1} (B^i_t)^\top P^i_{t+1} \sum_{j \neq i} B^j_t (\tK^{j*}_t - K^{j*}_t) \nonumber \\
        & \hspace{1cm} - (R^i_t + (B^i_t)^\top P^i_{t+1} B^i_t)^{-1} (B^i_t)^\top (P^i_{t+1} - P^{i*}_{t+1}) A^{i*}_t \nonumber \\
        & \hspace{1cm} + (R^i_t + (B^i_t)^\top P^i_{t+1} B^i_t)^{-1} (B^i_t)^{\top} \big( P^i_{t+1} - P^{i*}_{t+1} \big) \nonumber  \\
        & \hspace{3cm} B^i_t (R^i_t + (B^i_t)^\top P^{i*}_{t+1} B^i_t)^{-1} (B^i_t)^{-1} (B^i_t)^\top P^{i*}_{t+1}  A^{i*}_t \label{eq:tK-K*_inter}
    \end{align*}
    where the last equality is due to
    \begin{align*}
        & - (R^i_t + (B^i_t)^\top P^i_{t+1} B^i_t )^{-1} + (R^i_t + (B^i_t)^\top P^{i*}_{t+1} B^i_t )^{-1} \\
        & = - (R^i_t + (B^i_t)^\top P^i_{t+1} B^i_t )^{-1} (R^i_t + (B^i_t)^\top P^{i*}_{t+1} B^i_t ) (R^i_t + (B^i_t)^\top P^{i*}_{t+1} B^i_t )^{-1} \\
        & \hspace{2cm} + (R^i_t + (B^i_t)^\top P^i_{t+1} B^i_t )^{-1} (R^i_t + (B^i_t)^\top P^i_{t+1} B^i_t )(R^i_t + (B^i_t)^\top P^{i*}_{t+1} B^i_t )^{-1} \\
        & = (R^i_t + (B^i_t)^\top P^i_{t+1} B^i_t)^{-1} (B^i_t)^{\top} \big( P^i_{t+1} - P^{i*}_{t+1} \big) B^i_t (R^i_t + (B^i_t)^\top P^{i*}_{t+1} B^i_t)^{-1} (B^i_t)^{-1}
    \end{align*}
    Further analyzing \eqref{eq:tK-K*_inter} we get
    \begin{align}
        & (R^i_t + (B^i_t)^\top P^i_{t+1} B^i_t)(\tK^{i*}_t - K^{i*}_t) + (B^i_t)^\top P^i_{t+1} \sum_{j \neq i} B^j_t (\tK^{j*}_t - K^{j*}_t) \nonumber \\
        & = - (B^i_t)^\top (P^i_{t+1} - P^{i*}_{t+1}) (A^{i*}_t + B^i_t K^{i*}_t) \nonumber \\
        & \hspace{0.4cm} \underbrace{\begin{pmatrix}
            R^1_t + (B^1_t)^\top P^1_{t+1} B^1_t & (B^1_t)^\top P^1_{t+1} B^2_t & \hdots \\
            (B^2_t)^\top P^2_{t+1} B^1_t & R^2_t + (B^2_t)^\top P^2_{t+1} B^2_t & \hdots \\
            \vdots & \vdots & \ddots
        \end{pmatrix}}_{\tilde{\Phi}_{t+1}}
        \begin{pmatrix}
            \tK^{1*}_t - K^{1*}_t \\ \vdots \\ \tK^{N*}_t - K^{N*}_t
        \end{pmatrix} \nonumber \\
        & = 
        \begin{pmatrix}
            (B^1_t)^\top (P^{1*}_{t+1} - P^1_{t+1})(A^{1*}_t + B^1_t K^{1*}_t) \\ \vdots \\ (B^N_t)^\top (P^{N*}_{t+1} - P^N_{t+1})(A^{N*}_t + B^N_t K^{N*}_t)
        \end{pmatrix}  \label{eq:tK-K*}
    \end{align}
    where $\tilde\Phi^{-1}_{t+1}$ is guaranteed to exist as shown below. Using \eqref{eq:tK-K*} we now obtain an upper bound on $\max_{i \in [N]} \lVert \tK^{i*}_t - K^{i*}_t \rVert$:
    \begin{align}
        \max_{j \in [N]} \lVert \tK^{i*}_t - K^{i*}_t \rVert \leq \lVert \tilde\Phi^{-1}_{t+1} \rVert \lVert A^*_t \rVert \max_{j \in[N] } \lVert B^j_t \rVert_\infty \lVert P^{j*}_{t+1} - P^j_{t+1}\rVert_\infty \label{eq:tK-K*_1}
    \end{align}
    We also define
    \begin{align*}
         \hat\Phi_{t+1} := \tilde\Phi_{t+1} - \Phi^*_{t+1} = 
        \begin{pmatrix}
            (B^1_t)^\top (P^1_{t+1} - P^{1*}_{t+1}) \\ \vdots \\ (B^N_t)^\top (P^N_{t+1} - P^{N*}_{t+1})
        \end{pmatrix} \begin{pmatrix}
            B^1_t & \hdots & B^N_t
        \end{pmatrix}
    \end{align*}
    Next we characterize $\lVert \tilde\Phi^{-1}_{t+1} \rVert$:
    \begin{align}
        \lVert \tilde\Phi^{-1}_{t+1} \rVert & = \lVert (I + (\Phi^*_{t+1})^{-1} \hat\Phi_{t+1})^{-1} 
        (\Phi^*_{t+1})^{-1} \rVert \nonumber \\
        & \leq \lVert (\Phi^*_{t+1})^{-1} \rVert \sum_{k=0}^\infty \lVert \lVert (\Phi^*_{t+1})^{-1} \rVert^k \big( \max_{j \in [N]} \lVert B^j_t \rVert^2 \lVert P^i_{t+1} - P^{j*}_{t+1} \rVert \big)^k \nonumber \\
        & \leq 2 \lVert (\Phi^*_{t+1})^{-1} \rVert \leq 2 c^{(-1)}_\Phi \label{eq:tK-K*_2}
    \end{align}
    where the last inequality is possible due to the fact $\max_{j \in [N]} \lVert P^j_{t+1} - P^{j*}_{t+1} \rVert \leq 1/(2 c^{(-1)}_\Phi c^2_B)$ where $c^{(-1)}_\Phi = \max_{t \in \{0,\ldots,T-1\}} \lVert (\Phi^*_{t+1})^{-1} \rVert$ and $c_B = \max_{i \in [N],t \in \{0,\ldots,T-1\}} \lVert B^i_t \rVert$. Hence combining \eqref{eq:tK-K*_1}-\eqref{eq:tK-K*_2},
    \begin{align}
        \max_{j \in [N]} \lVert \tK^{i*}_t - K^{i*}_t \rVert \leq  \bc_1 \max_{j \in [N]}\lVert P^{j*}_{t+1} - P^j_{t+1} \rVert \label{eq:tKi_K*}
    \end{align}
    where { $\bc_1 := 2 c^{(-1)}_\Phi c_B c^*_A $ and $c^*_A = \lVert A^*_t \rVert$}. 
    
    \noindent \textbf{Now we characterize the difference $P^i_t - P^{i*}_t = P^i_t - \tP^i_t + \tP^i_t - P^{i*}_t$.} First we can characterize $\tP^i_t - P^{i*}_t$ using \eqref{eq:Pi*t} and \eqref{eq:tPit}:
    \begin{align*}
        & \tP^i_t - P^{i*}_t = (\tA^i_t)^\top P^i_{t+1} ( \tA^i_t + B^i_t \tK^{i*}_t) - (A^{i*}_t)^\top P^{i*}_{t+1} ( A^{i*}_t + B^i_t K^{i*}_t) \nonumber \\
        & = (\tA^i_t - A^{i*}_t)^\top P^i_{t+1} ( \tA^i_t + B^i_t \tK^{i*}_t) + (A^{i*}_t)^\top (P^i_{t+1} - P^{i*}_{t+1}) (\tA^i_t + B^i_t \tK^{i*}_t) \nonumber \\
        & \hspace{6cm} + (A^{i*}_t)^\top P^{i*}_{t+1} ( \tA^i_t + B^i_t \tK^{i*}_t - A^{i*}_t - B^i_t K^{i*}_t) \nonumber \\
        & = \bigg(\sum_{j \neq i} B^j_t (\tK^{j*}_t - K^{j*}_t) \bigg)^\top P^i_{t+1} ( \tA^i_t + B^i_t \tK^{i*}_t) + (A^{i*}_t)^\top (P^i_{t+1} - P^{i*}_{t+1}) (\tA^i_t + B^i_t \tK^{i*}_t) \nonumber \\
        & \hspace{6cm} + (A^{i*}_t)^\top P^{i*}_{t+1} \sum_{j =1}^N B^j_t (\tK^{j*}_t - K^{j*}_t)
    \end{align*}
    Using this characterization we bound $\lVert \tP^i_t - P^{i*}_t \rVert$ using the AM-GM, $\lVert \tP^i_t - P^{i*}_t \rVert$
    \begin{align}
        & \leq 2 \lVert A^*_t \rVert \lVert P^{i*}_{t+1} \rVert c_B \sum_{j=1}^N \lVert \tK^{j*}_t - K^{j*}_t \rVert + \frac{c^4_B}{2} \bigg( \sum_{j=1}^N \lVert \tK^{j*}_t - K^{j*}_t \rVert \bigg)^4 \label{eq:tPi_Pi*} \\
        &  + \big(\lVert A^*_t \rVert/2 + \lVert P^{i*}_{t+1} \rVert c_B + \lVert A^{i*}_t \rVert/2 \big) c_B \bigg(\sum_{j=1}^N \lVert \tK^{j*}_t - K^{j*}_t \rVert \bigg)^2  \nonumber \\
        & + \big(\lVert A^*_t \rVert/2 + 1/2 + \lVert A^{i*}_t \rVert/2 \big) \lVert P^i_{t+1} - P^{i*}_{t+1} \rVert^2 + \lVert A^{i*}_t \rVert \lVert A^*_t \rVert \lVert P^i_{t+1} - P^{i*}_{t+1} \rVert \nonumber \\
        & \leq \underbrace{\big( 2 c^*_A c^*_P c_B + (c^*_A /2 + c^*_P c_B + c^{i*}_A/2 ) c_B + c^4_B/2 \big)}_{\bc_2} \sum_{j=1}^N \lVert \tK^{j*}_t - K^{j*}_t \rVert \nonumber \\
        & \hspace{4cm} + \underbrace{c^*_A/2 + 1/2 + c^{i*}_A/2 + c^{i*}_A c^*_A}_{\bc_3} \lVert P^i_{t+1} - P^{i*}_{t+1} \rVert \nonumber \\
        & \leq \bc_2 N \max_{j \in [N]} \lVert \tK^{j*}_t - K^{j*}_t \rVert + \bc_3 \lVert P^i_{t+1} - P^{i*}_{t+1} \rVert \nonumber 
    \end{align}
    where { $c^{i*}_A := \max_{t \in \{0,\ldots,T-1\}} \lVert A^{i*}_t \rVert, c^*_P := \max_{i \in [N], t \in \{0,\ldots,T-1\}} \lVert P^{i*}_{t+1} \rVert$}, and the last inequality follows from the fact that { $\lVert P^i_{t+1} - P^{i*}_{t+1} \rVert, \lVert \tK^{j*}_t - K^{j*}_t \rVert \leq 1/N$}. Similarly $P^i_t - \tP^i_t$ can be decomposed using \eqref{eq:tPit} and \eqref{eq:Pit}:
    \begin{align*}
        P^i_t - \tP^i_t & = (K^i_t)^\top R^i_t K^i_t + \bigg(A_t + \sum_{j=1}^N B^j_t K^j_t \bigg)^\top P^i_{t+1} \bigg(A_t + \sum_{j=1}^N B^j_t K^j_t \bigg) \\
        & \hspace{1cm} - \bigg[ (\tK^{i*}_t)^\top R^i_t \tK^{i*}_t + \bigg(A_t + \sum_{j=1}^N B^j_t \tK^{i*}_t \bigg)^\top P^i_{t+1} \bigg(A_t + \sum_{j=1}^N B^{i}_t \tK^{j*}_t \bigg) \bigg].
    \end{align*}
    We start by analyzing the quadratic form $x^\top P^i_t x$:
    \begin{align*}
        & x^\top P^i_t x = x^\top \Bigg[ Q^i_t + (K^i_t)^\top R^i_t K^i_t + \bigg(A_t + \sum_{j=1}^N B^j_t K^j_t \bigg)^\top P^i_{t+1} \bigg(A_t + \sum_{j=1}^N B^j_t K^j_t \bigg) \bigg] x \\
        & = x^\top \bigg[ (K^i_t)^\top \big(R^i_t + (B^i_t)^\top P^i_{t+1} B^i_t \big) K^i_t + 2 (B^i_t K^i_t)^\top P^i_{t+1} \bigg( A_t + \sum_{j \neq i} B^j_t K^j_t \bigg) + Q^i_t \\
        & \hspace{4cm} + \bigg( A_t + \sum_{j \neq i} B^j_t K^j_t \bigg)^\top P^i_{t+1} \bigg( A_t + \sum_{j \neq i} B^j_t K^j_t \bigg)\bigg] x \\
        & = x^\top \bigg[ (K^i_t)^\top \big(R^i_t + (B^i_t)^\top P^i_{t+1} B^i_t \big) K^i_t + 2 (B^i_t K^i_t)^\top P^i_{t+1} \bigg( A_t + \sum_{j \neq i} B^j_t \tK^{j*}_t \bigg) + Q^i_t \\
        & \hspace{0.4cm} + 2 (B^i_t K^i_t)^\top P^i_{t+1} \sum_{j \neq i} B^j_t (K^j_t - \tK^{j*}_t)  + \bigg( A_t + \sum_{j \neq i} B^j_t \tK^{j*}_t \bigg)^\top P^i_{t+1} \bigg( A_t + \sum_{j \neq i} B^j_t \tK^{j*}_t \bigg) \\
        & \hspace{0.4cm} + \bigg( \sum_{j \neq i} B^j_t (K^j_t - \tK^{j*}_t) \bigg)^\top P^i_{t+1} \bigg( \sum_{j \neq i} B^j_t (K^j_t - \tK^{j*}_t) \bigg) \\
        & \hspace{0.4cm} + 2 \bigg( A_t + \sum_{j \neq i} B^j_t \tK^{j*}_t \bigg) P^i_{t+1} \bigg( \sum_{j \neq i} B^j_t (K^j_t - \tK^{j*}_t) \bigg) \bigg] x
    \end{align*}
    Completing squares, we get
    \begin{align}
        \label{eq:xPix} & x^\top P^i_t x = x^\top \Bigg[ (K^i_t - \tK^{i*}_t)^\top \big(R^i_t + (B^i_t )^\top P^i_{t+1} B^i_t \big) (K^i_t - \tK^{i*}_t) \\
        & \hspace{0.4cm} + \bigg( A_t + \sum_{j \neq i} B^j_t \tK^{j*}_t \bigg)^\top P^i_{t+1} \bigg( A_t + \sum_{j \neq i} B^j_t \tK^{j*}_t \bigg) + Q^i_t \nonumber \\
        & \hspace{0.4cm} - \bigg( A_t + \sum_{j \neq i} B^j_t \tK^{j*}_t \bigg)^\top P^i_{t+1} B^i_t  \big(R^i_t + (B^i_t )^\top P^i_{t+1} B^i_t \big)^{-1} \nonumber \\
        & \hspace{5cm}(B^i_t)^\top P^i_{t+1} \bigg( A_t + \sum_{j \neq i} B^j_t \tK^{j*}_t \bigg) \nonumber 
    \end{align}
    \begin{align*}
        & \hspace{0.4cm} + \bigg(2 \bigg(A_t + \sum_{j = 1}^N B^j_t \tK^{j*}_t \bigg) + \sum_{j=1}^N B^j_t (K^j_t - \tK^{j*}_t ) \bigg)^\top \nonumber  \\
        & \hspace{5cm}P^i_{t+1} \bigg( \sum_{j \neq i} B^j_t (K^j_t - \tK^{j*}_t \bigg) \bigg] x. \nonumber
    \end{align*}
    Now we take a look at the quadratic $ x^\top \tP^i_t x$:
    \begin{align}
        & x^\top \tP^i_t x \nonumber \\
        & = x^\top \bigg[ Q^i_t + (\tK^{i*}_t)^\top R^i_t \tK^{i*}_t + \bigg(A_t + \sum_{j=1}^N B^j_t \tK^{j*}_t \bigg)^\top P^i_{t+1} \bigg(A_t + \sum_{j=1}^N B^{j}_t \tK^{j*}_t \bigg) \bigg] x \nonumber \\ 
        & = x^\top \bigg[ (\tK^{i*}_t)^\top \big( R^i_t + (B^i_t)^\top P^i_{t+1} B^i_t \big) \tK^{i*}_t + 2 (B^i_t \tK^{i*}_t)^\top P^i_{t+1} \bigg(A_t + \sum_{j \neq i} B^j_t \tK^{j*}_t \bigg) \nonumber \\
        & \hspace{0.4cm} + Q^i_t + \bigg(A_t + \sum_{j \neq i} B^j_t \tK^{j*}_t \bigg)^\top P^i_{t+1} \bigg(A_t + \sum_{j \neq i}^N B^{j}_t \tK^{j*}_t \bigg) \bigg] x \nonumber  \\
        \label{eq:xtPix}  & = x^\top \bigg[ \bigg(A_t + \sum_{j \neq i} B^j_t \tK^{j*}_t \bigg)^\top P^i_{t+1} \bigg(A_t + \sum_{j \neq i}^N B^{j}_t \tK^{j*}_t \bigg) + Q^i_t \\
        & \hspace{0.4cm}  - \bigg( A_t + \sum_{j \neq i} B^j_t \tK^{j*}_t \bigg)^\top P^i_{t+1} B^i_t  \big(R^i_t + (B^i_t )^\top P^i_{t+1} B^i_t \big)^{-1} \nonumber \\
        & \hspace{6cm}(B^i_t)^\top P^i_{t+1} \bigg( A_t + \sum_{j \neq i} B^j_t \tK^{j*}_t \bigg) \bigg] x. \nonumber 
    \end{align}
    Using \eqref{eq:xPix} and \eqref{eq:xtPix}, we get
    \begin{align*}
        & x^\top (P^i_t - \tP^i_t) x = x^\top \Bigg[ (K^i_t - \tK^{i*}_t)^\top \big(R^i_t + (B^i_t )^\top P^i_{t+1} B^i_t \big) (K^i_t - \tK^{i*}_t) \nonumber \\
        & \hspace{0.4cm} + \bigg(2 \bigg(A_t + \sum_{j = 1}^N B^j_t \tK^{j*}_t \bigg) + \sum_{j=1}^N B^j_t (K^j_t - \tK^{j*}_t ) \bigg)^\top P^i_{t+1} \bigg( \sum_{j \neq i} B^j_t (K^j_t - \tK^{j*}_t \bigg) \bigg] x \\
        & = x^\top \Bigg[ (K^i_t - \tK^{i*}_t)^\top \big(R^i_t + (B^i_t )^\top P^i_{t+1} B^i_t \big) (K^i_t - \tK^{i*}_t) \nonumber \\
        & \hspace{0.4cm} + \bigg(2 \bigg(A_t + \sum_{j = 1}^N B^j_t K^{j*}_t \bigg) + 2 \sum_{j=1}^N B^j_t (\tK^{j*}_t - K^{j*}_t ) + \sum_{j=1}^N B^j_t (K^j_t - \tK^{j*}_t ) \bigg)^\top \\
        & \hspace{7cm} P^i_{t+1} \bigg( \sum_{j \neq i} B^j_t (K^j_t - \tK^{j*}_t \bigg) \bigg] x
    \end{align*}
    Using this characterization, we bound $\lVert P^i_t - \tP^i_t \rVert$:
    \begin{align*}
        & \lVert P^i_t - \tP^i_t \rVert \leq \big( \lVert R^i_t + (B^i_t)^\top P^{i*}_{t+1} B^i_t \rVert + \lVert (B^i_t)^\top \big( P^{i*}_{t+1} - P^{i*}_{t+1} \big) B^i_t \rVert \big) \lVert K^i_t - \tK^{i*}_t \rVert
    \end{align*}
    \begin{align*}
        & + (2 c^{i*}_A + 2 c_B \sum_{j=1}^N \big( \lVert \tK^{j*}_t - K^*_t \rVert + \lVert K^j_t - \tK^{j*}_t \rVert \big) \big( \lVert P^{i*}_{t+1} \rVert + \lVert P^i_{t+1} - P^{i*}_{t+1} \rVert \big) \\
        & \hspace{7cm}c_B \sum_{j=1}^N \lVert K^j_t - \tK^{j*}_t \rVert
    \end{align*}
    As before, assuming { $\lVert P^i_{t+1} - P^{i*}_{t+1} \rVert, \lVert \tK^{j*}_t - K^{j*}_t \rVert \leq 1/N$},
    \begin{align}
        \lVert P^i_t - \tP^i_t \rVert & \leq \underbrace{\big( \bc^* + c^2_B/2 + 2 c^2_B (c^*_P + 1) + 2 c^{i*}_A \big)}_{\bc_4} \sum_{j=1}^N \lVert K^j_t - \tK^{j*}_t \rVert  \nonumber \\
        & \hspace{1cm} + \underbrace{\big( c^2_B/2 + c^{i*}_A + c_B \big)}_{\bc_5} \lVert P^i_{t+1} - P^{i*}_{t+1} \rVert + \underbrace{c_B (c^*_P + 1)}_{\bc_6} \sum_{j=1}^N \lVert \tK^{j*}_t - K^{j*}_t \rVert \nonumber \\
        & \leq \bc_4 N \Delta_t + \bc_5 \lVert P^i_{t+1} - P^{i*}_{t+1} \rVert + \bc_6 \sum_{j=1}^N \lVert \tK^{j*}_t - K^{j*}_t \rVert \label{eq:Pi_tPi}
    \end{align}
    where { $\bc^* := \max_{i \in [N], t \in \{0,\ldots,T-1\}} \lVert R^i_t + (B^i_t)^\top P^{i*}_{t+1} B^i_t \rVert$}. Let us define  $e^K_t := \max_{j \in [N]} \lVert K^j_t - K^{j*}_t \rVert, e^P_t := \max_{j \in [N]} \lVert P^j_t - P^{j*}_t \rVert$. Using \eqref{eq:tKi_K*}, \eqref{eq:tPi_Pi*} and \eqref{eq:Pi_tPi} we get
    \begin{align*}
        e^K_t & \leq \bc_1 e^P_{t+1} + \Delta_t \\
        e^P_t & \leq (\bc_2 + \bc_6) N e^K_t + (\bc_3 + \bc_5) e^P_{t+1} + \bc_4 N \Delta \\
        & \leq \underbrace{(\bc_1 (\bc_2 + \bc_6) N + \bc_3 + \bc_5)}_{\bc_7} e^P_{t+1} + \underbrace{(\bc_2 + \bc_4 + \bc_6)N}_{\bc_8} \Delta_t
    \end{align*}
    Using this recursive definition, we deduce
    \begin{align*}
        e^P_t & \leq \bc^{T-t}_7 e^P_T + \bc_8 \sum_{s=0}^{T-1} \bc^s_7  \Delta_{t+s} = \bc_8 \sum_{s=0}^{T-1} \bc^s_7  \Delta_{t+s}
    \end{align*}
    Hence if $\Delta = \Os(\epsilon)$, in particular $\Delta_t \leq \epsilon/(2 \bc_1 \bc^t_7 \bc_8 T)$, then $e^P_t \leq \epsilon/2 \bc_1$ and $e^K_t \leq \epsilon$ for $t \in \{ 0, \ldots, T-1 \}$.

\bibliographystyle{unsrt}
\bibliography{references}

\newpage

\begin{center}
    \Large \bf Supplementary Materials
\end{center}

\section{Adapted Open-Loop Analysis of GS-MFTG}
\label{sec:open_loop_analysis}
In this section we characterize the Adapted Open-Loop Nash equilibrium (in short OLNE \cite{moon2018risk}), which is shown to satisfy the necessary conditions for Nash Equilibrium (Theorem \ref{thm:OLNE}). The structure of the OLNE inspires a  decomposition which simplifies the computation of the Nash equilibrium (NE) and enables us to establish its existence and uniqueness (Theorem \ref{thm:CLNE}). The required conditions for this result are  similar to those of LQ Games. We consider a set of controls referred to as \emph{adapted open-loop} $\Us^i_o$ \cite{moon2018risk} where control action at time $t$, $u^i_t$, is adapted to the noise process $\{\omega^0_0, \omega_0, \dots, \omega^0_t, \omega_t \}$. If $\Us^i = \Us^i_o$ then the corresponding NE is called \emph{Adapted Open-Loop Nash Equilibrium} (in short OLNE). Below we characterize the OLNE and then prove that the OLNE is a broader class of NE than the NE. This allows a decomposition of the GS-MFTG which simplifies the characterization of the NE. This result is distinct from other results in literature such as LQ-MFGs and Zero-Sum MFTGs, due to the general-sum competitive-cooperative nature of the problem.

\begin{theorem} \label{thm:OLNE}
(1) All OLNE policies $(\mcu^{i*})_{i \in [N]}$ are linear in the adjoint processes $(p^i_t)_{i \in [N], 0 \leq t < T}$ \eqref{eq:FB} and \eqref{eq:FB_inter},
\begin{align*}
	\bu^{i*}_t = - (L^{i}_t + \bL^{i}_t) \bp^i_t, \hspace{0.2cm} u^{i*}_t = - L^{i}_t p^i_t + \bL^{i}_t \bp^i_t,
\end{align*}
where $\forall i \in [N], t \in \{0,\ldots.T-1\}$ the adjoint processes depend on solutions of Riccati equations \eqref{eq:OLNE_Riccati_bar} and \eqref{eq:OLNE_Riccati}. 

(2) Furthermore, the OLNE's feedback representation is linear in $(x_t - \bx_t)$ and $\bx_t$,
\begin{align} \label{eq:OLNE_back_rep}
	u^{i}_t = -L^i_t P^i_t(x_t - \bx_t) - (L^i_t + \bL^i_t) \bP^i_t \bx_t,
\end{align}
$\forall i \in [N], t \in \{0,\ldots.T-1\}$.
 
(3) Finally, every OLNE satisfies the necessary conditions for NE.
\end{theorem}
In the proof we first develop necessary conditions for OLNE using the Stochastic Maximum Principle and these conditions are shown to be sufficient due to the quadratic structure of the cost. The OLNE is then shown to satisfy the necessary conditions for NE. This fact coupled with the feedback structure of OLNE \eqref{eq:OLNE_back_rep} suggests that the NE will be a function of the state deviation from mean-field, $(x_t - \bx_t)$, and the mean-field $\bx_t$. 

\begin{proof}
We will use the Stochastic Minimum Principle to first characterize the necessary conditions for the OLNE in terms of an adjoint process. We then study the adjoint process and provide sufficient conditions for its  existence and uniqueness in terms of solutions to a set of Riccati equations. We then prove that the necessary conditions for the OLNE are also sufficient due to the quadratic nature of the cost. This analysis is a generalization of the Zero-Sum MFTG adapted open-loop analysis in \cite{carmona2021linear} to the General-Sum setting. Additionally we show that OLNE satisfies the necessary conditions of the NE, by first characterizing the necessary conditions for NE and then showing that every OLNE satisfies these conditions.

\textbf{We first characterize the necessary conditions for OLNE.} Let us write down the expression for Adapted Open-Loop Nash Equilibrium (OLNE). An OLNE is a tuple of policies $(\mcu^{i*})_{i \in [N]}$ (where each $\mcu^{i*} \in \Us^i_o$ is measurable with respect to the noise process) such that,
	\begin{align*}
		J^i(\mcu^{i*},\mcu^{-i*}) \leq J^i (\mcu^{i},\mcu^{-i*}), \hspace{0.2cm} \forall \mcu^i \in \Us^i_o.
	\end{align*}
First we find the equilibrium condition of the OLNE. Let us first define $\zeta_t = (x_t, \bx_t, (u^i_t, \bu^i_t)_{i \in [N]})$ and
\begin{align*}
	b_t(\zeta_t) = (A_t - I)x_t + \bA_t \bx_t + \sum_{j=1}^N (B^i_t u^i_t + \bB^i_t \bu^i_t) = x_{t+1} -x_t -\omega^0_t - \omega_t
\end{align*}
for $t \in \{0,\ldots, T-1\}$. Now we introduce the adjoint process $p^i = (p^i_t)_{0 \leq t < T}$ which is an $\Fs_t$-adapted process and define-
\begin{align*}
	& h^i_t(\zeta_t, p^i_t) = b_t(\zeta_t)^\top p^i_t + c^i_t(\zeta_t) \\
	& = b_t(\zeta_t)^\top p^i_t + (x_t - \bx_t)^\top Q^i_t (x_t - \bx_t) + \bx_t^\top \bQ^i_t \bx_t + (u^i_t - \bu^i_t)^\top R^i_t (u^i_t - \bu^i_t) + (\bu^i_t)^\top \bR^i_t \bu^i_t
\end{align*}
Now we evaluate the form of derivatives of $h^i_t$ and convexity of $h^i_t$ with respect to $(u^i, \bu^i)$.
\begin{lemma} \label{lem:part_der}
	If $R^i_t \succ 0$ and $\bR^i_t \succ 0$, then $h^i_t$ is convex with respect to $(u^i, \bu^i)$ and
	\begin{align*}
		\delta_{x_t} h^i_t(\zeta_t, p^i_t) & = (p^i_t)^\top (A_t - I) + 2 (x_t - \bx_t )^\top Q^i_t, \\
		\delta_{\bx_t} h^i_t(\zeta_t, p^i_t) & = (p^i_t)^\top \bA_t - 2 (x_t - \bx_t)^\top Q^i_t + 2 \bx^\top_t \bQ^i_t, \\
		\delta_{u^i_t} h^i_t(\zeta_t, p^i_t) & = (p^i_t)^\top B^i_t + 2(u^i_t - \bu^i_t)^\top R^i_t, \\
		\delta_{u^j_t} h^i_t(\zeta_t, p^i_t) & = (p^i_t)^\top B^j_t, \hspace{0.4cm} j \neq i,\\
		\delta_{\bu^i_t} h^i_t(\zeta_t, p^i_t) & = (p^i_t)^\top \bB^i_t - 2(u^i_t - \bu^i_t)^\top R^i_t + 2 (\bu^i_t)^\top \bR^i_t, \\
		\delta_{\bu^j_t} h^i_t(\zeta_t, p^i_t) & = (p^i_t)^\top \bB^j_t,  \hspace{0.4cm} j \neq i.
	\end{align*}
\end{lemma}
\begin{proof}
	The partial derivatives follow by direct computation and the convexity is due to the fact that
	\begin{align*}
		\delta^2_{(u^i_t, \bu^i_t),(u^i_t, \bu^i_t)} h^i_t(\zeta_t, p^i_t) = \begin{pmatrix}
			2 R^i_t & -2 R^i_t \\ -2 R^i_t & \hspace{0.1cm} 2(R^i_t + \bR^i_t)
		\end{pmatrix} \succ 0.
	\end{align*}
 This completes the proof.
\end{proof}
The convexity property will be used later to characterize the sufficient conditions for the OLNE. We hypothesize that the adjoint process for each player follows the backwards difference equations
\begin{align} \label{eq:costate}
	p^i_t = \EE \big[ A^\top_{t+1} p^i_{t+1} + \bA^\top_{t+1} \bp^i_{t+1} + 2 Q^i_{t+1} (x_{t+1} - \bx_{t+1})  + 2 \bQ^i_{t+1} \bx_{t+1} | \Fs_t \big]
\end{align}
where $\bp^i_t = \EE[p^i_t | (\omega^0_s)_{0 \leq s \leq t}]$. This hypothesis will shown to be true while characterizing the necessary conditions for OLNE. The first step towards obtaining the necessary conditions for OLNE is to compute the Gateaux derivative of the cost with respect to perturbation in just the control policy of the $i^{th}$ agent itself. The Gateux derivative is shown to be a function of the adjoint process.
\begin{lemma} \label{lem:Gateaux}
If the adjoint process is defined as in \eqref{eq:costate}, then the Gateaux derivative of $J^i$ in the direction of $\beta^i \in \Us^i_o$ is
\begin{align*}
	\Ds J^i(\mcu^i, \mcu^{-i})(\beta^i,0) = \EE \sum_{t=0}^{T-1} \big[ ((p^i_t)^\top B^i_t + 2(u^i_t - \bu^i_t)^\top R^i_t + (\bp^i_t)^\top \bB^i_t + 2 (\bu^i_t)^\top \bR^i_t)^\top \beta^i_t \big]
\end{align*}
\end{lemma}
\begin{proof}
    We compute the Gateaux derivative by first introducing a perturbation in the control sequence of agent $i$, which results in the perturbed control $u^{i\prime}$. Then the Gateaux derivative is computed by comparing the costs under the perturbed and unperturbed control sequences.

	Let us denote for each $i \in [N]$ the control sequence $u^i = (u^i_t)_{t \geq 0}$ and the perturbed control sequence $u^{i\prime} = (u^{i\prime}_t)_{t \geq 0}$, and the corresponding state processes $x_t$ and $x'_t$ under the same noise process and write
	\begin{align*}
		x_{t+1} - x'_{t+1} = A_t (x_t - x'_t) + \bA_t (\bx_t - \bx'_t) + \sum_{i=1}^N \big( B^i_t (u^i_t - u^{i\prime}_t) +  \bB^i_t (\bu^i_t - \bu^{i\prime}_t ) \big).
	\end{align*}
	where $\bu^i_t = \EE[u^i_t | \Fs^0]$ and $\bu^{i \prime}_t = \EE[u^{i \prime}_t | \Fs^0]$. Let us introduce $\zeta_t = (x_t, \bx_t, (u^i_t, \bu^i_t)_{i \in [N]})$ and $\zeta'_t = (x'_t, \bx'_t, (u^{i\prime}_t, \bu^{i\prime}_t)_{i \in [N]})$. Next let us compute the difference between the costs,
	\begin{align}
		\sum_{t=0}^{T-1} (c^i_t(\zeta'_t) - c^i_t(\zeta_t)) & = \sum_{t=0}^{T-1} \big( c^i_t(\zeta'_t) - c^i_t(\zeta_t) + (b_t(\zeta'_t) - b_t(\zeta_t))^\top p^i_t - (b_t(\zeta'_t) - b_t(\zeta_t))^\top p^i_t  \big), \nonumber \\
		& = \sum_{t=0}^{T-1} \big( h^i_t(\zeta'_t,p^i_t) - h^i_t(\zeta_t,p^i_t) - (x'_{t+1} -x_{t+1} )^\top p^i_t + (x'_t -x_t)^\top p^i_t  \big) \nonumber \\
		& = \sum_{t=0}^{T-2} \big( h^i_t(\zeta'_t,p^i_t) - h^i_t(\zeta_t,p^i_t)  + (x'_{t+1} -x_{t+1} )^\top (p^i_{t+1} - p^i_t)  \big)+ \nonumber \\
		& \hspace{1cm} h^i_{T-1}(\zeta'_{T-1},p^i_{T-1}) - h^i_{T-1}(\zeta_{T-1},p^i_{T-1}) + (x'_T -x_T)^\top p^i_{T-1}  \label{eq:cost_diff}
	\end{align}
	Let us denote the perturbation in $u^{i \prime}$ as the $\Fs_t$-adapted stochastic process $\beta^i_t$ scaled by a constant $\epsilon$ such that $u^{i\prime} = u^i + \epsilon \beta^i$ for all $i \in [N]$ (consequently $\bu^{i\prime} = \bu^i + \epsilon \EE[\beta^i|\Fs^0]$). In order to compute the Gateaux derivative we define the infinitesimal change in state process and the mean-field as
	\begin{align}
		V_t = \lim_{\epsilon \rightarrow 0} \frac{1}{\epsilon} (x'_t - x_t), \hspace{0.2cm} \bV_t = \lim_{\epsilon \rightarrow 0} \frac{1}{\epsilon} (\bx'_t - \bx_t) \label{eq:V_T}
	\end{align}
	Now we compute the Gateaux derivative of $J^i$ in the direction of $\beta^i \in \Us^i_o$. By setting $p^i_T = 0$ we get the simplified expression
	\begin{align}
		& \Ds J^i(u^i, u^{-i})(\beta^i,0) \nonumber\\
		& = \EE \sum_{t=0}^{T-1} \big[ V^\top_{t+1} (p^i_{t+1} - p^i_t) + \delta_{x_t} h^i_t(\zeta_t, p^i_t)^\top V_t + \delta_{\bx_t} h^i_t(\zeta_t, p^i_t)^\top \bV_t \nonumber \\
        & \hspace{6cm} + \delta_{u^i_t} h^i_t(\zeta_t, p^i_t)^\top \beta^i_t + \delta_{\bu^i_t} h^i_t(\zeta_t, p^i_t)^\top \bbeta^i_t \big]  \nonumber \\
		& = \EE \sum_{t=0}^{T-1} \big[ V^\top_{t+1} (p^i_{t+1} - p^i_t) + ((p^i_t)^\top (A_t - I) + 2 (x_t - \bx_t )^\top Q^i_t)^\top V_t   \nonumber\\
		& \hspace{1cm} + ((p^i_t)^\top \bA_t - 2 (x_t - \bx_t)^\top Q^i_t + 2 \bx^\top_t \bQ^i_t)^\top \bV_t + ((p^i_t)^\top B^i_t + 2(u^i_t - \bu^i_t)^\top R^i_t)^\top \beta^i_t \nonumber \\
        & \hspace{1cm} + ((p^i_t)^\top \bB^i_t - 2(u^i_t - \bu^i_t)^\top R^i_t + 2 (\bu^i_t)^\top \bR^i_t)^\top \bbeta^i_t \big]. \label{eq:costate_inter}
	\end{align}
	Next using techniques similar to \cite{carmona2020policy} we deduce that
	\begin{align*}
		((p^i_t)^\top \bA_t - 2 (x_t - \bx_t)^\top Q^i_t + 2 \bx^\top_t \bQ^i_t)^\top \bV_t = ((\bp^i_t)^\top \bA_t + 2 \bx^\top_t \bQ^i_t)^\top V_t.
	\end{align*}
	Using the definition of the adjoint process \eqref{eq:costate} we can simplify the first three terms in \eqref{eq:costate_inter} to be
	\begin{align}
		& \EE \sum_{t=0}^{T-1} \big[ V^\top_{t+1} (p^i_{t+1} - p^i_t) + ((p^i_{t+1})^\top (A_{t+1} - I) + (\bp^i_{t+1})^\top \bA_{t+1}  \nonumber \\
        & \hspace{3cm} + 2 (x_{t+1} - \bx_{t+1} )^\top Q^i_{t+1} + 2 \bx^\top_{t+1} \bQ^i_{t+1})^\top V_{t+1} \big] = 0. \label{eq:costate_inter_1}
	\end{align}
	Next using techniques similar to \cite{carmona2020policy} we can also deduce that
	\begin{align}
		((p^i_t)^\top \bB^i_t - 2(u^i_t - \bu^i_t)^\top R^i_t + 2 (\bu^i_t)^\top \bR^i_t)^\top \bbeta^i_t = ((\bp^i_t)^\top \bB^i_t + 2 (\bu^i_t)^\top \bR^i_t)^\top \beta^i_t. \label{eq:costate_inter_2}
	\end{align}
	Using \eqref{eq:costate_inter}-\eqref{eq:costate_inter_2} we obtain
	\begin{align*}
		 \Ds J^i(\mcu^i, \mcu^{-i})(\beta^i,0) = \EE \sum_{t=0}^{T-1} \big[ ((p^i_t)^\top B^i_t + 2(u^i_t - \bu^i_t)^\top R^i_t + (\bp^i_t)^\top \bB^i_t + 2 (\bu^i_t)^\top \bR^i_t)^\top \beta^i_t \big]
	\end{align*}
\end{proof}
The necessary conditions for OLNE requires stationarity of the Gateaux derivative. Now using this condition and the form of the Gateaux deriative we state the first order necessary conditions for the OLNE the NE of the $N$-player LQ-MFTGs.
\begin{proposition} [OLNE Necessary Conditions]\label{prop:nec_cond_OLNE}
	If the set of policies $(u^{i*})_{i \in [N]}$ constitutes OLNE, then for all $i \in [N]$ and $t \in \{0,\ldots,T-1\}$,
	\begin{align*}
		(p^i_t)^\top B^i_t + 2(u^{i*}_t - \bu^{i*}_t)^\top R^i_t + (\bp^i_t)^\top \bB^i_t + 2 (\bu^{i*}_t)^\top \bR^i_t = 0
	\end{align*}
	where $	p^i_t = \EE \big[ A^\top_{t+1} p^i_{t+1} + \bA^\top_{t+1} \bp^i_{t+1} + 2 Q^i_{t+1} (x_{t+1} - \bx_{t+1})  + 2 \bQ^i_{t+1} \bx_{t+1} | \Fs_t \big]$.
\end{proposition}
\begin{proof}
	Let us fix  $i \in [N]$. Then using the necessary conditions for OLNE and Lemma \ref{lem:Gateaux}, we have
	\begin{align*}
	& \Ds J^i(u^i, u^{-i})(\beta^i,0) \\
        & = \EE \sum_{t=0}^{T-1} \big[ ((p^i_t)^\top B^i_t + 2(u^i_t - \bu^i_t)^\top R^i_t + (\bp^i_t)^\top \bB^i_t + 2 (\bu^i_t)^\top \bR^i_t)^\top \beta^i_t \big] = 0,
	\end{align*}
	for any $0 \neq \beta^i \in \Us^i$, as this statement has to be true for any perturbation $\beta^i$ each summand has to be equal to $0$ which results in the statement of the theorem.
\end{proof}
Using the necessary conditions for OLNE we now identify the form of the OLNE policies of the $N$-player LQ-MFTG.
\begin{proposition}
	The Adapted Open-Loop Nash Equilibrium (OLNE) policies of the $N$-player LQ-MFTG has the form and the precise expression given by
	\begin{align}
		u^{i*}_t = - L^{i}_t p^i_t - \bL^{i}_t \bp^i_t, \hspace{0.2cm} \bu^{i*}_t = - (L^{i}_t + \bL^{i}_t) \bp^i_t \label{eq:OL_eq_control}
	\end{align}
	where
	\begin{align}
		L^{i}_t = \frac{1}{2}(R^i_t)^{-1} (B^i_t)^\top, \hspace{0.2cm} \bL^{i}_t = \frac{1}{2} (R^i_t)^{-1} \big( R^i_t (\bR^i_t)^{-1}(B^i_t + \bB^i_t)^\top - (B^i_t)^\top \big) \label{eq:L_i_bL_i}
	\end{align}
    Furthermore, the feedback representation of the OLNE is
    \begin{align*}
        u^{i}_t = - L^{i}_t P^i_t (x_t - \bx_t) - (L^{i}_t + \bL^{i}_t) \bP^i_t \bx_t, \hspace{0.2cm}  \bu^{i}_t = - (L^{i}_t + \bL^{i}_t) \bP^i_t \bx_t
    \end{align*}
    where the matrices $P^i_t$ and $\bP^i_t$ are computed recursively using \eqref{eq:OLNE_Riccati_1}-\eqref{eq:OLNE_Riccati_2}.
\end{proposition}
\begin{proof}
    In this proof the form of the OLNE control policies is computed by utilizing the necessary conditions in Proposition \ref{prop:nec_cond_OLNE}.Then by introducing deterministic processes and reformulating the adjoint process the necessary conditions are transformed into a set of forward-backwards equations \eqref{eq:FB}. Moreover, these forward-backwards equations are shown to be equivalent to a set of Riccati equations \eqref{eq:OLNE_Riccati_1}-\eqref{eq:OLNE_Riccati_2}. Finally utilizing these Riccati equations we formalize the feedback representation of the OLNE.

	Let us define a set of deterministic processes,
	\begin{align}
		Z^{0,i}_t = (A^\top_t + \bA^\top_t) \bP^i_t + 2 \bQ^i_t, \hspace{0.2cm} Z^{1,i}_t = A^\top_t P^i_t + 2 Q^i_t, \label{eq:deter_adj}
	\end{align}
	where the matrices $\bP^i_t$ and $P^i_t$ are defined as follows:
	\begin{align} 
		\bP^i_t & = \big((A_t + \bA_t) \bP^i_{t+1} + 2 \bQ^i_t \big) \big( A_t + \bA_t - \sum_{j=1}^N (B^j_t + \bB^j_t)(L^{j}_t + \bL^{j}_t) \bP^j_{t} \big), \label{eq:OLNE_Riccati_1} \\
		P^i_t & = (A^\top P^i_{t+1} + 2Q^i_t) (A_t - \sum_{j=1}^N B^j_t L^{j}_t P^j_{t}).\label{eq:OLNE_Riccati_2}
	\end{align}
	Now we give the form of the adjoint process compatible with \eqref{eq:costate}:
	\begin{align}
		p^i_t = A^\top_{t+1} p^i_{t+1} + \bA^\top_{t+1} \bp^i_{t+1} + 2 Q^i_{t+1} (x_{t+1} - \bx_{t+1})  + 2 \bQ^i_{t+1} \bx_{t+1} - Z^{0,i}_{t+1} \omega^0_{t+1} - Z^{1,i}_{t+1} \omega_{t+1}. \label{eq:costate_expand}
	\end{align}
	Utilizing the necessary conditions for OLNE in Proposition \ref{prop:nec_cond_OLNE} and taking conditional expectation $\EE[\cdot | \Fs^0]$ we obtain,
	\begin{align} \label{eq:bu_P}
		\bu^{i*}_t = -\frac{1}{2}(\bR^i_t)^{-1}(B^i_t + \bB^i_t)^\top \bp^i_t = - (L^{i}_t + \bL^{i}_t) \bp^i_t
	\end{align}
	Substituting this back into the OLNE in Proposition \ref{prop:nec_cond_OLNE}, we get
	\begin{align} \label{eq:u_P}
		 u^{i*}_t = - \frac{1}{2} (R^i_t)^{-1} (B^i_t)^\top p^i_t - \frac{1}{2} (R^i_t)^{-1} \big( R^i_t (\bR^i_t)^{-1}(B^i_t + \bB^i_t)^\top - (B^i_t)^\top \big) \bp^i_t = - L^{i}_t p^i_t + \bL^{i}_t \bp^i_t.
	\end{align}
	Substituting \eqref{eq:bu_P} and \eqref{eq:u_P} into \eqref{eq:gen_agent_dyn} and restating the adjoint process, we arrive at the forward-backward necessary conditions for the OLNE:
	\begin{align}
		x_{t+1} & = A_t x_t + \bA_t \bx_t - \sum_{i=1}^N \big(B^i_t (L^{i}_t p^i_t + \bL^{i}_t \bp^i_t) + \bB^i_t(L^{i}_t + \bL^{i}_t) \bp^i_t \big) + \omega^0_{t+1} + \omega_{t+1}, \nonumber \\
		p^i_t & = A^\top_{t+1} p^i_{t+1} + \bA^\top_{t+1} \bp^i_{t+1} + 2 Q^i_{t+1} (x_{t+1} - \bx_{t+1})  + 2 \bQ^i_{t+1} \bx_{t+1} \label{eq:FB}  \\
        & \hspace{6cm}- Z^{0,i}_{t+1} \omega^0_{t+1} - Z^{1,i}_{t+1} \omega_{t+1}, \nonumber
	\end{align}
	and $p^i_T = 0$. Taking conditional expectation $\EE[\cdot | \Fs^0]$, we obtain
	\begin{align}
		\bx_{t+1} & = (A_t + \bA_t) \bx_t - \sum_{i=1}^N (B^i_t + \bB^i_t)(L^{i}_t + \bL^{i}_t) \bp^i_t + \omega^0_{t+1}, \nonumber \\
		\bp^i_t & = (A^\top_{t+1} + \bA^\top_{t+1}) \bp^i_{t+1} + 2 \bQ^i_{t+1} \bx_{t+1} - Z^{0,i}_{t+1} \omega^0_{t+1}, \hspace{0.2cm} \bp^i_T = 0. \label{eq:FB_inter}
	\end{align}
	Let us introduce the ansatz 
    \begin{align} \label{eq:ansatz_bp_i_t}
        \bp^i_t = \bP^i_t \bx_t
    \end{align}
    and substitute into \eqref{eq:FB_inter}:
	\begin{align}
		\bx_{t+1} & = \big((A_t + \bA_t) - \sum_{i=1}^N (B^i_t + \bB^i_t)(L^{i}_t + \bL^{i}_t) \bP^i_t \big) \bx_t + \omega^0_{t+1}, \nonumber \\
		\bP^i_t \bx_t & = \big((A^\top_{t+1} + \bA^\top_{t+1}) \bP^i_{t+1} + 2 \bQ^i_{t+1} \big) \bx_{t+1} - Z^{0,i}_{t+1} \omega^0_{t+1}, \hspace{0.2cm} \bp^i_T = 0. \label{eq:FB_inter_2}
	\end{align}
	Using \eqref{eq:FB_inter_2} we arrive at the Riccati equation,
	\begin{align} \label{eq:OLNE_Riccati_bar}
		\bP^i_t = \big((A^\top_{t+1} + \bA^\top_{t+1}) \bP^i_{t+1} + 2 \bQ^i_{t+1} \big) \big((A_t + \bA_t) - \sum_{i=1}^N (B^i_t + \bB^i_t)(L^{i}_t + \bL^{i}_t) \bP^i_t \big)
	\end{align}
	Next we write the forward-backward equations \eqref{eq:FB}-\eqref{eq:FB_inter} in terms of $x_t - \bx_t$ and $p^i_t - \bp^i_t$:
	\begin{align}
		x_{t+1} - \bx_{t+1} & = A_t (x_t - \bx_t) - \sum_{i=1}^N B^i_t L^{i}_t (p^i_t - \bp^i_t)  + \omega_{t+1}, \nonumber \\
		p^i_t - \bp^i_t & = A^\top_{t+1} (p^i_{t+1} - \bp^i_{t+1}) + 2 Q^i_{t+1} (x_{t+1} - \bx_{t+1})  - Z^{1,i}_{t+1} \omega_{t+1}, p^i_T = 0.
	\end{align}
	As before, let us introduce another ansatz 
    \begin{align} \label{eq:ansatz_p_i_t}
        p^i_t - \bp^i_t = P^i_t (x_t - \bx_t).
    \end{align}
    We arrive at the second Riccati equation,
	\begin{align} \label{eq:OLNE_Riccati}
		P^i_t & = (A^\top P^i_{t+1} + 2Q^i_t) (A_t - \sum_{j=1}^N B^j_t L^{j}_t P^j_{t}).
	\end{align}
    Furthermore, using \eqref{eq:ansatz_bp_i_t}-\eqref{eq:ansatz_p_i_t} we can deduce the feedback representation of the OLNE:
    \begin{align*}
        u^{i}_t & = - L^{i}_t p^i_t - \bL^{i}_t \bp^i_t = - L^{i}_t P^i_t (x_t - \bx_t) - (L^{i}_t + \bL^{i}_t) \bP^i_t \bx_t, \\
        \bu^{i}_t & = - (L^{i}_t + \bL^{i}_t) \bp^i_t = - (L^{i}_t + \bL^{i}_t) \bP^i_t \bx_t
    \end{align*}
    This concludes the proof of the proposition.
\end{proof}
\textbf{Having provided the structure of the OLNE following from the necessary conditions, and in terms of the Riccati equations \eqref{eq:OLNE_Riccati_bar}-\eqref{eq:OLNE_Riccati} we now prove that the necessary conditions are also sufficient.} This follows form the convexity, in particular, the quadratic nature of the cost function $J^i$ with respect to perturbations in control policy.
\begin{proposition}[OLNE Sufficient Condition]
	If there exists a state process $(x_t)_{t \geq 0}$ and an adjoint process $(p^i, Z^{0,i}, Z^{1,i})_{i \in [N]}$ satisfying \eqref{eq:FB} and \eqref{eq:deter_adj}, then the control laws given by \eqref{eq:OL_eq_control} constitute the OLNE of the $N$-player LQ-MFTG.
\end{proposition}
\begin{proof}
	The proof of this proposition starts by introducing a perturbation around the candidate OLNE controls (i.e. the set of controls which satisfy the OLNE necessary conditions). Then using a second-order expansion of the cost it is shown that the the perturbation around the candidate OLNE controls will always lead to a strictly higher cost. This is due to the quadratic nature of the cost. This proves that the necessary conditions for the OLNE are also the sufficient conditions for the OLNE. 
 
    Let us first write a second-order expansion for cost $J^i$ at control policies $(u^i)_{i \in [N]} \in \Us^i$ and $(u^{i\prime} = u^i + \epsilon \beta^i)_{i \in [N]}, \beta^i \in \Us^i$. We introduce the deterministic process corresponding to the control perturbations $(\beta^i)_{i \in [N]}$, and let $V_t = (x_t - x'_t)/\epsilon$ where $(x_t)_{t \geq 0}$ is the state process under control policy $(u^i)_{i \in [N]}$ and $(x'_t)_{t \geq 0}$ is the state process under control policy $(u^{i\prime})_{i \in [N]}$. Then, 
	\begin{align*}
		V_{t+1} = A_t V_t + \bA_t \bV_t + \sum_{i=1}^N B^i_t \beta^i_t + \bB^i_t \bbeta^i_t
	\end{align*}
	The difference in costs due to the control perturbation is:
	\begin{align}
		 & J^i(u^{i\prime},u^{-i\prime}) - J^i(u^{i},u^{-i}) = J^i(u^{i} + \epsilon \beta^i,u^{-i}+\epsilon \beta^{-i}) - J^i(u^{i},u^{-i}) \nonumber \\
		&= \epsilon \EE \sum_{t=0}^{T-1} \big[ V^\top_t (p^i_{t+1} - p^i_t) + \delta_{\zeta_t}h^i_t(\zeta_t,p^i_t) \czeta_t \big] + \frac{1}{2} \epsilon^2 \EE \sum_{t=0}^{T-1} \big[ \delta^2_{\zeta_t,\zeta_t} h^i_t(\zeta_t, p^i_t) \czeta_t \cdot \czeta_t \big] \label{eq:2nd_order_decomp}
	\end{align}
	where $\zeta_t = (x_t, \bx_t, (u^i_t, \bu^i_t)_{i \in [N]})$, $\zeta'_t = (x'_t, \bx'_t, (u^{i\prime}_t, \bu^{i\prime}_t)_{i \in [N]})$ and $\czeta_t = (\zeta'_t - \zeta_t)/\epsilon = (V_t, \bV_t, (\beta^i_t, \bbeta^i_t)_{i \in [N]})$. If we assume that $(u^i)_{i \in [N]} = (u^{i*})_{i \in [N]}$, meaning that they satisfy the necessary conditions for the OLNE (Proposition \ref{prop:nec_cond_OLNE}) and that there exists an adjoint process such that \eqref{eq:FB} and \eqref{eq:deter_adj} are satisfied, then the first order terms in the above given equation vanish, that is $\EE \sum_{t=0}^{T-1} [ V^\top_t (p^i_{t+1} - p^i_t) + \delta_{\zeta_t}h^i_t(\zeta_t,p^i_t) \czeta_t ] = 0$. To compute the last term we compute the second derivative of $h^i_t$ with respect to $\zeta_t$:
	\begin{align*}
	& \delta^2_{x_t,x_t} h^i_t(\zeta_t, p^i_t) = 2 Q^i_t, \hspace{0.1cm} \delta^2_{x_t,\bx_t} h^i_t(\zeta_t, p^i_t) = \delta^2_{\bx_t,x_t} h^i_t(\zeta_t, p^i_t) = -2Q^i_t,\\
        & \delta^2_{\bx_t,\bx_t} h^i_t(\zeta_t, p^i_t) = 2(Q^i_t + \bQ^i_t), \\
	& \delta^2_{u_t,u_t} h^i_t(\zeta_t, p^i_t) = 2 R^i_t, \hspace{0.1cm} \delta^2_{u_t,\bu_t} h^i_t(\zeta_t, p^i_t) = \delta^2_{\bu_t,u_t} h^i_t(\zeta_t, p^i_t) = -2R^i_t, \\
        & \delta^2_{\bu_t,\bu_t} h^i_t(\zeta_t, p^i_t) = 2(R^i_t + \bR^i_t)
	\end{align*}
	Now we characterize the last term in \eqref{eq:2nd_order_decomp}:
	\begin{align*}
		&J^i(\mcu^{i\prime},\mcu^{-i\prime}) - J^i(u^{i},u^{-i}) = \frac{1}{2} \epsilon^2 \EE \sum_{t=0}^{T-1} \big[ \delta^2_{\zeta_t,\zeta_t} h^i_t(\zeta_t, p^i_t) \czeta_t \cdot \czeta_t \big] \\
		& = \frac{1}{2} \epsilon^2 \EE \sum_{t=0}^{T-1} \big[ (V_t - \bV_t)^\top Q^i_t (V_t - \bV_t) + \bV^\top_t \bQ^i_t \bV_t + (\beta^i_t - \bbeta^i_t)^\top R^i_t (\beta^i_t - \bbeta^i_t) + (\bbeta^i_t)^\top \bR^i_t \bbeta^i_t \big] \\
        & > 0
	\end{align*}
	Therefore if there exists a state process $(x_t)_{t \geq 0}$ and an adjoint process $(p^i, Z^{0,i}, Z^{1,i})_{i \in [N]}$ satisfying \eqref{eq:FB} and \eqref{eq:deter_adj} then the control laws given by \eqref{eq:OL_eq_control} satisfy the necessary and sufficient conditions for OLNE of the $N$-player LQ-MFTG.
	\end{proof}



	\textbf{Finally, we prove that every OLNE satisfies the necessary conditions for Nash Equilibria (NE).} We start by computing the Gateaux derivative of $J^i$ 
	but now with the control actions $u^j_t, j \neq i$ being functions of the state at time $t$, $x_t$. This will give us necessary conditions for the NE since in the NE (due to the backwards nature of the discrete-time HJI equations) will have a feedback structure i.e. the NE control actions at time $t$ will depend on state $x_t$. Recalling \eqref{eq:cost_diff} and \eqref{eq:V_T}, we can write down
	\begin{align*}
		& \Ds J^i(\mcu^i, \mcu^{-i})(\beta^i,0) \nonumber\\
		& = \EE \sum_{t=0}^{T-1} \big[ V^\top_{t+1} (p^i_{t+1} - p^i_t) + \delta_{x_t} h^i_t(\zeta_t, p^i_t)^\top V_t \\
        & \hspace{0.3cm} + \bigg(\sum_{j \neq i} \frac{\delta u^j_t}{\delta x_t} \delta_{u^j_t} h^i_t(\zeta_t,p^i_t) + \frac{\delta \bu^j_t}{\delta x_t} \delta_{\bu^j_t} h^i_t(\zeta_t,p^i_t)\bigg)^\top \hspace{-0.1cm} V_t \\
		& \hspace{0.3cm} + \delta_{\bx_t} h^i_t(\zeta_t, p^i_t)^\top \bV_t +{ \bigg(\sum_{j \neq i} \frac{\delta u^j_t}{\delta \bx_t} \delta_{u^j_t} h^i_t(\zeta_t,p^i_t) + \frac{\delta \bu^j_t}{\delta \bx_t} \delta_{\bu^j_t} h^i_t(\zeta_t,p^i_t)\bigg)^\top \bV_t} + \delta_{u^i_t} h^i_t(\zeta_t, p^i_t)^\top \beta^i_t  \\
        & \hspace{0.3cm} + \delta_{\bu^i_t} h^i_t(\zeta_t, p^i_t)^\top \bbeta^i_t \big] \\
		& = \EE \sum_{t=0}^{T-1} \big[ V^\top_{t+1} (p^i_{t+1} - p^i_t) + \bigg((p^i_t)^\top (A_t - I) + 2 (x_t - \bx_t )^\top Q^i_t  \\
		& \hspace{0.3cm} + { \sum_{j \neq i} \bigg(\frac{\delta u^j_t}{\delta x_t} B^j_t + \frac{\delta \bu^j_t}{\delta x_t} \bB^j_t \bigg)^\top p^i_t} \bigg)^\top V_t + \bigg((p^i_t)^\top \bA_t - 2 (x_t - \bx_t)^\top Q^i_t + 2 \bx^\top_t \bQ^i_t   \nonumber\\
		& \hspace{0.3cm} + { \sum_{j \neq i} \bigg(\frac{\delta u^j_t}{\delta \bx_t} B^j_t + \frac{\delta \bu^j_t}{\delta \bx_t} \bB^j_t \bigg)^\top p^i_t} \bigg)^\top \bV_t + ((p^i_t)^\top B^i_t + 2(u^i_t - \bu^i_t)^\top R^i_t)^\top \beta^i_t \nonumber \\
        & \hspace{0.8cm} + ((p^i_t)^\top \bB^i_t - 2(u^i_t - \bu^i_t)^\top R^i_t + 2 (\bu^i_t)^\top \bR^i_t)^\top \bbeta^i_t \big].
	\end{align*}
	Next using techniques similar to \cite{carmona2020policy} we deduce that
	\begin{align*}
		& \bigg((p^i_t)^\top \bA_t - 2 (x_t - \bx_t)^\top Q^i_t + 2 \bx^\top_t \bQ^i_t + \sum_{j \neq i} \bigg(\frac{\delta u^j_t}{\delta \bx_t} B^j_t + \frac{\delta \bu^j_t}{\delta \bx_t} \bB^j_t \bigg)^\top p^i_t \bigg)^\top \bV_t \\
		& \hspace{4cm} = \bigg((\bp^i_t)^\top \bA_t + 2 \bx^\top_t \bQ^i_t + \sum_{j \neq i} \bigg(\frac{\delta u^j_t}{\delta \bx_t} B^j_t + \frac{\delta \bu^j_t}{\delta \bx_t} \bB^j_t \bigg)^\top \bp^i_t \bigg)^\top V_t.
	\end{align*}
	Now let the adjoint process satisfy the following condition:
	\begin{align} \label{eq:FB_adjoint}
		p^i_t = & \EE\bigg[ A^\top_{t+1} p^i_{t+1} + 2 Q^i_t (x_{t+1} - \bx_{t+1})  + { \sum_{j \neq i} \bigg(\frac{\delta u^j_{t+1}}{\delta x_{t+1}} B^j_{t+1} + \frac{\delta \bu^j_{t+1}}{\delta x_{t+1}} \bB^j_{t+1} \bigg)^\top p^i_{t+1}} \nonumber \\
        & \hspace{9cm} + \bA^\top_{t+1} \bp^i_{t+1}   \nonumber \\
		& \hspace{3cm} + 2 \bQ^i_{t+1} \bx_{t+1} + { \sum_{j \neq i} \bigg(\frac{\delta u^j_{t+1}}{\delta \bx_{t+1}} B^j_{t+1} + \frac{\delta \bu^j_{t+1}}{\delta \bx_{t+1}} \bB^j_{t+1} \bigg)^\top \bp^i_{t+1}}\bigg| \Fs_t \bigg].
	\end{align} 
	Then the Gateaux derivative will have the form:
	\begin{align*}
		\Ds J^i(u^i, u^{-i})(\beta^i,0) = \EE \sum_{t=0}^{T-1} \big[ ((p^i_t)^\top B^i_t + 2(u^i_t - \bu^i_t)^\top R^i_t + (\bp^i_t)^\top \bB^i_t + 2 (\bu^i_t)^\top \bR^i_t)^\top \beta^i_t \big],
	\end{align*}
	which, using the same techniques as before, will result in
	\begin{align*}
		u^{i*}_t = - L^{i}_t p^i_t - \bL^{i}_t \bp^i_t, \hspace{0.2cm} \bu^{i*}_t = - (L^{i}_t + \bL^{i}_t) \bp^i_t
	\end{align*}
	similar to Lemma \ref{lem:Gateaux} hence we have obtained the necessary conditions for NE. Notice that since every solution of \eqref{eq:costate} also satisfies the relation \eqref{eq:FB_adjoint}, every OLNE satisfies the necessary conditions for NE. This analysis is similar to the analysis of \cite{bacsar1998dynamic} (Chapter 6.2) between Open-loop (OL) and Feedback (FB) NE for deterministic $N$-player LQ games, but there it has been shown that OL (non-adapted) and FB NE are not related, that is one cannot be derived from the other.
\end{proof}

\section{Proof of Theorem \ref{thm:CLNE}} \textbf{[Closed-loop Analysis of GS-MFTG]}\label{sec:CLNE}

We will solve for the NE of the GS-MFTG using the discrete-time Hamilton-Jacobi-Isaacs (HJI) equations \cite{bacsar1998dynamic}. We will solve the problem of finding NE $(\mcv^{i*})_{i \in [N]}$. The procedure for computing $(\bmcu^{i*})_{i \in [N]}$ is similar, and hence is omitted. We first introduce the following backwards recursive equations
\begin{align} \label{eq:P_i_t}
	\big(R^i_t + (B^i_t)^\top Z^i_{t+1} B^i_t \big) K^{i*}_t + (B^i_t)^\top Z^i_{t+1} \sum_{j \neq i} (B^j_t)^\top K^{j*}_t & = (B^i_t)^\top Z^i_{t+1} A_t, \\
    \big(\bR^i_t + (\tB^i_t)^\top \bZ^i_{t+1} \tB^i_t \big) \bK^{i*}_t + (\tB^i_t)^\top \bZ^i_{t+1} \sum_{j \neq i} (\tB^j_t)^\top \bK^{j*}_t & = (\tB^i_t)^\top \bZ^i_{t+1} \tA_t, \hspace{0.2cm} i \in [N] \nonumber 
\end{align}
and the matrices $Z^i_t$ are determined as follows,
\begin{align} \label{eq:Z_i_t}
	Z^i_t & = F^\top_t Z^i_{t+1} F_t + (K^{i*}_t)^\top R^i_t K^{i*}_t + Q^i_{t}, \hspace{0.2cm} Z^i_{T} = Q^i_T, \\
    \bZ^i_t & = \bF^\top_t \bZ^i_{t+1} \bF_t + (\bK^{i*}_t)^\top \bR^i_t \bK^{i*}_t + \bQ^i_{t}, \hspace{0.2cm} \bZ^i_{T} = \bQ^i_T,  \hspace{0.2cm} i \in [N] \nonumber 
\end{align}
where $F_t = A_t - \sum_{i=1}^N B^i_t K^{i*}_t$ and $\bF_t = \tA_t - \sum_{i=1}^N \tB^i_t \bK^{i*}_t$. 
We start by writing down the discrete-time Hamilton-Jacobi-Isaacs (HJI) equations \cite{bacsar1998dynamic} in order to find the set of controls $(\mcv^{i*})_{i \in [N]}$:
\begin{align} 
	V^i_t(y) & = \min_{v^i_t} \EE[ g^i_t(\tf^{i*}_t(y,v^i_t), \mcv^{i*}_t(y), \ldots,v^i_t(y), \ldots, \mcv^{N*}_t(y),y)  + V^i_{t+1}(\tf^{i*}_t(y,v^i_t)) | v],\nonumber \\
	& = \EE[ g^i_t(\tf^{i*}_t(y,\mcv^{i*}_t), \mcv^{i*}_t(y), \ldots,\mcv^{i*}_t, \ldots, \mcv^{N*}_t(y),y)  + V^i_{t+1}(\tf^{i*}_t(y,\mcv^{i*}_t)) | y], \nonumber \\
	V^i_{T}(y) & = 0 \label{eq:HJI}
\end{align}
where
\begin{align*}
	\tf^i_t(y,v^i_t) = f_t(y,\mcv^{1*}_t(y),\ldots,v^i_t,\ldots,\mcv^{N*}_t(y)).
\end{align*}
The dynamics of the deviation process $y_t$ and its corresponding instantaneous costs are
\begin{align*}
	f_t(y_{t}, v^1_t, \ldots, v^N_t) & = A_t y_t + \sum_{i=1}^N B^i_t v^i_t + \omega_{t+1}, \\
	g^i_t(y_{t+1}, v^1_t, \ldots, v^N_t, y_t) & = \frac{1}{2} \Big( y^\top_{t+1} Q^i_{t+1} y_{t+1} + (v^i_t)^\top R^i_t v^i_t \Big), \hspace{0.2cm} g^i_T(y_T) = 0.
\end{align*}
Notice that scaling the cost by $1/2$ does not change the nature of the problem but makes the analysis more compact. Hence starting at time $T-1$ and using the HJI equations \eqref{eq:HJI}, we get
\begin{align}
	 V^i_{T-1}(y_{T-1}) &= \min_{v^i_{T-1}} \EE [y^\top_{T} Q^i_{T} y_{T} + (v^i_{T-1})^\top R^i_{T-1} v^i_{T-1} | y_{T-1}], \nonumber \\
	& = \min_{v^i_{T-1}} [(v^i_{T-1})^\top (R^i_{T-1} + (B^i_{T-1})^\top Q^i_T B^i_{T-1}) v^i_{T-1} \nonumber \\
	& \hspace{0.5cm} + 2 (B^i_{T-1}v^i_{T-1})^\top Q^i_T \tF^i_{T-1}y_{T-1} \nonumber \\
    & \hspace{0.5cm} + (\tF^i_{T-1}(y_{T-1}))^\top Q^i_T \tF^i_{T-1}(y_{T-1}) | y_{T-1}] + \tr(Q^i_T \Sigma)
\end{align}
where $\tF^i_t(y_t) = A_{t} y_{t} + \sum_{j \neq i} B^j_{t} \mcv^{j*}_{t}(y_{t})$. Now differentiating with respect to $v^i_t$, the necessary conditions for NE become
\begin{align*}
	& -(R^i_{T-1} + (B^i_{T-1})^\top Q^i_T B^i_{T-1}) \mcv^{i*}_{T-1}(y_{T-1}) - (B^i_{T-1})^\top Q^i_T \sum_{j \neq i} \mcv^{j*}_{T-1}(y_{T-1}) \\
    & \hspace{8cm}= (B^i_{T-1})^\top Q^i_T A_{T-1} y_{T-1}
\end{align*}
$\forall i \in [N]$. Hence, $\mcv^{i*}_{T-1}(x_{T-1})$ is linear in $x_{T-1}$, $\mcv^{i*}_{T-1}(x_{T-1}) = -K^{i*}_{T-1} x_{T-1}$. Hence, we get
\begin{align*}
	& (R^i_{T-1} + (B^i_{T-1})^\top Q^i_T B^i_{T-1}) P^{i}_{T-1}y_{T-1} + (B^i_{T-1})^\top Q^i_T \sum_{j \neq i} P^{j}_{T-1}(y_{T-1}) \\
    & \hspace{8cm}= (B^i_{T-1})^\top Q^i_T A_{T-1} y_{T-1}
\end{align*}
The value function at time $T-1$ can now be calculated as:
\begin{align*}
	V^i_{T-1}(x_{T-1}) & = y^\top_{T-1} [F^\top_{T-1} Q^i_{T} F_{T-1} + (K^{i*}_{T-1})^\top R^i_{T-1} K^{i*}_{T-1} ] y_{T-1} + \tr(Q^i_T \Sigma), \\
	& = y^\top_{T-1} [Z^i_{T-1} - Q^i_{T-1}] y_{T-1} + \tr(Q^i_T \Sigma)
\end{align*}
where $Z^i_{T-1} = F^\top_t Z^i_{T} F_t + (K^{i*}_{T-1})^\top R^i_{T-1} K^{i*}_{T-1} + Q^i_{T-1}$. Now let us take
\begin{align*}
	V^i_{t+1}(y) = y^\top (Z^i_{t+1} - Q^i_{t+1}) x + \sum_{s=t+2}^T\tr(Z^i_{s} \Sigma), \hspace{0.2cm} i \in [N].
\end{align*}
Using the HJI equations \eqref{eq:HJI},
\begin{align*}
	V^i_t(y_t) = \min_{v^i_t \in \Us^i_t} [(\tF^i_t(y_t) + B^i_t v^i_t)^\top Z^i_{t+1} (\tF^i_t(y_t) + B^i_t v^i_t) + (v^i_t)^\top R^i_t v^i_t] + \sum_{s=t+1}^T\tr(Z^i_{s} \Sigma)
\end{align*}
and differentiating with respect to $v^i_t$ and using the necessary conditions of NE we get
\begin{align*}
	-\big(R^i_t + (B^i_t)^\top Z^i_{t+1} B^i_t \big) \mcv^{i*}_t(x_t) - (B^i_t)^\top Z^i_{t+1} \sum_{j \neq i} (B^j_t)^\top  \mcv^{j*}_t(x_t) = (B^i_t)^\top Q^i_t A_t.
\end{align*}
Again we can notice that $\mcv^{i*}_t$ is linear in $y_t$, and thus we get $\mcv^{i*}_t(x_t) = - K^{i*}_t y_t$ and recover  \eqref{eq:P_i_t}. Using this expression of $\mcv^{i*}_t$ the value function for agent $i$ becomes
\begin{align*}
	V^i_t(y_t) & = y^\top_t[F^\top_t Z_{t+1} F_t + (K^{i*}_t)^\top R^i_t K^{i*}_t] y_t + \sum_{s=t+1}^T\tr(Z^i_{s} \Sigma), \\
	& = y^\top_t (Z^i_t - Q^i_t) y_t  + \sum_{s=t+1}^T\tr(Z^i_{s} \Sigma)
\end{align*}
where the second inequality is obtained using \eqref{eq:Z_i_t}. Hence we have completed the characterization of $(\mcv^{i*})_{i \in [N]}$, and the characterization of $(\bmcu^{i*})_{i \in [N]}$ follows similar techniques, and hence is omitted. As a result the NE has the following linear structure
\begin{align}
	u^{i*}_t(x_t) = \mcv^{i*}_t(y_t) + \bmcu^{i*}_t(\bx_t) = -K^{i*}_t (x_t - \bx_t) - \bK^{i*}_t \bx_t, \hspace{0.2cm} i \in [N], t \in \{0,\ldots.T-1\}
\end{align}
such that the matrices $(K^{i*}_t,\bK^{i*}_t)_{0 \leq t \leq T-1}$ satisfy \eqref{eq:P_i_t}-\eqref{eq:Z_i_t}. The sufficient condition for existence and uniqueness of solution to \eqref{eq:P_i_t} can be obtained by concatenating \eqref{eq:P_i_t} for all $i \in [N]$ and requiring that the matrices $\Phi_t$ and $\bPhi_t$ be invertible, where
\begin{align*}
	\Phi_t & = \begin{pmatrix}
		R^1_t + (B^1_t)^\top Z^1_{t+1} B^1_t & (B^1_t)^\top Z^1_{t+1} B^2_t & \cdots \\
		(B^2_t)^\top Z^2_{t+1} B^1_t & R^2_t + (B^2_t)^\top Z^2_{t+1} B^2_t & \cdots \\
		\vdots & \vdots  & \ddots 
	\end{pmatrix}, \hspace{0.1cm} \\
    \bPhi_t & = \begin{pmatrix}
		\bR^1_t + (\tB^1_t)^\top \bZ^1_{t+1} \tB^1_t & (\tB^1_t)^\top \bZ^1_{t+1} \tB^2_t & \cdots \\
		(\tB^2_t)^\top \bZ^2_{t+1} \bB^1_t & \bR^2_t + (\tB^2_t)^\top \bZ^2_{t+1} \bB^2_t & \cdots \\
		\vdots & \vdots  & \ddots 
	\end{pmatrix}.
\end{align*}
These sufficient conditions are similar to the sufficient conditions for the $N$-player LQ games (Corollary 6.1 \cite{bacsar1998dynamic}). In case $\Phi_t$ and $\bPhi_t$ are invertible, the control matrices $K^{i*}_t$ and $\bK^{i*}_t$ can be computed as
\begin{align*}
	 \begin{pmatrix}
		K^{1*}_t \\ \vdots \\ K^{N*}_t
	\end{pmatrix} = \Phi_t^{-1} \begin{pmatrix}
	(B^1)^\top Q^1 A_t \\ \vdots \\ (B^N)^\top Q^N A_t
\end{pmatrix}, \hspace{0.2cm} \begin{pmatrix}
		\bK^{1*}_t \\ \vdots \\ \bK^{N*}_t
	\end{pmatrix} = \bPhi_t^{-1} \begin{pmatrix}
	(\tB^1_t)^\top \bQ^1_t A_t \\ \vdots \\ (\tB^N_t)^\top Q^N_t A_t
\end{pmatrix}.
\end{align*}
This completes the proof.


\section{Proof of Theorem \ref{thm:eps_Nash}} \textbf{[NE of the MFTG is $\Os(1/M)$-Nash for the finite agent CC game]}
\label{sec:eps_Nash}
\begin{proof}
    For this proof we will analyze the $\epsilon$-Nash property for a fixed $i \in [N]$. Central to this analysis is the quantification of the difference between the finite and infinite population costs for a given set of control policies. First we express the state and mean-field processes in terms of the noise processes, for the finite and infinite population settings. This then allows us to write the costs (in both settings) as quadratic functions of the noise process, which simplifies quantification of the difference between these two costs.
    
    Let us first concatenate the states of $j^{th}$ agents in all teams such that $x^j_t = [(x^{1,j}_t)^\top,\ldots, (x^{N,j}_t)^\top]^\top$. For simplicity of analysis we assume $M_i= M$. If for some $i \in [N]$, $M_i \neq M$ then we can redefine $M = \min_i M_i$. Consider the dynamics of joint state $x^j$ under the NE of the GS-MFTG (Theorem \ref{thm:CLNE})
    \begin{align} \label{eq:epsN_x_M}
        x^{j*,M}_{t+1} = \underbrace{(A^*_t - \sum_{j=1}^N B^j_t K^{j*}_t)}_{L^*_t} x^{j*,M}_t + \underbrace{(\bA^*_t - \sum_{j=1}^N \bB^j_t \bK^{j*}_t)}_{\bL^*_t} \bx^{M*}_t + \omega^j_{t+1} + \omega^0_{t+1}
    \end{align}
    where the superscript $M$ denotes the dynamics in the finite population game \eqref{eq:finite_agent_dyn}-\eqref{eq:finite_agent_utility} and $\bx^{M*} = \frac{1}{M}\sum_{j \in [M]} x^{j*,M}_t$ is the empirical mean-field. We can also write the dynamics of the empirical mean-field as
    \begin{align} \label{eq:epsN_bx_M}
        \bx^{M*}_{t+1} = \underbrace{\big( A_t + \bA_t - \sum_{j=1}^N (B^j_t K^{j*}_t + \bB^j_t \bK^{j*}_t) \big)}_{\tL^*_t}\bx^{M*}_t + \underbrace{\frac{1}{M} \sum_{j \in [M]}\omega^j_{t+1} + \omega^0_{t+1}}_{\bomega^M_{t+1}}.
    \end{align}
    For simplicity we assume that $x^{j*,M}_0 = \omega^j_0 + \omega^0_0$ which also implies that $\bx^{M*}_0 = \bomega^M_0$. Using \eqref{eq:epsN_bx_M} we get the recursive definition of $\bx^M_t$ as
    \begin{align*}
        \bx^{M*}_t = \sum_{s=0}^t \tL^*_{[t-1,s]}\bomega^N_s, \text{ where } \tL^*_{[s,t]} := \tL^*_s \tL^*_{s-1} \ldots \tL^*_t \text{, if } s \geq t. \text{ and } \tL^*_{[s,t]} = I \text{ otherwise}.
    \end{align*}
    Hence $\bx^{M*}_t$ can be characterized as a linear function of the noise process
    \begin{align*}
        \bx^{M*}_t = \big( \bPsi^* \bomega^M \big)_t, \text{ where } \bPsi^* = 
        \begin{pmatrix}
            I & 0 & 0 & \hdots \\
            \tL^*_{[0,0]} & I & 0 & \hdots \\
            \tL^*_{[1,0]} & \tL^*_{[1,1]} & I & \hdots \\
            \tL^*_{[2,0]} & \tL^*_{[2,1]} & \tL^*_{[2,2]} & \hdots \\
            \vdots & \vdots & \vdots & \ddots
        \end{pmatrix} \text{ and }
        \bomega^M = \begin{pmatrix}
            \bomega^M_0 \\ \bomega^M_1 \\ \bomega^M_2 \\ \vdots \\ \bomega^M_{T}
        \end{pmatrix}
    \end{align*}
    where $(M)_t$ denotes the $t^{th}$ block of matrix $M$ and the covariance matrix of $\bomega^M$ is {$\EE[\bomega^M (\bomega^M)^\top] = \diag((\Sigma/M + \Sigma^0)_{0 \leq t \leq T})$}. Similarly using \eqref{eq:epsN_x_M} we can write
    \begin{align*}
        x^{j*,M}_t = \big( \Psi^* \tomega^{j*,M} \big)_t, \text{ where } \Psi^* = 
        \begin{pmatrix}
            I & 0 & 0 & \hdots \\
            L^*_{[0,0]} & I & 0 & \hdots \\
            L^*_{[1,0]} & L^*_{[1,1]} & I & \hdots \\
            L^*_{[2,0]} & L^*_{[2,1]} & L^*_{[2,2]} & \hdots \\
            \vdots & \vdots & \vdots & \ddots
        \end{pmatrix}
    \end{align*}
    and
    \begin{align*}
        & \tomega^{j*,M} = \\
        & \begin{pmatrix}
            \omega^j_0 + \omega^0_0 \\ \omega^j_1 + \omega^0_1 \\ \omega^j_2 + \omega^0_2 \\ \vdots \\ \omega^j_{T} + \omega^0_{T}
        \end{pmatrix} + \underbrace{
        \begin{pmatrix}
            0 & 0 & 0 & 0 & \hdots \\
            \bL^*_0 & 0 & 0 & 0 & \hdots \\
            0 & \bL^*_1 & 0 & 0 & \hdots \\
            0 & 0 & \bL^*_2 & 0 & \hdots \\
            \vdots & \vdots & \vdots & \vdots & \ddots
        \end{pmatrix} \begin{pmatrix}
            I & 0 & 0 & \hdots \\
            \tL^*_{[0,0]} & I & 0 & \hdots \\
            \tL^*_{[1,0]} & \tL^*_{[1,1]} & I & \hdots \\
            \tL^*_{[2,0]} & \tL^*_{[2,1]} & \tL^*_{[2,2]} & \hdots \\
            \vdots & \vdots & \vdots & \ddots
        \end{pmatrix}}_{\Ls^*} \begin{pmatrix}
            \bomega^M_0 \\ \bomega^M_1 \\ \bomega^M_2 \\ \vdots \\ \bomega^M_{T}
        \end{pmatrix}.
    \end{align*}
    Considering the infinite agent limit $\bx^{M*} \xrightarrow{M \rightarrow \infty} \bx^*$ and $x^{j*,M} \xrightarrow{M \rightarrow \infty} x^*$, we have
    \begin{align*}
        \bx^*_t = \big( \bPsi^* \omega^0 \big)_t, \hspace{0.2cm} \omega^0 = \begin{pmatrix}
            \omega^0_0 \\ \vdots \\ \omega^0_{T} \end{pmatrix} \text{ and }
        x^*_t = \big( \Psi^* \tomega \big)_t, \hspace{0.2cm}  \tomega = \begin{pmatrix}
            \omega_0 + \omega^0_0 \\ \omega_1 + \omega^0_1 \\ \omega_2 + \omega^0_2 \\ \vdots \\ \omega_{T} + \omega^0_{T} 
        \end{pmatrix} + \Ls^* \omega^0
    \end{align*}
    where the covariance of $\omega^0$ is {$\EE[\omega^0 (\omega^0)^\top] = \diag((\Sigma^0)_{0 \leq t \leq T})$}. Similarly we characterize the deviation process $x^{j*,M}_t - \bx^{M*}_t$, using \eqref{eq:epsN_x_M} and \eqref{eq:epsN_bx_M}
    \begin{align*}
        x^{j*,M}_{t+1} - \bx^{M*}_{t+1} = L^*_t (x^{j*,M}_t - \bx^{M*}_t) + \underbrace{\frac{M-1}{M}\omega^j_{t+1} + \frac{1}{M} \sum_{k \neq j} \omega^k_{t+1}}_{\bomega^{j,M}_{t+1}}.
    \end{align*}
    Hence 
    \begin{align*}
        x^{j*,M}_t - \bx^{M*}_t & = \sum_{s=0}^t L^*_{[t-1,s]} \bomega^{j,M}_s \\
        & = \big( \Psi^* \bomega^{j,M} \big)_t, \text{ where } \bomega^{j,M} = \frac{M-1}{M} \begin{pmatrix}
            \omega^j_0 \\ \vdots \\ \omega^j_{T-1}
        \end{pmatrix} + \frac{1}{M} \begin{pmatrix}
            \sum_{k \neq j} \omega^k_0 \\ \vdots \\ \sum_{k \neq j} \omega^k_{T-1}
        \end{pmatrix}
    \end{align*}
    where the covariance matrix of $\bomega^{j,M}$ is {$\EE[\bomega^{j,M} (\bomega^{j,M})^\top] = \diag(((M-1)/M \times \Sigma)_{0 \leq t \leq T})$}. Similarly the infinite agent limit of this process is $x^*_t - \bx^*_t = (\Psi^* \omega)_t$ where $\omega = (\omega^\top_0, \ldots, \omega^\top_{T-1})^\top$ whose covariance is {$\EE[\omega \omega^\top] = \diag((\Sigma)_{0 \leq t \leq T})$}. Now we compute the finite agent cost in terms of the noise processes,
    \begin{align*}
        & J^i_M(\mcu^{*}) \\
        & = \EE \bigg[ \frac{1}{M} \sum_{j \in [M]} \sum_{t = 0}^T \lVert x^{j*,M}_t - \bx^{M*}_t \rVert^2_{Q_t} + \lVert \bx^{M*}_t \rVert^2_{\bQ_t} + \lVert u^{j*,M}_t - \bu^{M*}_t \rVert^2_{R_t} + \lVert \bu^{M*}_t \rVert^2_{\bR_t}  \bigg] \\
        & = \EE \bigg[ \frac{1}{M} \sum_{j \in [M]} \sum_{t = 0}^T \lVert x^{j*,M}_t - \bx^{M*}_t \rVert^2_{Q_t + (K^*_t)^\top R_t K^*_t} + \lVert \bx^{M*}_t \rVert^2_{\bQ_t + (\bK^*_t)^\top \bR_t \bK^*_t}\bigg] \\
        & = \EE \bigg[ \frac{1}{M} \sum_{j \in [M]} (\Psi^* \bomega^{j,M})^\top \big( Q + (K^*)^\top R K^* \big) \Psi^* \bomega^{j,M} \\
        & \qquad \qquad+ (\bPsi^* \bomega^{M})^\top \big( \bQ + (\bK^*)^\top \bR \bK^* \big) \bPsi^* \bomega^{M} \bigg] \\
        & = \tr \big( (\Psi^*)^\top \big( Q + (K^*)^\top R K^* \big) \Psi^* \EE \big[\bomega^{j,M} (\bomega^{j,M})^\top \big] \big) \\
        & \qquad \qquad + \tr \big( (\bPsi^*)^\top \big( \bQ + (\bK^*)^\top \bR \bK^* \big) \bPsi^* \EE \big[\bomega^M (\bomega^M)^\top \big] \big)
    \end{align*}
    where $Q = \diag((Q_t)_{0 \leq T})$, $R = \diag((R_t)_{0 \leq T})$ and $K^* = \diag((K^*_t)_{0 \leq T})$ with $K^*_T = 0$. Using a similar technique we can compute the infinite agent cost:
    \begin{align*}
        J^i(\mcu^{*}) & = \EE \bigg[ \sum_{t = 0}^T \lVert x^*_t - \bx^*_t \rVert^2_{Q_t} + \lVert \bx^*_t \rVert^2_{\bQ_t} + \lVert u^*_t - \bu^*_t \rVert^2_{R_t} + \lVert \bu^*_t \rVert^2_{\bR_t}  \bigg] \\
        & =\tr \big( (\Psi^*)^\top \big( Q + (K^*)^\top R K^* \big) \Psi^* \EE \big[\omega \omega^\top \big] \big) \\
        & \qquad \qquad+ \tr \big( (\bPsi^*)^\top \big( \bQ + (\bK^*)^\top \bR \bK^* \big) \bPsi^* \EE \big[\bomega^0 (\bomega^0)^\top \big] \big).
    \end{align*}
    Now evaluating the difference between the finite and infinite population costs:
    \begin{align}
        & J^i_M(\mcu^{*}) - J^i(\mcu^{*}) \nonumber \\
        & = \tr \big( (\Psi^*)^\top \big( Q + (K^*)^\top R K^* \big) \Psi^* \big( \EE \big[\bomega^{j,M} (\bomega^{j,M})^\top - \EE \big[\omega \omega^\top \big]  \big) \big] \big) \nonumber \\
        & \hspace{3cm} + \tr \big( (\bPsi^*)^\top \big( \bQ + (\bK^*)^\top \bR \bK^* \big) \bPsi^* \big( \EE \big[\bomega^M (\bomega^M)^\top \big] - \EE \big[\bomega^0 (\bomega^0)^\top \big] \big) \big) \nonumber \\
        & \leq \underbrace{\Big(\lVert (\Psi^*)^\top \big( Q + (K^*)^\top R K^* \big) \Psi^* \rVert_F + \lVert (\bPsi^*)^\top \big( \bQ + (\bK^*)^\top \bR \bK^* \big) \bPsi^* \rVert_F \Big)}_{C_1} \nonumber \\
        & \hspace{9cm} \tr \big( \diag((\Sigma/M)_{0 \leq t \leq T}) \big) \nonumber \\
        & \leq C_1 \frac{\sigma T}{M} \label{eq:eps_Nash_1}
    \end{align}
    where $\sigma = \lVert \Sigma \rVert_F$. Now let us consider the same dynamics but under non-NE controls. The finite and infinite population costs under these controls are 
    \begin{align*}
        J^i_M(\mcu) & = \tr \big( \Psi^\top \big( Q + K^\top R K \big) \Psi \EE \big[\bomega^{j,M} (\bomega^{j,M})^\top \big] \big) \\
        & \hspace{5cm}+ \tr \big( \bPsi^\top \big( \bQ + \bK^\top \bR \bK \big) \bPsi \EE \big[\bomega^M (\bomega^M)^\top \big] \big) \\
        J^i(\mcu) & =\tr \big( \Psi^\top \big( Q + K^\top R K \big) \Psi \EE \big[\omega \omega^\top \big] \big) + \tr \big( \bPsi^\top \big( \bQ + \bK^\top \bR \bK \big) \bPsi \EE \big[\bomega^0 (\bomega^0)^\top \big] \big)
    \end{align*}
    with all matrices defined accordingly. Let us denote the control which infimizes the $M$ agent cost as $\tu^M$, meaning $\inf_{\mcu} J^i_M(\mcu) = J^i_M(\tmcu)$. Using the same techniques as before we get
    \begin{align*}
        & J^i_M(\tmcu) - J^i(\tmcu) \\
        & = \tr \Big( \big( \tilde{\bPsi}^\top \big( \bQ + \tilde{\bK}^\top \bR \tilde{\bK} \big) \tilde{\bPsi} + \Psi^\top \big( Q + K^\top R K \big) \Psi \big)  \diag((\Sigma/M)_{0 \leq t \leq T}) \Big) \\
        &\qquad \geq \lambda_{\min} (Q_0) \tr \big( \diag((\Sigma/M)_{0 \leq t \leq T}) \big) \\
        & \qquad=  \lambda_{\min} (Q_0) \frac{\sigma T}{M}.
    \end{align*}
    Using this we can further deduce
    \begin{align}
        J^i_M(\tmcu) \geq J^i(\tmcu) + \lambda_{\min} (Q_0) \frac{\sigma T}{M} \geq J^i(u^*) + \lambda_{\min} (Q_0) \frac{\sigma T}{M}. \label{eq:eps_Nash_2}
    \end{align}
    Hence we deduce
    \begin{align*}
        J^i_M(u^*) - \inf_u J^i_M(u) = J^i_M(u^*) - J^i(u^*) + J^i(u^*) - \inf_u J^i_M(u) \leq (C_1 + \lambda_{\min} (Q_0))  \frac{\sigma T}{M} 
    \end{align*}
    which completes the proof.
\end{proof}

\section{Proof of Lemma \ref{lemma:receding_horizon_gradient}} \textbf{[Policy Gradient Characterization]}\label{sec:receding_horizon_gradient}
    We characterize the policy gradient of the cost by first proving the quadratic structure of the cost. 
    For a fixed $t \in \{0,\ldots,T-1\}$ and a given set of controllers $(K^i,K^{-i})$ consider the partial cost
    \begin{align*}
        \tJ^i_{y,[t',T]}(K^i,K^{-i}) = \sum_{s = t'}^{T} \EE \bigg[ y^\top_s \big(Q^i_s + (K^i_s)^\top R^i_s K^i_s \big) y_s \bigg]
    \end{align*}
    for $t \leq t' \leq T$ and $K_T = 0$. From direct calculation $\tJ^i_{y,[t,T]} = \tJ^i_{y,t}$. First will show that $\tJ^i_{y,[t',T]}$ has a quadratic structure
    \begin{align} \label{eq:quad_cost}
        \tJ^i_{y,[t',T]}(K^i,K^{-i}) = \EE[y^\top_{t'} P^i_{y,t'} y_{t'}] + N^i_{y,t'}
    \end{align}
    where $P^i_{y,t}$ is defined as in \eqref{eq:tP_iyt} and 
    \begin{align*}
        N^i_{y,t} = N^i_{y,t+1} + \tr(\Sigma P^i_{y,t+1}), \hspace{0.2cm} N^i_{y,T} = 0
    \end{align*}
    where $\Sigma = \diag((\Sigma^i)_{i \in [N]})$. The hypothesis is true for the base case since cost $\tJ^i_{y,[T,T]}(K^i,K^{-i}) = \EE[y^\top_T Q^i_T y_T] = \EE[y^\top_T P^i_{y,T} y_T] + N^i_{y,T}$. Now assume for a given $t' \in \{ t,\ldots, T-1\}$, $\tJ^i_{y,[t'+1,T]}(K^i,K^{-i}) = \EE[y^\top_{t'+1} P^i_{y,t'+1} y_{t'+1}] + N^i_{y,t'+1}$.
    \begin{align*}
        & \tJ^i_{y,[t',T]}(K^i,K^{-i}) = \EE \big[ y^\top_{t'} \big(Q^i_{t'} + (K^i_{t'})^\top R^i_{t'} K^i_{t'} \big) y_{t'} \big] + \tJ^i_{y,[t'+1,T]}, \\
        & =  \EE \big[ y^\top_{t'} \big(Q^i_{t'} + (K^i_{t'})^\top R^i_{t'} K^i_{t'} \big) y_{t'} \big] + \EE[y^\top_{t'+1} P^i_{y,t'+1} y_{t'+1}] + N^i_{y,t'+1} \\
        & =  \EE \bigg[ y^\top_{t'} \bigg(Q^i_{t'} + (K^i_{t'})^\top R^i_{t'} K^i_{t'} + \big(A_{t'} - \sum_{j=1}^N B^j_{t'} K^j_{t'} \big)^\top P^i_{y,t'+1} \big(A_{t'} - \sum_{j=1}^N B^{j}_t K^j_{t'} \big) \bigg) y_{t'} \bigg] \\
        & \hspace{9cm} + \tr(\Sigma P^i_{y,t'+1}) + N^i_{y,t'+1} \\
        & = \EE[y^\top_{t'} P^i_{y,t'} y_{t'}] + N^i_{y,t'}.
    \end{align*}
    Hence we have shown \eqref{eq:quad_cost}, which implies 
    \begin{align*} 
    \tJ^i_{y,t}(K^i,K^{-i}) = \EE[y^\top_{t} P^i_{y,t} y_{t}] + N^i_{y,t}.
    \end{align*}
    Using \eqref{eq:tP_iyt} we can also write the cost $\tJ^i_{y,t}(K^i,K^{-i})$ in terms of the controller $K^i_t$ as
    \begin{align*}
         \tJ^i_{y,t}(K^i,K^{-i}) &= N^i_{y,t} + \EE \Bigg[y^\top_t \bigg( Q^i_t + (K^i_t)^\top \big( R^i_t + (B^i_t)^\top P^i_{y,t+1} B^i_t \big) K^i_t  \\
        & \quad- 2 (B^i_t K^i_t)^\top P^i_{y,t+1} \big( A^i_t - \sum_{j \neq i} B^j_t K^j_t \big) \\
        &\quad + \big( A^i_t - \sum_{j \neq i} B^j_t K^j_t \big)^\top P^i_{y,t+1} \big( A^i_t - \sum_{j \neq i} B^j_t K^j_t \big) \bigg) y_t \bigg] .
    \end{align*}
    Taking the derivative with respect to $K^i_t$ and using the fact that $\EE[y_t y_t^\top]=\Sigma_y$ we can conclude the proof:
    \begin{align*}
    	\frac{\delta \tJ^i_{y,t}(K^i,K^{-i})}{\delta K^i_t} &  = 2 \big((R^i_t +  (B^i_t)^\top P^i_{y,t+1} B^i_t)K^i_t - (B^i_t)^\top P^i_{y,t+1} \big(A_t-\sum_{j \neq i}B^j_t K^j_t \big)  \big)\Sigma_y.
    \end{align*}

\section{Proof of Lemma \ref{lem:PL}} \textbf{[Polyak-\L ojasiewicz inequality]}\label{sec:PL}
    We show that the gradient domination property (PL condition) is much simpler than those in the literature, as the receding-horizon approach obviates the need for computing the advantage function and the cost difference lemma as in \cite{li2021distributed,fazel2018global,malik2019derivative}. Especially upon comparison with with Lemma 3.9 in \cite{hambly2023policy} one can see that there is no summation in the PL condition and moreover the PL condition does not depend on the covariance $\lVert \Sigma_{K^*} \rVert$ but instead it depends on $\lVert \Sigma_y \rVert$ which is a parameter we can modify. 
    Let us start by defining the matrix sequences $P^{i*}_{y,t}$ and $N^{i*}_{y,t}$ such that
    \begin{align} \label{eq:tP*_iyt}
        P^{i*}_{y,t} = Q^i_t + (K^{i*}_t)^T R^i_t K^{i*}_t + (A_t - \sum_{j=1}^N B^j_t K^{j*}_t)^\top P^{i*}_{y,t+1} (A_t - \sum_{j=1}^N B^j_t K^{j*}_t), P^{i*}_{y,T} = Q^i_T.
    \end{align}
    and
    \begin{align*}
        N^{i*}_{y,t} = N^{i*}_{y,t+1} + \tr(\Sigma P^{i*}_{y,t+1}), \hspace{0.2cm} N^{i*}_{y,T} = 0
    \end{align*}
    From \eqref{eq:quad_cost} in proof of Lemma \ref{lemma:receding_horizon_gradient} we know that
    \begin{align*}
        \tJ^i_{y,t}(K^{i*},K^{-i*}) = \EE[y^\top_{t} P^{i*}_{y,t} y_{t}] + N^{i*}_{y,t} \text{  and  }  \tJ^i_{y,t}(\tK^{i},K^{-i*}) = \EE[y^\top_{t} P^{i'}_{y,t} y_{t}] + N^{i'}_{y,t}
    \end{align*}
    where $N^{i'}_{y,t} =  N^{i*}_{y,t+1} + \tr(\Sigma P^{i*}_{y,t+1}) = N^{i*}_t$ and
    \begin{align} \label{eq:tP'_iyt}
        & P^{i'}_{y,t} = \\
        & Q^i_t + (K^{i}_t)^T R^i_t K^{i}_t + (A_t - \sum_{j \neq i}^N B^j_t K^{j*}_t - B^i_t K^i_t)^\top P^{i*}_{y,t+1} (A_t - \sum_{j \neq i}^N B^j_t K^{j*}_t - B^i_t K^i_t) \nonumber
    \end{align}
    Using \eqref{eq:tP*_iyt} and \eqref{eq:tP'_iyt} we can deduce
    \begin{align*}
        & \tJ^i_{y,t}(\tK^i,K^{-i*}) - \tJ^i_{y,t}(K^{i*},K^{-i*}) = - (\tJ^i_{y,t}(K^{i*},K^{-i*}) - \tJ^i_{y,t}(\tK^i,K^{-i*})) \\
        & = - \EE\big[y^\top_t \big( Q^i_t + (K^{i*}_t)^\top R^i_t K^{i*}_t \\
        & \hspace{3cm} + \big( A_t - \sum_{j=1}^N B^j_t K^{j*}_t \big)^\top P^{i*}_{y,t+1}  \big( A_t - \sum_{j=1}^N B^j_t K^{j*}_t \big) - P^{i'}_{y,t} \big) y_t \big] \\
        & =  - \EE\big[y^\top_t \big( Q^i_t + (K^{i*}_t - K^{i}_t + K^{i}_t)^\top R^i_t (K^{i*}_t - K^{i}_t + K^{i}_t) \\
        & \hspace{0.4cm}+ \big( A_t - \sum_{j \neq i} B^j_t K^{j*}_t - B^i_t (K^{i*}_t - K^{i}_t + K^{i}_t) \big)^\top P^{i*}_{y,t+1} \\
        & \hspace{3cm} \big( A_t - \sum_{j \neq i} B^j_t K^{j*}_t - B^i_t (K^{i*}_t - K^{i}_t + K^{i}_t) \big) - P^{i'}_{y,t} \big) y_t \big] \\
        & = - \EE \big[ y^\top_t \big( (K^{i*}_t - K^{i}_t)^\top (R^i_t + (B^i_t)^\top P^{i*}_{y,t+1} B^i_t) (K^{i*}_t - K^{i}_t) \\
        & \hspace{1cm}+ 2 (K^{i*}_t - K^{i}_t)^\top \big( R^i_t K^{i}_t - (B^i_t)^\top P^{i*}_{y,t+1} (A_t - \sum_{j \neq i}B^j_t K^{j*}_t - B^i_t K^{i}_t) \big) \big) y_t \big]
    \end{align*}
    Let us define the natural gradient \cite{fazel2018global,malik2019derivative} as 
    \begin{align*}
    E^{i*}_t = \nabla^i_{y,t} (\tK^i,K^{-i*}) \Sigma^{-1}_y = R^i_t K^{i}_t - (B^i_t)^\top P^{i*}_{y,t+1} (A_t - \sum_{j \neq i}B^j_t K^{j*}_t - B^i_t K^{i}_t).
    \end{align*}
    Then, using completion of squares we get
    \begin{align}
        & \tJ^i_{y,t}(\tK^i,K^{-i*}) - \tJ^i_{y,t}(K^{i*},K^{-i*}) \nonumber \\
        & = - \EE \big[ y^\top_t \big( (K^{i*}_t - K^{i}_t)^\top (R^i_t + (B^i_t)^\top P^{i*}_{y,t+1} B^i_t) (K^{i*}_t - K^{i}_t)+ 2 (K^{i*}_t - K^{i}_t)^\top E^{i*}_t \big) y_t \big] \label{eq:forward_PL}\\
        & = - \EE \big[ y^\top_t \big( (K^{i*}_t - K^{i}_t)^\top (R^i_t + (B^i_t)^\top P^{i*}_{y,t+1} B^i_t) (K^{i*}_t - K^{i}_t)+ 2 (K^{i*}_t - K^{i}_t)^\top E^{i*}_t  \nonumber\\
        & \hspace{0.2cm} + (E^{i*}_t)^\top (R^i_t + (B^i_t)^\top P^{i*}_{y,t+1} B^i_t)^{-1} E^{i*}_t - (E^{i*}_t)^\top (R^i_t + (B^i_t)^\top P^{i*}_{y,t+1} B^i_t)^{-1} E^{i*}_t \big) y_t \big] \nonumber\\
        & = - \EE \big[ y^\top_t \big( (K^{i*}_t - K^{i}_t + (R^i_t + (B^i_t)^\top P^{i*}_{y,t+1} B^i_t)^{-1} E^{i*}_t)^\top (R^i_t + (B^i_t)^\top P^{i*}_{y,t+1} B^i_t) \nonumber\\
        & \hspace{5.5cm} (K^{i*}_t - K^{i}_t + (R^i_t + (B^i_t)^\top P^{i*}_{y,t+1} B^i_t)^{-1} E^{i*}_t) \big) y_t \big] \nonumber\\
        & \hspace{1cm}+ \EE \big[ y^\top_t \big( (E^{i*}_t)^\top (R^i_t + (B^i_t)^\top P^{i*}_{y,t+1} B^i_t)^{-1} E^{i*}_t \big) y_t \big] \nonumber\\
        & \leq \EE \big[ y^\top_t \big( (E^{i*}_t)^\top (R^i_t + (B^i_t)^\top P^{i*}_{y,t+1} B^i_t)^{-1} E^{i*}_t \big) y_t \big] \nonumber\\
        & = \EE \big[ \tr \big( y_t y^\top_t \big( (E^{i*}_t)^\top (R^i_t + (B^i_t)^\top P^{i*}_{y,t+1} B^i_t)^{-1} E^{i*}_t \big)  \big) \big] \nonumber \\
        & \leq \frac{\lVert \Sigma_y \rVert}{\sigma_R} \lVert E^{i*}_t  \rVert^2_F \label{eq:grad_dom_E}\\
        & \leq \frac{\lVert \Sigma_y \rVert}{{\sigma}_R \sigma^2_y} \lVert \nabla^i_{y,t}(\tK^i,K^{-i*}) \rVert^2_F \label{eq:grad_dom_nab}.
    \end{align}
    This concludes the proof.

\section{Cost augmentation technique and $\Os(\gamma)$-Nash equilibrium} \label{sec:cost_augment}
In this section we will prove how each agent can independently use a cost-augmentation technique to ensure the satisfaction of the diagonal domination condition (Assumption \ref{asm:diag_dom}) hence ensuring convergence of the MRPG algorithm. We also prove that this cost augmentation results in an $\Os(\gamma)$-Nash equilibrium where $\gamma$ is the max over the cost augmentation parameters. First we introduce the $\gamma^i$-augmented cost functions as follows.
\begin{align*}
    \uJ^i_{y,t} (K^i,K^{-i},\gamma^i) & = \EE \Big[ \sum_{s=t}^T (1+\gamma^i_s) y^\top_s\big( Q^i_s + (K^i_s)^\top R^i_s K^i_s \big) y_s \Big], \\
    \uJ^i_{\bx,t} (K^i,K^{-i},\gamma^i) & = \EE \Big[ \sum_{s=t}^T (1+\gamma^i_s) \bx^\top_s\big( \bQ^i_s + \bK^i_s)^\top \bR^i_s \bK^i_s \big) \bx_s \Big]
\end{align*}
The $\gamma^i_t \geq 0$ are chosen such that they satisfy the gradient domination condition which can be written as follows. 
\begin{align} \label{eq:gamma_condition}
\gamma^i_t \geq \frac{\sqrt{2m(N-1)} \gamma^2_{B,t} \gamma^i_{\uP,t+1}}{\sigma(R^i_t)} - 1
\end{align}
The stochastic gradient of the augmented cost function is,
\begin{align*}
	\unabla^i_{y,t} (K^i,K^{-i}) & = \frac{m}{N_br^2} \sum_{j=1}^{N_b} \uJ^i_{y,t}(\hK^i(e_j,t),K^{-i}) e_j \\
	\unabla^i_{\tx,t} (\bK^i,\bK^{-i}) & = \frac{m}{N_br^2} \sum_{j=1}^{N_b} \uJ^i_{\tx,t}(\hat{\bK}^i(e_j,t),\bK^{-i}) e_j
\end{align*}
respectively, where $e_j \sim \mathbb{S}^{pN \times mN}(r)$ is the perturbation and $\hK^i(e,t):= (K^i_t+e,\ldots,K^i_{T-1})$ 
is the  perturbed controller set at timestep $t$. $N_b$ denotes the mini-batch size and $r$ the smoothing radius of the stochastic gradient. The MRPG algorithm for this augmented cost is very similar to Algorithm \ref{alg:RL_GS_MFTG} but it utilizes the augmented stochastic gradients $\unabla^i_{y,t}$ and $\unabla^i_{\tx,t}$ and also a projection operator $\proj_D$ where $D >0$ is large enough.

\begin{algorithm}[h!]
	\caption{MRPG for augmented cost GS-MFTG}
	\begin{algorithmic}[1] \label{alg:RL_GS_MFTG_aug}
		\STATE {Initialize $K^i_t =0, \bK^i_t = 0$ for all $i \in [N], t \in \{0,\ldots,T-1\}$}
		\FOR {$t = T-1,\ldots,1,0,$}
		\FOR {$k = 1,\ldots,K$}
		\STATE {\bf Natural Policy Gradient} for $i \in [N]$
  \begin{align}
			\hspace{-0.5cm}\begin{pmatrix} K^i_t \\ \bK^i_t	\end{pmatrix} & \leftarrow \proj_D
            \bigg(\begin{pmatrix} K^i_t \\ \bK^i_t	\end{pmatrix} - \eta^i_k  \begin{pmatrix} \unabla^i_{y,t} (K^i,K^{-i}) \Sigma^{-1}_y \\ \unabla^i_{\tx,t} (\bK^i,\bK^{-i}) \Sigma^{-1}_{\tx}\end{pmatrix}\bigg)  
		\end{align}
        \hspace{-0.7cm}
		\ENDFOR
		\ENDFOR
	\end{algorithmic}
\end{algorithm}
Since the diagonal dominance condition is satisfied due to \eqref{eq:gamma_condition}, Algorithm \ref{alg:RL_GS_MFTG_aug} will converge linearly to the NE of the cost augmented GS-MFTG game. Now we prove that the NE under the augmented cost structure is $\Os(\bar{\gamma})$ away from the NE of the original game. In this section we only deal with the deviation process $(y_t)_{\forall t}$ \eqref{eq:dyn_decomp}-\eqref{eq:cost_decomp} and similar guarantees can be obtained for the mean-field process $(\bx_t)_{\forall t}$.
\begin{theorem}
    The NE under the augmented cost structure denoted as set of controllers $(\uK^{i*}_t)_{\forall i,t}$ such that
    \begin{align}
        \uJ^i_{y}(\uK^{i*},\uK^{-i*},\gamma^i) \leq \uJ^i_{y}(K^i,\uK^{-i*},\gamma^i), \hspace{0.2cm} \forall i \in [N], K \in \RR^{p \times m} \label{eq:NE_aug}
    \end{align}
    is an $\Os(\gamma)$-Nash equilibrium for the GS-MFTG \eqref{eq:gen_agent_dyn}-\eqref{eq:gen_agent_cost}.
    \begin{align}
        J^i_{y}(\uK^{i*},\uK^{-i*}) \leq J^i_{y}(K,\uK^{-i*}) + \Os(\gamma), \hspace{0.2cm} \forall i \in [N], K \in \RR^{p \times m} \label{eq:aug_cost_Nash}
    \end{align}
    where $\gamma = \max_{0 \leq t \leq T-1, i \in [N]} \gamma^i_t$. Also the difference between NE cost and augmented NE cost are
    \begin{align*}
        \lvert \uJ^i_{y}(\uK^{i*},\uK^{-i*},\gamma^i) - J^i_{y}(K^{i*},K^{-i*}) \lvert = \Os(\gamma).
    \end{align*}
\end{theorem}
\begin{proof}
    Let us first define the augmented cost function
    \begin{align*}
        \uJ^i_{y,t} (K^i,K^{-i},\gamma^i) = \EE \Big[ \sum_{s=t}^T (1+\gamma^i_t) y^\top_s\big( Q^i_s + (K^i_s)^\top R^i_s K^i_s \big) y_s \Big]
    \end{align*}
    with $\gamma^i_T = 0$ with $i \in [N]$. Using a similar argument this augmented cost can be written down as,
    \begin{align*}
        \uJ^i_{y,t}(K^i,K^{-i},\gamma^i) = \EE [y^\top_t \uP^i_t y_t] + \uN^i_t
    \end{align*}
    where $\uP^i_t$ and $\uN^i_t$ can be written recursively as,
    \begin{align}
        \uP^i_t & = (1+\gamma^i_t) (Q^i_t + (K^i_t)^\top R^i_t K^i_t) \nonumber \\
        & \hspace{3cm} + (A_t + \sum_{j=1}^N B^j_t K^j_t)^\top \uP^i_{t+1} (A_t + \sum_{j=1}^N B^j_t K^j_t), \hspace{0.2cm} \uP^i_T = Q^i_T \nonumber \\
        \uN^i_t & = \uN^i_{t+1} + \tr(\Sigma \uP^i_{t+1}), \hspace{0.2cm} \uN^i_T = 0 \label{eq:uP_uN}
    \end{align}
    Recalling the \emph{non-augmented} cost can be written down as,
    \begin{align*}
        J^i_{y,t}(K^i,K^{-i}) = \EE [y^\top_t P^i_t y_t] + N^i_t
    \end{align*}
    where $P^i_t$ and $N^i_t$ can be written recursively as,
    \begin{align}
        P^i_t & = (Q^i_t + (K^i_t)^\top R^i_t K^i_t) + (A_t + \sum_{j=1}^N B^j_t K^j_t)^\top P^i_{t+1} (A_t + \sum_{j=1}^N B^j_t K^j_t), \hspace{0.2cm} P^i_T = Q^i_T \nonumber \\
        N^i_t & = N^i_{t+1} + \tr(\Sigma P^i_{t+1}), \hspace{0.2cm} N^i_T = 0 \label{eq:P_N}
    \end{align}
    The difference between the augmented and non-augmented costs is as follows.
    \begin{align}
        & \lvert \uJ^i_{y,t}(K^i,K^{-i},\gamma^i) - J^i_{y,t}(K^i,K^{-i}) \rvert \nonumber \\
        & = \EE [y^\top_t (\uP^i_t - P^i_t) y_t] + \uN^i_t - N^i_t = \tr(\Sigma_y (\uP^i_t - P^i_t)) + \uN^i_t - N^i_t \label{eq:uJ-J}
    \end{align}
    So to quantify the difference $\uJ^i_{y,t}(K^i,K^{-i},\gamma^i) - J^i_{y,t}(K^i,K^{-i})$ we need to upper bound $\lVert \uP^i_t - P^i_t \rVert$ and $\lvert \uN^i_t - N^i_t \rvert$. Using \eqref{eq:uP_uN} and \eqref{eq:P_N} the quantity $\lvert \uN^i_t - N^i_t \rvert$ is 
    \begin{align}
        \lvert \uN^i_t - N^i_t \rvert = \sum_{s=t+1}^{T-1} \tr(\Sigma (\uP^i_s - P^i_s)) \label{eq:uN-N}
    \end{align}
    The quantity $\lVert \uP^i_t - P^i_t \rVert$ can be bounded as follows.
    \begin{align}
        & \lVert \uP^i_t - P^i_t \rVert \nonumber \\
        & = \lVert \gamma^i_t (Q^i_t + (K^i_t)^\top R^i_t K^i_t) + (A_t + \sum_{j=1}^N B^j_t K^j_t)^\top (\uP^i_{t+1} - P^i_{t+1}) (A_t + \sum_{j=1}^N B^j_t K^j_t) \rVert \nonumber \\
        & \leq \gamma^i_t (\gamma_Q + \gamma_R \lVert K^i_t \rVert^2) + (\gamma_A + \sum_{j=1}^N \gamma_B \lVert K^j_t \rVert )^2 \lVert \uP^i_{t+1} - P^i_{t+1} \rVert \nonumber \\
        & \leq \gamma^i_t \underbrace{(\gamma_Q + \gamma_R D^2)}_{c_1} + \underbrace{(\gamma_A + N \gamma_B D )^2}_{c_2} \lVert \uP^i_{t+1} - P^i_{t+1} \rVert \label{eq:uP-P}
    \end{align}
    Using \eqref{eq:uP-P} we can recursively say
    \begin{align}
        \lVert \uP^i_t - P^i_t \rVert \leq c_1 \sum_{s=t}^{T-1} c^{s-t}_2 \gamma^i_s + c^{T-t}_2 \lVert \uP^i_T - P^i_T \rVert = \Os \Big(\max_{t \leq s \leq T-1} \gamma^i_s \Big) = \Os(\gamma) \label{eq:uP-P_Os}
    \end{align}
    Similarly 
    \begin{align}
        \lvert \uN^i_t - N^i_t \rvert = \Os(\gamma) \label{eq:uN-N_Os}
    \end{align}
    using \eqref{eq:uN-N} and using \eqref{eq:uJ-J} with $t=0$
    \begin{align}
        \lvert \uJ^i_{y}(K^i,K^{-i},\gamma^i) - J^i_{y}(K^i,K^{-i}) \rvert = \Os(\gamma) \label{eq:cost_diff_aug}
    \end{align}
    Using \eqref{eq:NE_aug} and \eqref{eq:cost_diff_aug} for any $K^i \in \RR^{(p \times m) \times T}$,
    \begin{align}
        J^i_y(\uK^{i*},\uK^{-i*}) - \uJ^i_y(K^i, \uK^{-i*},\gamma^i) & \leq J^i_y(\uK^{i*},\uK^{-i*}) - \uJ^i_y(\uK^{i*}, \uK^{-i*},\gamma^i) = \Os(\gamma) \label{eq:eps_NE_aug_inter}
    \end{align}
    Now let us fix $i \in [N]$ and define a set of controllers $(K^i, \uK^{-i*})$ where $K^i \in \RR^{(p \times m) \times T}$. Let us define $\tP^i_t, \tilde{\uP}^i_t, \tN^i_t$ and $\tilde{\uN}^i_t$ as follows.
    \begin{align*}
        \tilde\uP^i_t & = (1+\gamma^i_t) (Q^i_t + (K^i_t)^\top R^i_t K^i_t) \\
        & \hspace{1cm}+ (A_t + B^i_t K^i_t + \sum_{j \neq i}^N B^j_t \uK^{j*}_t)^\top \tilde\uP^i_{t+1} (A_t + B^i_t K^i_t + \sum_{j \neq i}^N B^j_t \uK^{j*}_t), \hspace{0.2cm} \tilde\uP^i_T = Q^i_T \nonumber \\
        \tilde\uN^i_t & = \tilde\uN^i_{t+1} + \tr(\Sigma \tilde\uP^i_{t+1}), \hspace{0.2cm} \tilde\uN^i_T = 0 \\
        \tilde P^i_t & = (Q^i_t + (K^i_t)^\top R^i_t K^i_t) \\
        & \hspace{1cm}+ (A_t + B^i_t K^i_t + \sum_{j \neq i}^N B^j_t \uK^{j*}_t)^\top \tilde P^i_{t+1} (A_t + B^i_t K^i_t + \sum_{j \neq i}^N B^j_t \uK^{j*}_t), \hspace{0.2cm} \tilde P^i_T = Q^i_T \nonumber \\
        \tilde N^i_t & = \tilde N^i_{t+1} + \tr(\Sigma \tilde P^i_{t+1}), \hspace{0.2cm} \tilde N^i_T = 0
    \end{align*}
    Using these expressions we can deduce
    \begin{align}
        J^i_{y}(K^i,\uK^{-i*}) & = \uJ^i_{y}(K^i,\uK^{-i*},\gamma^i) - \tr(\Sigma_y (\tilde\uP^i_0 - \tP^i_0 )) - \tilde\uN^i_0 + \tN^i_0 \nonumber \\
        & \geq \uJ^i_{y}(\uK^{i*},\uK^{-i*},\gamma^i) - \tr(\Sigma_y (\tilde\uP^i_0 - \tP^i_0 )) - \tilde\uN^i_0 + \tN^i_0  \label{eq:eps_NE_aug_inter_2}
    \end{align}
    Using analysis similar to \eqref{eq:uP_uN}-\eqref{eq:uN-N_Os} we can deduce that $\lVert \tilde\uP^i_0 - \tP^i_0 \rVert = \Os(\gamma)$ and $\lvert \tilde\uN^i_0 - \tN^i_0 \rvert = \Os(\gamma)$. Using \eqref{eq:eps_NE_aug_inter}-\eqref{eq:eps_NE_aug_inter_2} we can write
    \begin{align*}
        & J^i_y(\uK^{i*},\uK^{-i*}) - J^i_{y}(K^i,\uK^{-i*}) \\
        & = J^i_y(\uK^{i*},\uK^{-i*}) - \uJ^i_y(\uK^i, \uK^{-i*},\gamma^i) + \uJ^i_{y}(\uK^{i*},\uK^{-i*},\gamma^i) - J^i_{y}(K^i,\uK^{-i*}) \\
        & \leq  J^i_y(\uK^{i*},\uK^{-i*}) - \uJ^i_y(\uK^i, \uK^{-i*},\gamma^i) + J^i_{y}(K^i,\uK^{-i*}) - J^i_{y}(K^i,\uK^{-i*}) \\
        & \hspace{7cm} + \tr(\Sigma_y (\tilde\uP^i_0 - \tP^i_0 )) + \tilde\uN^i_0 - \tN^i_0 \\
        & \leq \big| J^i_y(\uK^{i*},\uK^{-i*}) - \uJ^i_y(\uK^i, \uK^{-i*},\gamma^i) \big| + \big| \tr(\Sigma_y (\tilde\uP^i_0 - \tP^i_0 )) \big| + \big| \tilde\uN^i_0 - \tN^i_0 \big| = \Os(\gamma)
    \end{align*}
    which results in \eqref{eq:aug_cost_Nash}.
    
    Similarly we also bound the difference between $\uJ^i_{y,t}(\uK^{i*},\uK^{-i*},\gamma^i) - J^i_{y,t}(K^{i*},K^{-i*})$ where $(K^{i*}_t)_{\forall i,t}$ is the NE controllers and $(\uK^{i*}_t)_{\forall i,t}$ are the NE controllers under the augmented cost. These matrices are defined as follows,
    \begin{align*}
        \uJ^i_{y,t} (\uK^{i*},\uK^{-i*},\gamma^i) = \EE \Big[ \sum_{s=t}^T (1+\gamma^i_t) y^\top_s\big( Q^i_s + (\uK^{i*}_s)^\top R^i_s \uK^{i*}_s \big) y_s \Big]
    \end{align*}
    with $\gamma^i_T = 0$ with $i \in [N]$. Using a similar argument this augmented cost can be written down as,
    \begin{align*}
        \uJ^i_{y,t}(\uK^{i*},\uK^{-i*},\gamma^i) = \EE [y^\top_t \uP^{i*}_t y_t] + \uN^{i*}_t
    \end{align*}
    where $\uP^i_t$ and $\uN^i_t$ can be written recursively as,
    \begin{align}
        \uP^{i*}_t & = (1+\gamma^i_t) (Q^i_t + (\uK^{i*}_t)^\top R^i_t \uK^{i*}_t) + (\uA^*_t)^\top \uP^{i*}_{t+1} \uA^*_t, \hspace{0.2cm} \uP^{i*}_T = Q^i_T \nonumber \\
        \uN^{i*}_t & = \uN^{i*}_{t+1} + \tr(\Sigma \uP^{i*}_{t+1}), \hspace{0.2cm} \uN^{i*}_T = 0 \label{eq:uP*_uN}
    \end{align}
    where $\uA^*_t = A_t + \sum_{j=1}^N B^j_t \uK^{j*}_t$ and $\uK^{i*}_t$ is the NE according to the augmented cost defined as follows. 
    \begin{align*}
        \uK^{i*}_t = - ((1+\gamma^i_t)R^i_t + (B^i_t)^\top \uP^{i*}_{t+1} B^i_t)^{-1} (B^i_t)^\top \uP^{i*}_{t+1} \uA^{i*}_t
    \end{align*}
    where $\uA^{i*}_t = A_t + \sum_{j \neq i} B^j_t \uK^{j*}_t$. Recalling the \emph{non-augmented} cost can be written down as,
    \begin{align*}
        J^i_{y,t}(K^{i*},K^{-i*}) = \EE [y^\top_t P^{i*}_t y_t] + N^{i*}_t
    \end{align*}
    where $P^{i*}_t$ and $N^{i*}_t$ can be written recursively as,
    \begin{align}
        P^{i*}_t & = (Q^i_t + (K^{i*}_t)^\top R^i_t K^{i*}_t) + (A^*_t)^\top P^{i*}_{t+1} A^*_t, \hspace{0.2cm} P^{i*}_T = Q^i_T \nonumber \\
        N^{i*}_t & = N^{i*}_{t+1} + \tr(\Sigma P^{i*}_{t+1}), \hspace{0.2cm} \uN^{i*}_T = 0 \label{eq:P*_N}
    \end{align}
    where $A^*_t = A_t + \sum_{j=1}^N B^j_t K^{j*}_t$ and $K^{i*}_t$ is the NE according to the augmented cost defined as follows. 
    \begin{align*}
        K^{i*}_t = - (R^i_t + (B^i_t)^\top P^{i*}_{t+1} B^i_t)^{-1} (B^i_t)^\top P^{i*}_{t+1} A^{i*}_t
    \end{align*}
    where $A^{i*}_t = A_t + \sum_{j \neq i} B^j_t K^{j*}_t$. First we characterize the difference between the NE controllers under the two cost functions. 
    \begin{align*}
        \uK^{i*}_t - K^{i*}_t & = - ((1+\gamma^i_t)R^i_t + (B^i_t)^\top \uP^{i*}_{t+1} B^i_t)^{-1} (B^i_t)^\top \uP^{i*}_{t+1} \uA^{i*}_t \\
        & \hspace{2cm} + (R^i_t + (B^i_t)^\top P^{i*}_{t+1} B^i_t)^{-1} (B^i_t)^\top P^{i*}_{t+1} A^{i*}_t \\
        & = \uDelta K^{i*}_{t,1} - \uDelta K^{i*}_{t,2}
    \end{align*}
    where
    \begin{align*}
        & \uDelta K^{i*}_{t,1} = -\big[ ((1+\gamma^i_t) R^i_t + (B^i_t)^\top \uP^i_{t+1} B^i_t)^{-1} - (R^i_t + (B^i_t)^\top \uP^i_{t+1} B^i_t)^{-1} \big] (B^i_t)^\top \uP^{i*}_{t+1} \uA^{i*}_t \\
        & \uDelta K^{i*}_{t,2} = \\
        & \hspace{0.1cm} - (R^i_t + (B^i_t)^\top \uP^i_{t+1} B^i_t)^{-1} (B^i_t)^\top \uP^{i*}_{t+1} \uA^{i*}_t + (R^i_t + (B^i_t)^\top P^i_{t+1} B^i_t)^{-1} (B^i_t)^\top P^{i*}_{t+1} A^{i*}_t
    \end{align*}
    Using the equality $A^{-1} - B^{-1} = A^{-1} B B^{-1} - A^{-1} A B^{-1} = A^{-1} (B-A) B^{-1}$ for any set of matrices $A$ and $B$ whose inverses exist we can simplify the expression $\uDelta K^{i*}_{t,1}$ as follows.
    \begin{align*}
        \uDelta K^{i*}_{t,1} & = -\big[ ((1+\gamma^i_t) R^i_t + (B^i_t)^\top \uP^i_{t+1} B^i_t)^{-1} (\gamma^i_t R^i_t) (R^i_t + (B^i_t)^\top \uP^i_{t+1} B^i_t)^{-1} \big] (B^i_t)^\top \uP^{i*}_{t+1} \uA^{i*}_t
    \end{align*}
    Hence the norm of this expression can be upper bounded by 
    \begin{align*}
        \lVert \uDelta K^{i*}_{t,1} \rVert \leq \gamma^i_t \underbrace{\lVert R^i_t \rVert \lVert (R^i_t + (B^i_t)^\top \uP^{i*}_{t+1} B^i_t)^{-1} \rVert \lVert (R^i_t + (B^i_t)^\top P^{i*}_{t+1} B^i_t)^{-1} \rVert \lVert (B^i_t)^\top \uP^{i*}_{t+1} \uA^{i*}_t \rVert }_{\bar{c}_0}
    \end{align*}
    The norm of the second expression $\uDelta K^{i*}_{t,2}$ can be bounded using techniques from proof of Theorem \ref{thm:main_res}.
    \begin{align*}
        \uDelta K^{i*}_{t,2} \leq \underbrace{\lVert \uPhi^{-1}_{t+1} \rVert \lVert A^*_t\rVert \gamma_B}_{\bar{c}_1} \lVert \uP^{i*}_{t+1} - P^{i*}_{t+1} \rVert
    \end{align*}
    where
        \begin{align*}
        \uPhi_{t+1} = \begin{pmatrix}
            (1+\gamma^1_t)R^1_t + (B^1_t)^\top \uP^{1*}_{t+1} B^1_t & \hdots & (B^1_t)^\top \uP^{1*}_{t+1} B^N_t \\
            (B^2_t)^\top \uP^{2*}_{t+1} B^1_t & \hdots & (B^2_t)^\top \uP^{2*}_{t+1} B^N_t \\
            \vdots & \ddots & \vdots \\
            (B^N_t)^\top \uP^{N*}_{t+1} B^1_t & \hdots & (1+\gamma^N_t)R^N_t + (B^N_t)^\top \uP^{N*}_{t+1} B^N_t
        \end{pmatrix}
    \end{align*}
    This matrix is invertible if $\gamma^i_t$ are chosen so as to satisfy the diagonal dominance condition.
    The difference between the augmented and non-augmented costs \emph{at their corresponding NE} is as follows.
    \begin{align}
        \uJ^i_{y,t}(\uK^{i*},\uK^{-i*},\gamma^i) - J^i_{y,t}(K^{i*},K^{-i*}) & = \EE [y^\top_t (\uP^{i*}_t - P^{i*}_t) y_t] + \uN^{i*}_t - N^{i*}_t \nonumber \\
        & = \tr(\Sigma_y (\uP^{i*}_t - P^{i*}_t)) + \uN^{i*}_t - N^{i*}_t \label{eq:uJ*-J*}
    \end{align}
    So now we characterize the difference $\uP^{i*}_t - P^{i*}_t$ using the alternate expressions for both the matrices.
    \begin{align*}
        \uP^{i*}_t - P^{i*}_t = (\uA^{i*}_t)^\top \uP^{i*}_{t+1} (\uA^{i*}_{t} + B^i_t \uK^{i*}_t) - (A^{i*}_t)^\top P^{i*}_{t+1} (A^{i*}_{t} + B^i_t K^{i*}_t) 
    \end{align*}
    Using this expression we bound the norm of $\uP^{i*}_t - P^{i*}_t$ as follows.
    \begin{align}
        \lVert  \uP^{i*}_t - P^{i*}_t \rVert & \leq \underbrace{(2 \gamma^*_A \gamma^*_P \gamma_B + (\gamma^*_A/2 + \gamma^*_P \gamma_B + \gamma^{i*}_A/2)\gamma_B + \gamma^4_B/2)}_{\bar{c}_2} \sum_{j=1}^N \lVert \uK^{j*}_t - K^{j*}_t \rVert \nonumber \\
        & \hspace{4cm} + \underbrace{(\gamma^*_A/2+1/2 + (\gamma^{i*}_A)^2 \gamma^*_A/2 )}_{\bar{c}_3} \lVert \uP^{i*}_{t+1} - P^{i*}_{t+1} \rVert \nonumber \\
        & \leq \underbrace{(\bar{c}_1 \bar{c}_2 N + \bar{c}_3)}_{\bar{c}_4} \lVert \uP^{i*}_{t+1} - P^{i*}_{t+1} \rVert + \underbrace{\bar{c}_0 \bar{c}_2 N}_{\bar{c}_5} \gamma^i_t \label{eq:uP*-P*}
    \end{align}
    Using \eqref{eq:uP*-P*} recursively we get
    \begin{align}
        \lVert  \uP^{i*}_t - P^{i*}_t \rVert & \leq \bar{c}_5 \sum_{s=0}^{T-t-1} \bar{c}_4^s \gamma^i_{t+s} + \bar{c}_4^{T-t} \lVert  \uP^{i*}_T - P^{i*}_T \rVert \nonumber \\
        & \leq \bar{c}_5 \bar{c}_4^{T-t} (T-t) \max_{0 \leq s \leq T-t-1} \gamma^i_{s+j} = \Os \bigg(\max_{t \leq s \leq T-1} \gamma^i_s \bigg)  \label{eq:uP*-P*_Os}
    \end{align}
    Similarly \eqref{eq:uP*_uN} and \eqref{eq:P*_N}
    \begin{align}
        \lvert \uN^{i*}_t - N^{i*}_t \rvert = \Os(\gamma) \label{eq:uN*-N*_Os}
    \end{align}
    using \eqref{eq:uJ*-J*} with $t=0$, \eqref{eq:uP*-P*_Os} and \eqref{eq:uN*-N*_Os} we arrive at
    \begin{align*}
        \lvert \uJ^i_{y}(K^i,K^{-i},\gamma^i) - J^i_{y}(K^i,K^{-i}) \rvert = \Os(\gamma)
    \end{align*}
\end{proof}

\section{MRPG algorithm: Analysis with sample-path gradients}
\label{subsec:Anal_stoc_cost}
In this section we will show how to utilize the sample-paths of $N$ teams comprising $M$ agents each to approximate the expected cost in stochastic gradient computation \eqref{eq:policy_grad_1}. This empirical cost is referred to as the \emph{sample-path cost}, and the sample-path costs for the stochastic processes $y_t$ and $\tx_t$ (\eqref{eq:y_t} and \eqref{eq:tx_t}) are defined as below.
\begin{align}
	 \cJ^{i}_{y,t}(K^i,K^{-i}) & = \frac{1}{M} \sum_{j=1}^M\sum_{s = t}^T  (y^j_s)^\top (Q^i_t + (K^i_s)^\top R^i_t K^i_s ) y^j_s, \label{eq:sample_path_cost_y} \\
	 \cJ^{i}_{\tx,t}(\bK^i,\bK^{-i}) & = \sum_{s = t}^T  (\tx_s)^\top (\bQ^i_t + (\bK^i_s)^\top \bR^i_t \bK^i_s ) \tx_s , \label{eq:sample_path_cost_tx}
\end{align}
where $K^i_T, \bK^i_T = 0$. Most literature \cite{li2021distributed,fazel2018global,malik2019derivative} uses the sample-path cost as a stand-in for the expected cost to be used in stochastic gradient computation as in \eqref{eq:policy_grad_1}. 
\cite{malik2019derivative} mentions that the sample-path cost is an unbiased estimator of the expected cost (under a stabilizing control policy) and is $\delta$ close to the expected cost, if the length of the sample path is $\Os(\log(1/\delta))$. We show that this intuition although true for the \emph{infinite-horizon ergodic-cost} setting, does not carry over to the finite-horizon or even infinite-horizon discounted-cost setting. The finite horizon setting prevents us from obtaining sample paths of arbitrary lengths, which results in an approximation error. We show in Lemma \ref{lem:finite_rollout_cost} that the sample-path cost is an unbiased \emph{but high variance} estimator of the expected cost. And using the Markov inequality we show how to use mini-batch size $N_b$ to get a good estimate of the stochastic gradient. 

Similar analysis can be carried out for the infinite-horizon discounted-cost setting and we provide an outline of this argument below. In the infinite-horizon ergodic-cost setting the expected cost essentially depends on the \emph{infinite tail} of the Markov chain which achieves a stationary distribution since all stabilizing controllers in stochastic LQ settings ensure unique stationary distributions. As a result the sample path cost is bound to approach the expected cost as the Markov chain distribution approaches the stationary distribution. On the other hand, in the infinite-horizon discounted cost setting, the expected cost depends on the \emph{head} of the Markov chain (with an \emph{effective} time horizon of $1/(1-\gamma)$)

In the infinite-horizon discounted-cost setting the cost cannot be written down cleanly in terms of the stationary distribution. This is due to the fact that in the discounted cost setting the cost depends essentially on the \emph{head} of the sample path and in the ergodic cost setting the cost depends on the \emph{tail} of the sample path. As the tail of the sample paths approaches the stationary distribution (for stabilizing controllers), the ergodic cost approaches the expected cost. 

Now we present the result proving that in the finite-horizon setting the sample-path is an unbiased estimator of the expected cost and the second moment of the difference is bounded. 

\begin{lemma} \label{lem:finite_rollout_cost}
	\begin{align*}
		\EE[\cJ^{i}_{y,t}(K^i,K^{-i})] = \tJ^{i}_{y,t}(K^i,K^{-i}), & \hspace{0.2cm} \EE[\cJ^{i}_{\tx,t}(\bK^i,\bK^{-i})] = \tJ^{i}_{\tx,t}(\bK^i,\bK^{-i}) \\
		\EE[(\cJ^{i}_{y,t}(K^i,K^{-i}) - \tJ^{i}_{y,t}(K^i,K^{-i}))^2] & \leq 2 \tr(\Phi)^2
	\end{align*}
\end{lemma}

\begin{proof} 
	For ease of exposition in this proof we will consider an single agent stochastic LQ system under set of controllers $K = (K_t)_{0 \leq t \leq T-1}$
	\begin{align*}
		& \text{(dynamics)} & x_{t+1}  = (A_t - B_t K_t) x_t + \omega_t = A_K(t) x_t + \omega_t, \\
        & \text{(sample-path cost)} & \cJ_{t}(K)  = \sum_{s = t}^T  (x_s)^\top (Q_s + (K_s)^\top R_s K_s ) x_s, \\
        & \text{(expected cost)} & \tJ_{t}(K)  = \EE \Big[\sum_{s = t}^T  (x_s)^\top (Q_s + (K_s)^\top R_s K_s ) x_s \Big]
	\end{align*}
	where $x_0,\omega_t \sim \Ns(0,\Sigma)$ and are i.i.d. and $K_T = 0$. The results can be generalized to the systems \eqref{eq:tx_t}-\eqref{eq:min_tJ_ixt} by concatenating the controllers for all $N$ players into a joint controller. Proof of the unbiasedness follows trivially.
	\begin{align*}
		\EE[\cJ_t(K)] = \EE\big[ \sum_{\tau = t}^T (x_\tau)^\top Q_K(\tau) x_\tau \big] =  \tJ_t(K).
	\end{align*} 
 We first recall standard results from literate characterizing the expectation and variance of the quadratic forms of Gaussian random variables.
	\begin{lemma}[\cite{seber2012linear}]\label{lem:prop_quad}
		Let $z \sim \Ns(0,I_p)$ and $M \in \RR^{p \times p}$ be a symmetric matrix then $\EE[z^\top M z] = \tr(M)$ and $\Var(z^\top M z) = 2 \lVert M \rVert^2_F$.
	\end{lemma}
	Now we show that $\cJ_t$ is an unbiased estimator for $\tJ_t$ and also bound the second moment of the difference between the two.
	From the above equation we can write
	\begin{align*}
		x_t = \sum_{\tau=1}^t A_{\tau}^{t-1} \omega_{\tau-1} \text{ where } A_t^s = A_K(s) A_K(s-1) \ldots A_K(t), \forall s \geq t \geq 0 
	\end{align*}
	and $A_t^s = I, \forall s < t$. Let us introduce
	\begin{align*}
		\Psi = 
		\begin{pmatrix} 
			\Sigma^{\frac{1}{2}} & 0 & 0 & \cdots & 0 \\
			A^1_1 \Sigma^{\frac{1}{2}} & \Sigma^{\frac{1}{2}} & 0 & \cdots & 0 \\
			A^2_1 \Sigma^{\frac{1}{2}} & A^2_2 \Sigma^{\frac{1}{2}} & \Sigma^{\frac{1}{2}} & \cdots & 0 \\
			\vdots & \vdots & \vdots & \ddots & \vdots \\
			A^{t-1}_1 \Sigma^{\frac{1}{2}} & A^{t-1}_2 \Sigma^{\frac{1}{2}} & A^{t-1}_3 \Sigma^{\frac{1}{2}} & \cdots & \Sigma^{\frac{1}{2}}
		\end{pmatrix},  \hspace{0.5cm} \varpi = \begin{pmatrix} \Sigma^{-\frac{1}{2}}_\omega \omega_0 \\ \vdots \\  \Sigma^{-\frac{1}{2}}_\omega \omega_{t-1} \end{pmatrix},
	\end{align*}
	We can deduce that $( \Psi \varpi)_t = x_t$ where $( \Psi \varpi)_t$ is the $t$-th block of $\Psi \varpi$. We can use these quantities to express the cost $\cJ_t$ as a quadratic function as follows
	\begin{align*}
		\cJ_t(K) = \sum_{\tau = t}^T (x_\tau)^\top Q_K(\tau) x_\tau = \sum_{\tau = t}^T ( \Psi \varpi)_\tau^\top Q_K(\tau) ( \Psi \varpi)_\tau = \varpi^\top \Phi \varpi
	\end{align*}
	where $\Phi = \Psi^\top \diag((Q_K(\tau))_{\tau = 1}^t) \Psi$ and $Q_K(\tau) = Q_\tau + (K_\tau)^\top R_\tau K_\tau$. 
	Proof of bounded second moment follows from,
	\begin{align*}
		\EE[(\tJ_t(K) - \tJ_t(K))^2] & \leq \big( \EE[\cJ_t(K) - \tJ_t(K)] \big)^2 + \Var(\cJ_t(K)) \\
		& = \Var(\varpi^\top \Phi \varpi) = 2 \lVert \Phi \rVert_F^2 \leq 2 \tr(\Phi)^2, 
	\end{align*}
using Lemma \ref{lem:prop_quad} and the fact that $\lVert \cdot \rVert_F \leq \tr(\cdot)$.
\end{proof}
The lemma states that the sample path costs $\cJ^i_{y,t}$ and $\cJ^i_{\tx,t}$ are unbiased estimates of expected costs $\tJ^i_{y,t}$ and $\tJ^i_{\tx,t}$, respectively and the second moments of $\cJ^i_{y,t}$ and $\cJ^i_{\tx,t}$ are bounded. Notice that the bound on second moment can be converted into a uniform bound by ensuring a uniform bound on the cost of controllers $(K^i,K^{-i})$. The sample-path cost can be used to obtain sample-path policy gradients as follows. We denote these sample-path gradients with respect to controller $K_{i,t}$ as $\tnabla_{K_{i,t}} \cJ^i_K(T)$ such that
\begin{align*}
	\cnabla^i_{y,t} (K^i,K^{-i}) & = \frac{m}{N_br^2} \sum_{j=1}^{N_b} \cJ^i_{y,t}(\hK^i(e_j,t),K^{-i}) e_j, e_j \sim \mathbb{S}^{pN \times mN}(r), \\
	\cnabla^i_{\tx,t} (\bK^i,\bK^{-i}) & = \frac{m}{N_br^2} \sum_{j=1}^{N_b} \cJ^i_{\tx,t}(\hat{\bK}^i(e_j,t),\bK^{-i}) e_j, e_j \sim \mathbb{S}^{pN \times mN}(r),
\end{align*}
respectively, where $\hK^i(e,t):= (K^i_t+e,\ldots,K^i_{T-1})$ and $\hat{\bK}^i(e,t):= (\bK^i_t+e,\ldots,\bK^i_{T-1})$ are the controller sets with perturbations $e$ at timesteps $t$. 
\begin{algorithm}[h!]
	\caption{SP-MRPG for GS-MFTG}
	\begin{algorithmic}[1] \label{alg:RL_GS_MFTG_fin_SP}
		\STATE {Initialize $K^i_t =0, \bK^i_t = 0$ for all $i \in [N]$}
		\FOR {$t = T,T-1,\ldots,1,0,$}
		\FOR {$k = 1,\ldots,K$}
		\STATE {\underline{\bf Gradient Descent}}
		\begin{align} \label{eq:GDU_fin}
			\begin{pmatrix} K^i_t \\ \bK^i_t	\end{pmatrix} & \leftarrow \begin{pmatrix} K^i_t \\ \bK^i_t	\end{pmatrix} - \eta^i_k  \begin{pmatrix} \cnabla^i_{y,t} (K^i,K^{-i}) \\ \cnabla^i_{\tx,t} (\bK^i,\bK^{-i})\end{pmatrix} 
		\end{align}
		\ENDFOR
		\ENDFOR
	\end{algorithmic}
\end{algorithm}

Now we introduce the MRPG algorithm which uses sample-path policy gradients instead of stochastic policy gradients \eqref{eq:policy_grad_1}. This algorithm will be called SP-MRPG in short. Using Lemma \ref{lem:finite_rollout_cost} we can bound the difference between the stochastic gradient $\tnabla_{K_{i,t}} \tJ^i(K)$ and the sample-path gradient $\tnabla_{K_{i,t}}  \cJ^i_t (K_t)$. The following lemma states the fact that $\tnabla_{K_{i,t}}  \cJ^i_t (K_t)$ is an unbiased estimator of $\tnabla_{K_{i,t}} \tJ^i(K)$ and bounds the second moment of the difference between the two. The techniques used to prove this Lemma are similar to \cite{fazel2018global,malik2019derivative} hence are omitted.
\begin{lemma} \label{lem:finite_rollout_grad}
	\begin{align*}
		\EE[\cnabla^i_{y,t} (K^i,K^{-i})] & = \tnabla^i_{y,t} (K^i,K^{-i}), \\
		\EE[ \lvert \cnabla^i_{y,t} (K^i,K^{-i}) - \tnabla^i_{y,t} (K^i,K^{-i}) \rvert ] & \leq \frac{2 m^2}{N_b r^4} \tr(\Phi)^2
	\end{align*}
\end{lemma}
The proof follows from lemma \ref{lem:finite_rollout_cost}. Notice that although we do not have a high confidence bound between the sample path policy gradient and stochastic gradient as in Lemma \ref{lemma:smoothness_bias}, we instead have a bound on the second moment of their difference. Now using the the Markov inequality we can convert the second moment bound into a high confidence bound in the following lemma.
\begin{lemma}
Using Markov inequality and Lemma \ref{lem:finite_rollout_grad}, 
\begin{align*}
    \lvert \cnabla^i_{y,t} (K^i,K^{-i}) - \tnabla^i_{y,t} (K^i,K^{-i}) \rvert \leq  \frac{2 m^2}{N_b r^4 \delta} \tr(\Phi)^2
    \end{align*}
    with probability at least $1-\delta$.
\end{lemma}
Hence for a given $\epsilon > 0$, if $r = \Os(\epsilon)$ and $N_b = \tilde\Theta(\epsilon^{-4} \delta^{-1}$, then the approximation error between the sample-path gradient and the policy gradient $\lVert \cnabla^i_{y,t} (K^i,K^{-i}) - \nabla^i_{y,t} (K^i,K^{-i}) \rVert = \Os(\epsilon)$. Notice that the mini-batch size for the sample-path gradient $\tilde\Theta(\epsilon^{-4} \delta^{-1}$ is much higher compared to the mini-batch size needed for the stochastic gradient $\tilde\Theta(\epsilon^{-2})$.

We now show this mini-batch size dependence using empirical results below. In Figure \ref{fig:numer_SP_MRPG} we compare the difference between the MRPG algorithm (Algorithm \ref{alg:RL_GS_MFTG}) which utilizes stochastic gradients vs the SP-MRPG algorithm (Algorithm \ref{alg:RL_GS_MFTG_fin_SP}) which utilizes the sample-path gradients.
\begin{figure}[h!]
    \centering
    \includegraphics[width=0.6\linewidth]{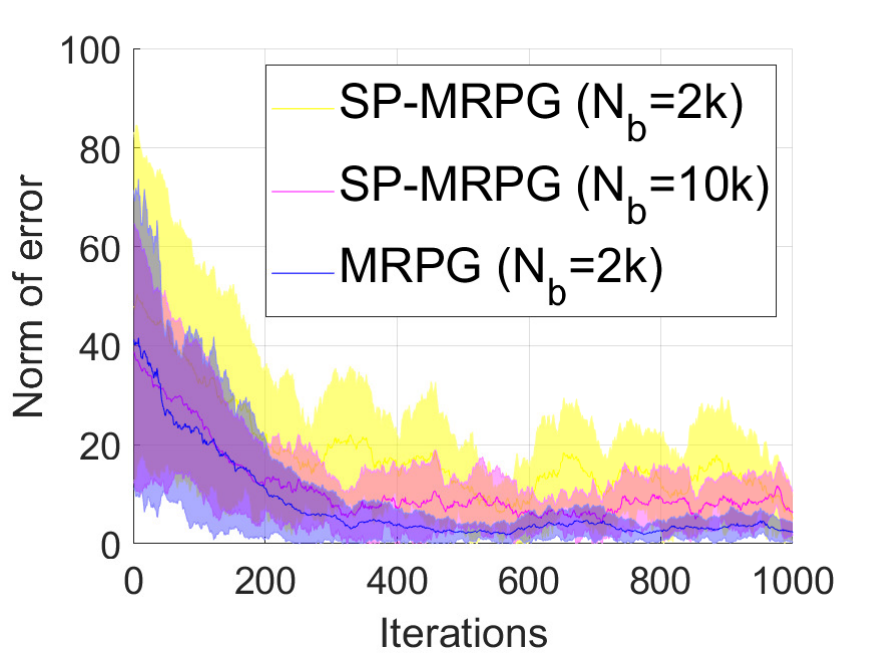}
     \caption{Comparison between MRPG $(N_b = 2,000)$, SP-MRPG $(N_b = 2,000)$ and SP-MRPG $(N_b = 10,000)$.}
      \label{fig:numer_SP_MRPG}
\end{figure}

The figure shows convergence of norm of error over number of iterations with number of teams $N=1$, time horizon $T=1$, number of agents per team $M=1000$ and  agents have scalar dynamics. The number of inner-loop iterations $K=1000$ and learning rate $\eta^i_k = 0.001$. We simulate the MRPG algorithm with mini-batch size $N_b=2000$, SP-MRPG with mini-batch size $N_b=2000$ and SP-MRPG with mini-batch size $N_b=10,000$. It is shown that if mini-batch size is the same MRPG has better convergence compared to SP-MRPG which is plagued with high variance. If we increase the mini-batch size in SP-MRPG we can reduce the variance due to lower variance of the sample-path gradient estimates.

\end{document}